\documentclass[a4paper,11pt]{article}
\usepackage{geometry}
\usepackage[numbers,sort&compress]{natbib}
\usepackage{hyperref}
\usepackage{url}
\usepackage{hyperref}
\usepackage{url}
\usepackage{amsmath,amssymb,amsthm,mathtools}
\usepackage{amsfonts}
\usepackage{bm}
\usepackage{graphicx}
\usepackage{subcaption}
\usepackage{booktabs}
\usepackage{enumitem}
\usepackage{microtype}
\usepackage{xcolor}
\usepackage{multirow}
\usepackage{algorithm}
\usepackage{algpseudocode}

\mathtoolsset{showonlyrefs}

\providecommand{\rev}[1]{#1}

\theoremstyle{plain}
\newtheorem{theorem}{Theorem}[section]
\newtheorem{proposition}[theorem]{Proposition}
\newtheorem{lemma}[theorem]{Lemma}

\theoremstyle{definition}
\newtheorem{definition}[theorem]{Definition}

\theoremstyle{remark}
\newtheorem{remark}[theorem]{Remark}

\newcommand{\R}{\mathbb{R}}
\newcommand{\N}{\mathbb{N}}
\newcommand{\Z}{\mathbb{Z}}
\newcommand{\bx}{\bm{x}}
\newcommand{\by}{\bm{y}}
\newcommand{\bz}{\bm{z}}
\newcommand{\bv}{\bm{v}}
\newcommand{\bw}{\bm{w}}
\newcommand{\bb}{\bm{b}}
\newcommand{\bc}{\bm{c}}
\newcommand{\bu}{\bm{u}}

\newcommand{\Hol}{\mathcal H^{\alpha}_{\lambda,R}([0,1]^d)}

\newcommand{\Net}{\mathcal N}
\newcommand{\bit}{\operatorname{bit}}
\newcommand{\Linf}{L^\infty([0,1]^d)}
\newcommand{\dexa}{\varrho}     
\newcommand{\dexac}{\varrho_c} 
\newcommand{\eact}{\sigma}      
\newcommand{\gact}{\phi}        

\DeclareMathOperator{\supp}{supp}

\usepackage[capitalize,noabbrev]{cleveref}

\title{Arbitrary-Accuracy Neural Approximation with \\Optimal Neuron Count and Optimal Bit Complexity}

\author{
Zilan Cheng\thanks{Division of Mathematical Sciences, School of Physical and Mathematical Sciences, Nanyang Technological University, 637371, Singapore. Email: \texttt{zilan001@e.ntu.edu.sg}.}
\and Li-Lian Wang\thanks{Division of Mathematical Sciences, School of Physical and Mathematical Sciences, Nanyang Technological University, 637371, Singapore. Email: \texttt{lilian@ntu.edu.sg}.}
\and Zhongjian Wang\thanks{Corresponding author. Division of Mathematical Sciences, School of Physical and Mathematical Sciences, Nanyang Technological University, 637371, Singapore. Email: \texttt{zhongjian.wang@ntu.edu.sg}.}
}

\begin{document}
\maketitle
\begin{abstract}
We study the minimum number of hidden neurons required for
arbitrary-accuracy approximation of \rev{continuous} functions on $[0,1]^d$ and the associated encoding complexity. For $d\geq 2$, we construct \rev{two} fixed, explicitly defined activation \rev{functions and, for each, a} closed-form network with two hidden layers of widths $d$ and $1$ \rev{that achieves} arbitrary accuracy in the uniform norm. We prove that $d+1$ is the exact minimum total number of hidden neurons among standard feedforward networks with locally integrable activations and affine outputs. We further give a simpler construction using a single elementary activation that combines the floor and exponential functions. This construction requires three hidden layers of widths $d$, $1$, and $2$. If a skip connection is allowed, widths $d$, $1$, and $1$ suffice. For a bounded $\alpha$-Hölder class \rev{with $\alpha\in(0,1]$}, \rev{the networks with the digit-extraction activation and with the elementary activation} require
$O(\varepsilon^{-d/\alpha}\log(1/\varepsilon))$ bits, matching the metric-entropy lower bound up to a logarithmic factor\rev{. The other $(d,1)$ network, which stores differences of neighboring values along a spanning tree, needs only $\Theta(\varepsilon^{-d/\alpha})$ bits and is optimal in both neuron count and bit complexity}.
\end{abstract}

\section{Introduction}
\label{sec:intro}

Universal approximation theorems \citep{cybenko1989approximation,hornik1989multilayer,hornik1991approximation,leshno1993multilayer} state that feedforward networks with suitable nonpolynomial activations are dense in $C([0,1]^d)$. In all classical statements, and in the quantitative theory that followed for ReLU networks \citep{yarotsky2017error,yarotsky2018optimal,petersen2018optimal,shen2022optimal}, the number of neurons grows without an explicit bound as the prescribed accuracy $\varepsilon$ tends to zero. A different regime was opened by \citet{maiorov1999ridge}: there is an analytic sigmoidal activation for which a two-hidden-layer network with \rev{a fixed number of hidden neurons} approximates every $f\in C([0,1]^d)$ to arbitrary accuracy, with only the parameters depending on $f$ and $\varepsilon$. Such \emph{fixed-size} universal approximators have since been constructed with progressively more explicit activations \citep{ismailov2014bounded,guliyev2018two,yarotsky2021elementary,zhang2022deep}. In what follows, we list two questions that are left open by this line of work and will be answered here.

\paragraph{How few neurons suffice?}
\rev{For uniform approximation with a single activation and without skip connections, the pioneering early works of \citet{maiorov1999ridge}, \citet{ismailov2014bounded} and \citet{guliyev2018two} use $O(d^2)$ hidden neurons. Their stated neuron counts $9d+3$, $3d+2$, and $3d+2$ rest on a nonstandard definition, in which the network is a sum of subnetworks, each with a single second-layer unit and its own first layer, but only one of these first layers is counted. The $O(d)$ dependence of the number of parameters on the input dimension was already established \citep{maiti2024optimal}, and if different activations and skip connections are allowed, \citet{yarotsky2021elementary} obtained the construction with the fewest hidden neurons known so far, $d+2$. However, the exact minimum was not known.} We construct and prove that \textbf{the exact answer is $d+1$}: a two-hidden-layer network of widths $(d,1)$ with a single explicit piecewise-constant activation and without any skip connection approximates every \rev{continuous function on $[0,1]^d$} to arbitrary $L^\infty$ accuracy (Theorem~\ref{thm:exact}), and no feedforward network with at most $d$ hidden neurons can do so, whatever its depth and whatever locally integrable activation it uses (Theorem~\ref{thm:lower}). The lower bound combines a rank argument for the first hidden layer with the fact, going back to \citep{lin1993fundamentality, maiorov1999ridge, pinkus1999approximation}, that sums of a fixed number of ridge functions are not dense; we give a direct proof valid for arbitrary locally integrable activations.

\paragraph{What does a small network cost?}\leavevmode
\rev{Networks whose size grows as $\varepsilon$ decreases spread the information about $f$ over many parameters: for uniform accuracy $\varepsilon$ on functions of smoothness $s>0$, shallow ReLU networks usually have about $\varepsilon^{-d/s}$ parameters of $O(\log(1/\varepsilon))$ bits each \citep{hong2024bridging,yarotsky2020phase}, and very deep ReLU networks can remove the logarithm, with about $\varepsilon^{-d/(2s)}$ parameters of about $\varepsilon^{-d/(2s)}$ bits each \citep{yarotsky2020phase}.}
\rev{In contrast, fixed-size} constructions pay for accuracy with the magnitude and precision of their parameters, a phenomenon \citet{zhang2022deep} call the \emph{curse of memory} but do not quantify. 
We show that our constructions \rev{store almost all the information about $f$ in a single integer parameter, which can be quantified by its bit length. Over the} H\"older
class \rev{with $\alpha\in(0,1]$, the network of Section~\ref{subsubsec:difference-encoding} (and this single parameter) uses $\Theta(\varepsilon^{-d/\alpha})$ bits in the worst case} (Theorem~\ref{thm:bits}). 
\rev{This is already optimal, as} the classical Kolmogorov--Tikhomirov
\emph{$\varepsilon$-entropy} estimate
\citep{kolmogorov1959varepsilon} implies (Proposition~\ref{prop:entropy}).
\rev{This estimate has also been investigated by \citet{li2026optimal} via the parameter magnitude, for $L^p$ approximation with a Lipschitz continuous activation; we will discuss it after Proposition 5.2.}

\paragraph{Elementary activations}
The activation attaining $d+1$ neurons is explicit but intricate: it is a piecewise-constant function that lists, block by block, the digits of every integer in every base (see \eqref{eq:dexa}). We therefore also consider a floor-exponential activation
\begin{equation}
\eact(t)=\begin{cases}2^{t}-1,&t<0,\\ \lfloor t\rfloor,&t\ge 0.\end{cases}
\label{eq:eact-intro}
\end{equation}
We call an activation \emph{elementary} if it is given by a finite formula in
finitely many standard functions; $\eact$ is elementary in this sense. Using
$\eact$ as the single shared activation, we give an explicit three-hidden-layer
network of widths $(d,1,2)$, hence $d+3$ hidden neurons, and a variant of widths
$(d,1,1)$ with one skip connection, hence $d+2$ hidden neurons
(Theorem~\ref{thm:elementary}). All parameters of both networks are rational
and given in closed form, which allows exact-arithmetic verification
(Section~\ref{sec:numerics}); to our knowledge, no earlier fixed-size
\rev{universal approximator} has this property\rev{; in $L^p$, the recent network of \citet{li2026optimal} also has closed-form parameters}. Elementary activations have appeared in
fixed-size constructions~\citep{yarotsky2021elementary,zhang2022deep}, but there the function-dependent weights are obtained from the density of an
irrational winding on the torus and are not given in closed form;
\citet{yarotsky2021elementary} notes that computing the weight is practically infeasible even for small problems. Moreover, those earlier constructions need $O(d^2)$ neurons when a single elementary activation is used. Table~\ref{tab:comparison}
summarizes the results.

\begin{table}[t]
\caption{Comparison of fixed-size universal approximators on
$[0,1]^d$. Unless ``two'' is indicated, a single activation is used. \rev{(E): elementary activation. Rows marked ($L^p$) approximate in $L^p([0,1]^d)$, $0<p<\infty$. Counts marked $^\dagger$ are derived by us from the published constructions. 
} $^{\ast}$Optimal among all locally integrable activations for $d\ge2$ (Theorem~\ref{thm:lower}).
}
\label{tab:comparison}
\centering
\footnotesize
\setlength{\tabcolsep}{2.5pt}
\vspace{-2mm}
\begin{tabular}{@{}llllll@{}}
\toprule
\textbf{Reference} & \textbf{Neurons} & \textbf{Depth} & \textbf{Activation} & \rev{\textbf{Parameters}} & \rev{\textbf{Values}}\\
\midrule
\citet{maiorov1999ridge} & \rev{$O(d^2)^{\dagger}$} & 2 & analytic sigmoidal & \rev{$O(d^2)^\dagger$} & existence\\
\citet{ismailov2014bounded} & \rev{$O(d^2)^{\dagger}$} & 2 & $C^\infty$ sigmoidal & \rev{$O(d^2)^\dagger$} & existence\\
\citet{guliyev2018two} & \rev{$O(d^2)^{\dagger}$} & 2 & computable $C^\infty$ sigmoidal & \rev{$O(d^2)^\dagger$} & existence\\
\citet{yarotsky2021elementary} & $d+2$ (skip)\rev{$^\dagger$} & 3\rev{$^\dagger$} & two: floor and analytic \rev{(E)} & \rev{$2d+5^\dagger$} & existence\\
 & $O(d^2)$ & $O(1)$\rev{$^\dagger$} & two: $\sin$ and $\arcsin$ \rev{(E)} & \rev{$O(d^2)$} & existence\\
 & $O(d^2)$ & $O(1)$\rev{$^\dagger$} & $C^1$ sigmoidal \rev{(E)} & \rev{$O(d^2)$} & existence\\
\citet{zhang2022deep} & $\leq 396d(2d+1)$ & 11 & $C^0$ \rev{EUAF} \rev{(E)} & \rev{$O(d^2)$} & existence\\
\rev{\citet{maiti2024optimal}} & \rev{$O(d^2)^\dagger$} & \rev{$11^\dagger$} & \rev{$C^0$ EUAF \rev{(E)}} & \rev{$O(d)$} & \rev{existence}\\
\rev{\citet{li2026optimal} ($L^p$)} & \rev{$3d+4$} & \rev{4} & \rev{Lipschitz DTA \rev{(E)}} & \rev{$\le8d+7$} & \rev{existence}\\
 & \rev{$3d+13$} & \rev{6} & \rev{Lipschitz DTA \rev{(E)}} & \rev{$\le8d+27$} & \rev{closed-form}\\
\midrule
Theorem~\ref{thm:exact}\rev{, $\dexa$} & $d+1$ \textbf{(optimal$^\ast$)} & 2 & explicit piecewise constant & \rev{$3d+3$} & closed-form\\
\rev{Theorem~\ref{thm:exact}, $\widetilde\dexa$} & \rev{$d+1$ \textbf{(optimal$^\ast$)}} & \rev{2} & \rev{explicit piecewise constant} & \rev{$3d+3$} & \rev{closed-form}\\
Theorem~\ref{thm:elementary} & $d+3$ & 3 & floor--exponential \rev{(E)} & \rev{$2d+8$} & closed-form\\
Theorem~\ref{thm:elementary} & $d+2$ (skip) & 3 & floor--exponential \rev{(E)} & \rev{$2d+6$} & closed-form\\
\rev{Theorem~\ref{thm:contLp} ($L^p$)} & \rev{$d+1$} & \rev{2} & \rev{explicit continuous} & \rev{$3d+3$} & \rev{closed-form}\\
\bottomrule
\end{tabular}
\vspace{-4mm}
\end{table}

We summarize our \textbf{main contributions} as follows.
\begin{enumerate}[leftmargin=*,itemsep=1pt,topsep=2pt]
\item \textbf{Optimal neuron count.}
We construct a network of widths $(d,1)$ with a single explicit activation
$\dexa$ in \eqref{eq:dexa} and rational, closed-form parameters that
approximates every $f\in\rev{C([0,1]^d)}$ to arbitrary $L^\infty$ accuracy
(Theorem~\ref{thm:exact}), and we verify it in exact arithmetic
(Section~\ref{sec:numerics}). We prove that $d+1$ is the exact minimum
neuron count, at any depth and for any locally integrable activation
(Theorem~\ref{thm:lower}). We further show that, at least for $d=2$, no
continuous activation attains this count in $L^\infty$
(Proposition~\ref{prop:cont}), whereas under the weaker $L^p$ criterion,
$\rev{0<p}<\infty$, a continuous modification of $\dexa$ that we construct
attains it for every $d\ge2$ (Theorem~\ref{thm:contLp}).

\item \textbf{Elementary activation.}
For the single elementary activation in~\eqref{eq:eact-intro}, we give an
explicit construction with $d+3$ hidden neurons and no skip
connections, together with a $d+2$-neuron variant using one skip
connection (Theorem~\ref{thm:elementary}). All parameters are rational
and given in closed form, again permitting exact-arithmetic evaluation.
\item \textbf{\rev{Optimal} storage complexity.}
We develop the \emph{bit complexity}, defined as the total binary length of the
rational network parameters, to quantify the information stored in a
fixed-size network. Our constructions \rev{are information-theoretically optimal or optimal up to a logarithmic factor over the H\"older class} (Theorem~\ref{thm:bits}\rev{, Proposition~\ref{prop:entropy}}).
\end{enumerate}
The rest of the paper is organized as follows:
Section~\ref{sec:setting} introduces the main concepts and notation.
Section~\ref{sec:optimal-neurons} establishes the optimal hidden-neuron
count by proving both sufficiency and necessity.
Section~\ref{sec:elementary} presents the constructions based on the
elementary floor--exponential activation.
Section~\ref{sec:bits} analyzes the bit complexity of the proposed
networks.
Section~\ref{sec:numerics} provides exact-arithmetic numerical
verification, and Section~\ref{sec:discussion} discusses limitations
and open problems. Additional proofs, extensions, numerical results,
and figures are provided in the appendices.

\section{Setting}
\label{sec:setting}
To formulate precisely the minimum-neuron problem and quantify the associated encoding cost, we first introduce the function class, network setting, and bit complexity, followed by the grid-quantization scheme underlying all our constructions.
\paragraph{H\"older class.}
Throughout, $d\in\N_+$, $\alpha\in(0,1]$ and $\lambda,R>0$. 
Consider
\begin{equation}
\Hol:=\bigl\{f\in C([0,1]^d):\|f\|_{\Linf}\le R,\ |f(\bx)-f(\by)|\le\lambda\|\bx-\by\|_\infty^{\alpha},\ \forall\bx,\by\bigr\}.
\label{eq:holder-class}
\end{equation}
Since $\Hol\subseteq\mathcal H^\alpha_{\lambda,\lceil R\rceil}([0,1]^d)$, we assume without loss of generality that $R\in\N_+$ whenever we need to make certain neural network parameters integers. 

\paragraph{Networks.}
An \emph{activation} is a function $\gact:\R\to\R$, applied componentwise to vectors. For the number of hidden layers $L\in\N_+$ and widths $\bm n=(n_1,\dots,n_L)\in\N_+^L$, a \emph{feedforward network} with input dimension $d$, activation $\gact$ and architecture $\bm n$ is a function $\Phi:[0,1]^d\to\R$ of the form
\begin{equation}
h_0(\bx)=\bx,\quad h_k(\bx)=\gact\bigl(W_kh_{k-1}(\bx)+\bb_k\bigr),\ k=1,\dots,L,\quad
\Phi(\bx)=\bm a^{\top}h_L(\bx)\rev{{}+b_{L+1}},
\label{eq:network}
\end{equation}
with $W_k\in\R^{n_k\times n_{k-1}}$ ($n_0=d$), $\bb_k\in\R^{n_k}$, $\bm a\in\R^{n_L}$, $\rev{b_{L+1}}\in\R$. We denote the set of all such networks by $\Net(\bm n;\gact)$. 
We also consider the larger class $\Net^{\mathrm{skip}}(\bm n;\gact)$ in which the output layer is affine in all hidden layers, $\Phi(\bx)=\sum_{k=1}^{L}\bm a_k^{\top}h_k(\bx)\rev{{}+b_{L+1}}$; the input still enters only through the first hidden layer. The number of \emph{hidden neurons} is $|\bm n|:=n_1+\dots+n_L$, the \emph{depth} is $L$, and the \emph{width} is $\max_k n_k$.
\begin{definition}[\rev{Universality} and minimal neuron count]
\label{def:fixed-size}
\rev{The network class $\Net(\bm n;\gact)$ is \emph{universal} if it is dense in $C([0,1]^d)$ in the uniform norm; the same definition applies to $\Net^{\mathrm{skip}}(\bm n;\gact)$. The \emph{minimal neuron count} is
\begin{equation*}
\nu_d:=\min\bigl\{|\bm n|:\ \exists L\in\N_+,\ \bm n\in\N_+^L,\ \gact\in L^1_{\mathrm{loc}}(\R)\ \text{such that}\ \Net(\bm n;\gact)\ \text{is universal}\bigr\}.
\end{equation*}}
\vspace{-4mm}
\end{definition}
Only the architecture and the activation are fixed in Definition~\ref{def:fixed-size}; the weights and biases may depend on $f$ and $\varepsilon$. 
All activations used in our constructions are locally bounded and hence locally integrable.

\paragraph{Bit complexity.}
For a nonnegative integer $k$ let $\bit(k):=\lceil\log_2(k+1)\rceil$ be its binary length; and for a rational number $\theta=a/b$ in lowest terms ($a\in\Z$, $b\in\N_+$), let $\bit(\theta):=1+\bit(|a|)+\bit(b)$, the additional bit recording the sign. If all weights and biases of a network $\Phi$ are rational, define its \emph{bit complexity} by
\begin{equation}
\bit(\Phi):=\sum_{\theta\in\mathcal P(\Phi)}\bit(\theta),
\label{eq:bit-complexity}
\end{equation}
where $\mathcal P(\Phi)$ denotes the set of nonzero weights and biases of $\Phi$. 

\begin{figure}[t]
\centering
\begin{subfigure}[t]{0.32\textwidth}\centering
\includegraphics[scale=0.31]{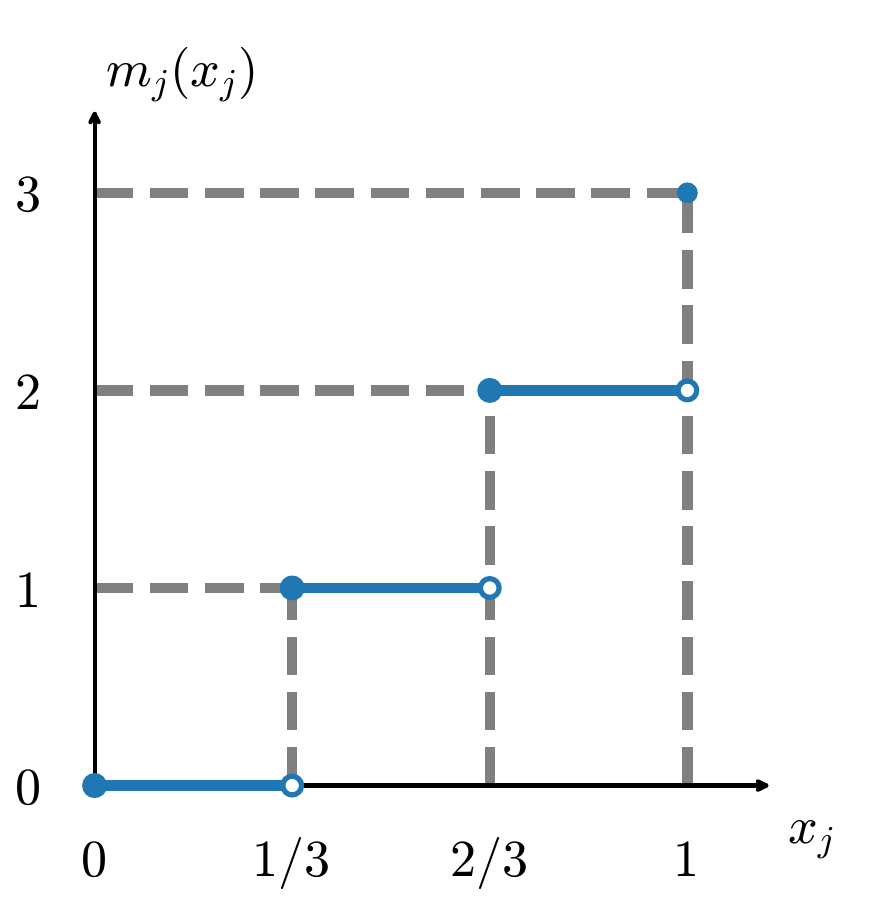}
\caption{$m_j(x_j)=\lfloor Mx_j\rfloor$ maps $[\frac iM,\frac{i+1}M)$ to $i$ and $1$ to $M$.}
\label{subfig:mj}
\end{subfigure}\hfill
\begin{subfigure}[t]{0.32\textwidth}\centering
\includegraphics[scale=0.31]{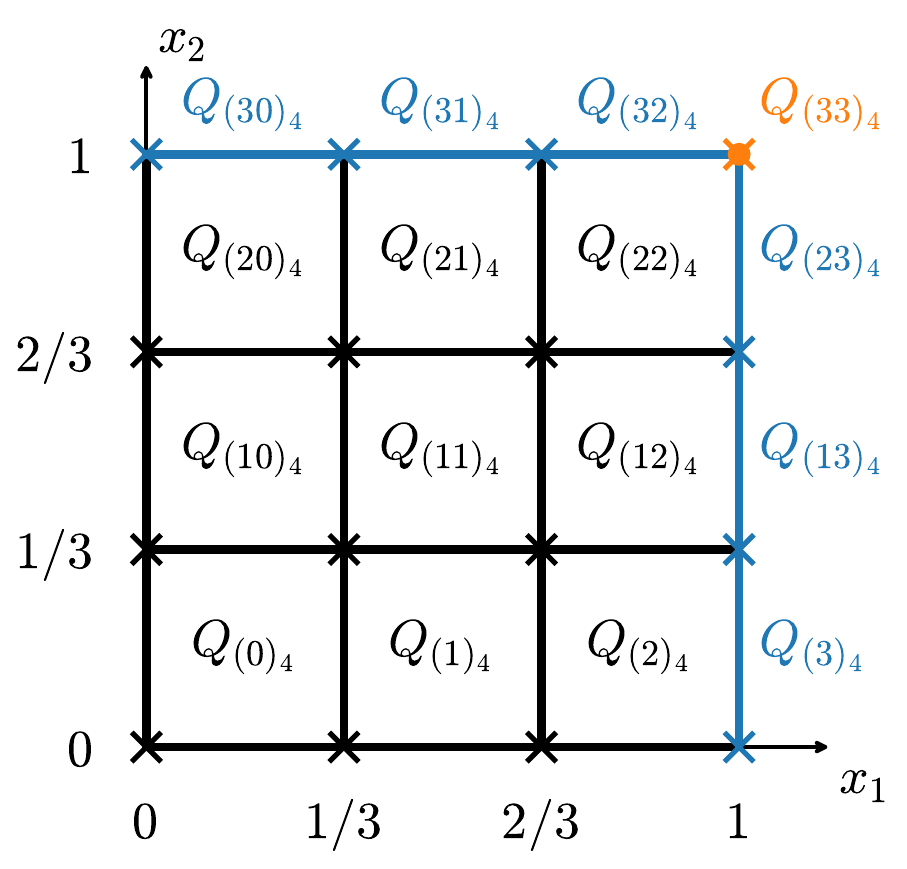}
\caption{$r(\bx)$ for $d=2$: each cell $Q_\ell$ is mapped to $\ell$; crosses mark the representatives $\bx^{(\ell)}$.}
\label{subfig:r}
\end{subfigure}\hfill
\begin{subfigure}[t]{0.32\textwidth}\centering
\includegraphics[scale=0.31]{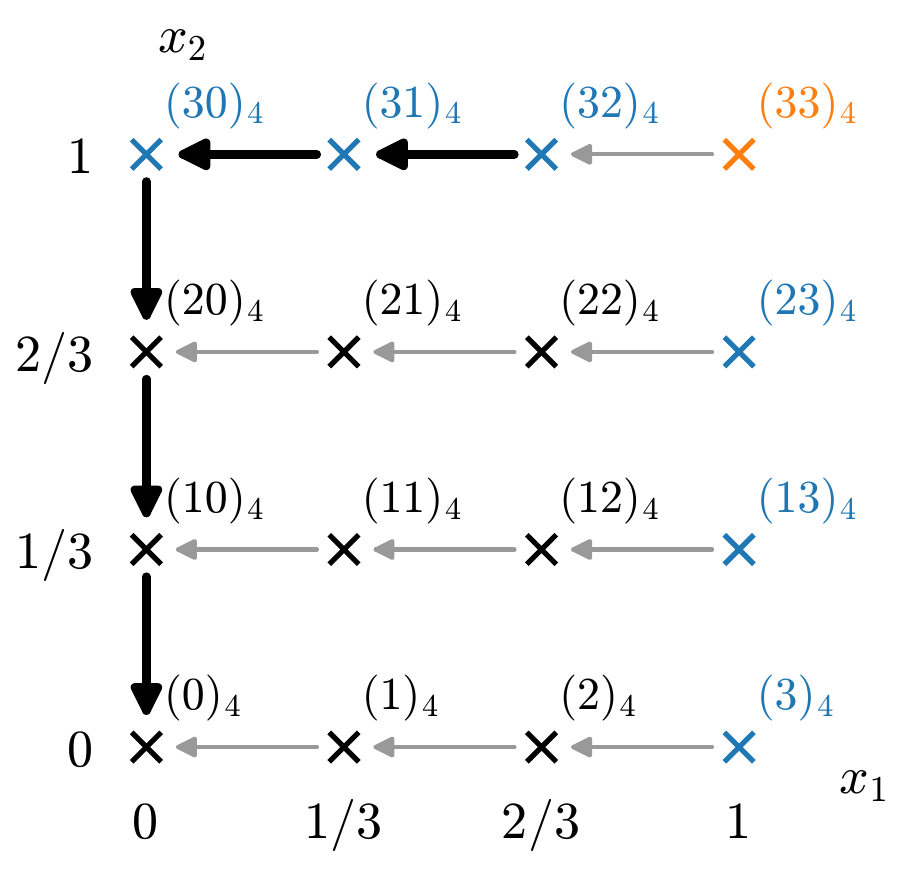}
\caption{\rev{Spanning tree: arrows point from $\ell$ to $\operatorname{p}(\ell)$; bold: the path from $(32)_4$ to $0$.}}
\label{subfig:tree}
\end{subfigure}
\caption{Position encoding with $M=3$\rev{, and the spanning tree of Section~\ref{subsubsec:difference-encoding}}.}
\label{fig:position_encoding}
\vspace{-3mm}
\end{figure}

\paragraph{Grid quantization (for all constructions).}
Fix $M\in\N_+$ and set $K:=(M+1)^d$. For $\bx\in[0,1]^d$ let $m_j(x_j):=\lfloor Mx_j\rfloor\in\{0,\dots,M\}$\rev{;} the \emph{cell address} $\binoppenalty=10000 r(\bx):=\sum_{j=1}^{d}(M+1)^{j-1}m_j(x_j)\in\{0,\dots,K-1\}$
\rev{is} the base-$(M+1)$ integer whose digits are the coordinates $m_j(x_j)$. For $\ell\in\{0,\dots,K-1\}$, the \emph{cell} $Q_\ell:=\{\bx\in[0,1]^d:r(\bx)=\ell\}$ is a half-open box (degenerate on $x_j=1$) and can be represented by the grid point $\bx^{(\ell)}:=\bigl(m^{(\ell)}_1/M,\dots,m^{(\ell)}_d/M\bigr)$ whose coordinates $m^{(\ell)}_j=\lfloor\ell/(M+1)^{j-1}\rfloor\bmod(M+1)$ are the digits of $\ell$. See Figure~\ref{fig:position_encoding} and Table~\ref{tab:position_encoding} in Appendix~\ref{app:fig_tab} for an example of $M=3$. Then $\|\bx-\bx^{(\ell)}\|_\infty\le 1/M$ for all $\bx\in Q_\ell$. 
Given \rev{$f\in C([0,1]^d)$, $R\in\N_+$ with $\|f\|_{\Linf}\le R$, and} $\varepsilon>0$, we choose
\begin{equation}
\delta:=2^{-\lceil\log_2(2/\varepsilon)\rceil}\in\bigl(\tfrac{\varepsilon}{4},\tfrac{\varepsilon}{2}\bigr],\quad
\rev{M\in\N_+:\ \omega_f(\tfrac1M)\le\tfrac\varepsilon2},\quad
B:=\max\bigl\{2,\,2^{\lceil\log_2(\lfloor 2R/\delta\rfloor+1)\rceil}\bigr\},
\label{eq:parameters}
\end{equation}
\rev{where $\omega_f(t):=\sup\{|f(\bx)-f(\by)|:\|\bx-\by\|_\infty\le t\}$ is the modulus of continuity of $f$;} $\delta$ is dyadic and $B$ is a power of two.
For \rev{such} $f$, we define the quantization indices $q_\ell$ and the
corresponding quantized approximations $\widehat f_\ell$ of
$f(\bx^{(\ell)})$ by
\begin{equation}
q_\ell:=\Bigl\lfloor\frac{f(\bx^{(\ell)})+R}{\delta}\Bigr\rfloor\in\{0,\dots,B-1\},\quad \widehat f_\ell:=-R+\delta q_\ell,\quad 0\le f(\bx^{(\ell)})-\widehat f_\ell<\delta,
\label{eq:quantized-values}
\end{equation}
which satisfy, for every $\bx\in Q_\ell$,
\begin{equation}
|f(\bx)-\widehat f_\ell|\le|f(\bx)-f(\bx^{(\ell)})|+|f(\bx^{(\ell)})-\widehat f_\ell|<\rev{\omega_f(1/M)}+\delta\le\varepsilon.
\label{eq:basic-error}
\end{equation}
The construction below is a network that outputs $\widehat f_{r(\bx)}$, or equivalently
$\widehat f_\ell$ when $\bx\in Q_\ell$. 
Moreover, \rev{for $f\in\Hol$, $\omega_f(1/M)\le\lambda M^{-\alpha}$, so we may take}
\begin{equation}
\rev{M:=\Bigl\lceil\bigl(\tfrac{2\lambda}{\varepsilon}\bigr)^{1/\alpha}\Bigr\rceil}=\Theta(\varepsilon^{-1/\alpha}),\quad K=\Theta(\varepsilon^{-d/\alpha}),\quad B=\Theta(\varepsilon^{-1})\quad \text{as}~\varepsilon\to0.
\label{eq:parameter-rates}
\end{equation}

\section{Optimal neuron count: $d+1$ hidden neurons}
\label{sec:optimal-neurons}
\subsection{Sufficiency: an explicit $(d,1)$ network}
\label{subsec:upper}
Our goal is to realize the piecewise-constant approximation $\widehat f_{r(\bx)}=-R+\delta q_{r(\bx)}$. Our strategy is first to determine the cell containing $\bx$ through its address $r(\bx)$, and then to recover $q_{r(\bx)}$ from a single stored integer\rev{; the two constructions below differ in how this integer encodes $q_0,\dots,q_{K-1}$}.
\rev{\subsubsection{Construction by digit extraction}\label{subsubsec:digit-extraction}}
\rev{We first} pack $q_0,\dots,q_{K-1}$ into a single integer as digits in a suitable base, so that $q_{r(\bx)}$ can be recovered by extracting the digit indexed by $r(\bx)$. This motivates an activation $\dexa$ whose negative branch provides the floor operation needed for addressing and whose nonnegative branch performs digit extraction.
\begin{figure}[h]
\centering
\includegraphics[width=0.7\linewidth]{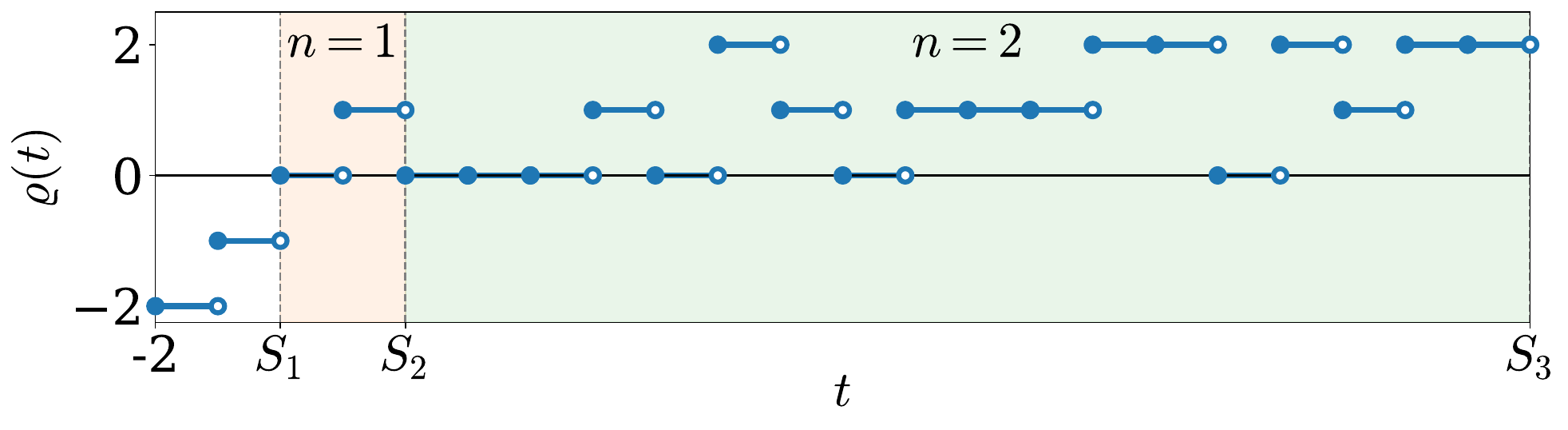}
\vspace{-4mm}
\caption{$\varrho$ in $[-2,S_3)$, or equivalently, $[-2,0)\bigcup\big(\bigcup_{n=1}^2\bigcup_{j=0}^{(n+1)^n-1}
    \bigcup_{r=0}^{n-1}
    I_{n,j,r}\big)$.}
\label{fig:varrho}
\end{figure}
\paragraph{The digit-extraction activation.}
For $n\in\N_+$, let $S_1:=0$ and
$S_n:=\sum_{k=1}^{n-1}k(k+1)^k$, so that
$S_{n+1}-S_n=n(n+1)^n$ and $S_n\to\infty$. The half-line
$[0,\infty)$ is the disjoint union of the blocks
$[S_n,S_{n+1})$, $n\in\N_+$, and the $n$-th block is the disjoint
union of the unit intervals
\begin{equation}
I_{n,j,r}:=[S_n+nj+r,S_n+nj+r+1),
\quad
j\in\{0,\ldots,(n+1)^n-1\},
\quad
r\in\{0,\ldots,n-1\}.
\label{eq:unit-intervals}
\end{equation}
Write $j=\sum_{s=0}^{n-1}b_s(j)(n+1)^{n-1-s}$ with digits $b_s(j)\in\{0,\dots,n\}$. Define
\begin{equation}
\dexa(t):=\begin{cases}\lfloor t\rfloor,& t<0,\\ b_r(j), & t\in I_{n,j,r}.\end{cases}
\label{eq:dexa}
\end{equation}
Equivalently, on $I_{n,j,r}$ the value $b_r(j)$ is $\lfloor j/(n+1)^{n-1-r}\rfloor-(n+1)\lfloor j/(n+1)^{n-r}\rfloor$. As shown in Figure \ref{fig:varrho}, the function $\dexa$ is piecewise constant on countably many intervals, satisfies $0\le\dexa(t)\le n$ on $[S_n,S_{n+1})$, and is therefore locally bounded and locally integrable. Its defining property is the \emph{digit-extraction identity}
\begin{equation}
\dexa(S_n+nj+r)=b_r(j),\quad j\in\{0,\ldots,(n+1)^n-1\},
\quad
r\in\{0,\ldots,n-1\}.
\label{eq:digit-extraction}
\end{equation}

\paragraph{The network.}
Let $f\in\rev{C([0,1]^d)}$, $\varepsilon>0$, and let $M,\delta,B,q_\ell$ be as in~\eqref{eq:parameters}--\eqref{eq:quantized-values}. Set
\begin{equation}
n:=\max\{K,B-1\},\quad
J_f:=\sum_{\ell=0}^{K-1}q_\ell\,(n+1)^{n-1-\ell},\quad
N_f:=S_n+nJ_f .
\label{eq:Jf-Nf}
\end{equation}
Since $q_\ell\le B-1\le n$, the base-$(n+1)$ digits of $J_f$ are $b_\ell(J_f)=q_\ell$ for $\ell<K$ and $0$ for $K\le\ell<n$; in particular $0\le J_f<(n+1)^n$, and~\eqref{eq:digit-extraction} gives $\dexa(N_f+\ell)=q_\ell$ for $0\le\ell<K$. 

As illustrated in Figure~\ref{fig:nn2}, the outputs
of the hidden layers and the affine output layer are defined as:
\begin{align}
u_j(\bx)
:=&\dexa\bigl(Mx_j-(M+1)\bigr)
=\lfloor Mx_j\rfloor-(M+1),
\quad j=1,\dots,d,
&&\text{(1st hidden layer)}
\label{eq:exact-layer1}\\
z(\bx)
:=&\dexa\Bigl(
N_f+\sum_{j=1}^{d}(M+1)^{j-1}
\bigl(u_j(\bx)+M+1\bigr)
\Bigr)
&&\smash{\raisebox{-0.5\baselineskip}{%
    \text{(2nd hidden layer)}}}
\notag\\
=&\dexa\bigl(N_f+r(\bx)\bigr)
=q_{r(\bx)}
&&
\label{eq:exact-layer2}\\
\Phi^{\dexa}_{f,\varepsilon}(\bx)
:=&-R+\delta z(\bx)
=\widehat f_{r(\bx)}.
&&\text{(output layer)}
\label{eq:exact-output}
\end{align}
In~\eqref{eq:exact-layer1} the argument of $\dexa$ is at most $-1$, so the negative branch (the floor) applies; in~\eqref{eq:exact-layer2} the argument is $N_f+r(\bx)\in[S_n,S_{n+1})$ and the digit-extraction branch applies. \rev{The accuracy of $\Phi^{\dexa}_{f,\varepsilon}$ is stated in Theorem~\ref{thm:exact}, together with that of the construction with tree difference encoding (Section~\ref{subsubsec:difference-encoding}).}

\begin{figure}[t]
\centering
\includegraphics[width=0.9\linewidth]{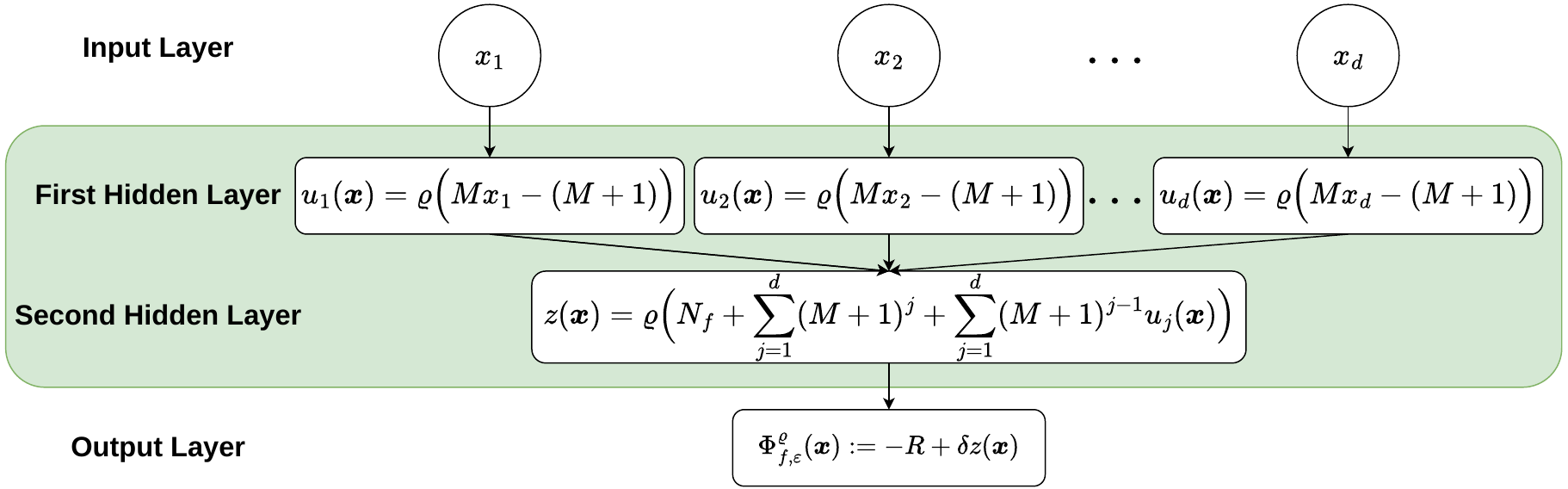}
\caption{Architecture of the network
$\Phi^\varrho_{f,\varepsilon}$ in Theorem~\ref{thm:exact}.
The $d$ neurons in the first hidden layer quantize the input coordinates,
whose weighted combination forms the cell address $r(\bx)$; the neuron in
the second hidden layer decodes the corresponding quantized value from
the stored integer $N_f$. \rev{$\Phi^{\widetilde\dexa}_{f,\varepsilon}$ (Section~\ref{subsubsec:difference-encoding}) has the same architecture, with $\widetilde\dexa$, $\widetilde N_f$, $-R+\delta q_0$ in place of $\dexa$, $N_f$, $-R$.}}
\label{fig:nn2}
\vspace{-3mm}
\end{figure}

\rev{\subsubsection{Construction with tree difference encoding}\label{subsubsec:difference-encoding}
Inspired by the idea of storing differences of neighboring values \citep{yarotsky2018optimal,li2026optimal}, we now encode differences between quantization indices along a fixed spanning tree of the grid. These differences are small: if $\bx^{(\ell)}$ and $\bx^{(\ell')}$ differ by $1/M$ in one coordinate, then $|f(\bx^{(\ell)})-f(\bx^{(\ell')})|\le\omega_f(1/M)\le\varepsilon/2<2\delta$ by~\eqref{eq:parameters}, so $|q_\ell-q_{\ell'}|\le2$, and each difference plus two is one base-$5$ digit. With a redesigned lookup activation that sums these differences along tree paths, as detailed below, the resulting network retains the hidden-layer widths $(d,1)$.

\paragraph{Tree difference encoding.}
For $\ell\neq0$, define the \emph{parent} $\operatorname{p}(\ell):=\ell-(M+1)^{i(\ell)-1}$ with $i(\ell):=\min\{j:m^{(\ell)}_j>0\}$. It decreases the first positive digit of $\ell$ by one, so $\bx^{(\operatorname{p}(\ell))}$ is a neighbor of $\bx^{(\ell)}$ (Figure~\ref{subfig:tree}). Repeated parent steps reach the root $0$ after $\sum_jm^{(\ell)}_j\le dM$ steps; let $\mathcal A(\ell)$ be the set of nonzero addresses on this path ($\mathcal A(0):=\varnothing$; in Figure~\ref{subfig:tree}, $\mathcal A(14)=\{14,13,12,8,4\}$). We encode $\Delta_\ell:=q_\ell-q_{\operatorname{p}(\ell)}\in\{-2,\dots,2\}$ in $\widetilde J_f$; summing $\Delta$ along a path telescopes:
\begin{equation}
\widetilde J_f:=\sum_{\ell=1}^{K-1}(\Delta_\ell+2)\,5^{\ell-1}\in\{0,\dots,5^{K-1}-1\},
\quad
\sum_{k\in\mathcal A(\ell)}\Delta_k=q_\ell-q_0,\quad 0\le \ell<K.
\label{eq:Jf-diff}
\end{equation}

\paragraph{The tree-difference activation.}
We split $[0,\infty)$ into consecutive blocks, one for each grid $(M,d)\in\N_+^2$, ordered first by $K=(M+1)^d$ and then by $d$. The block of $(M,d)$ consists of the $K5^{K-1}$ unit intervals $\binoppenalty=10000 \widetilde I_{M,d,j,\ell}:=[\widetilde S_{M,d}+Kj+\ell,\,\widetilde S_{M,d}+Kj+\ell+1)$, $0\le j<5^{K-1}$, $0\le \ell<K$, where $\widetilde S_{M,d}$ is the total length of the preceding blocks. Write $j=\sum_{k=1}^{K-1}\widetilde b_k(j)5^{k-1}$ with digits $\widetilde b_k(j)\in\{0,\dots,4\}$; on $\widetilde I_{M,d,j,\ell}$ the activation $\widetilde\dexa\in L^1_{\mathrm{loc}}(\R)$ (see Figure~\ref{fig:varrho_tilde}) extracts the digits of $j$ at the addresses $k\in\mathcal A(\ell)$ of the $(M,d)$ grid, subtracting $2$ from each and summing:
\begin{equation}
\widetilde\dexa(t):=\begin{cases}\lfloor t\rfloor,& t<0,\\ \sum_{k\in\mathcal A(\ell)}\bigl(\widetilde b_k(j)-2\bigr), & t\in\widetilde I_{M,d,j,\ell},\end{cases}
\label{eq:dexa-diff}
\end{equation}

\paragraph{The network.}
Set $\widetilde N_f:=\widetilde S_{M,d}+K\widetilde J_f$ with $\widetilde J_f$ as in~\eqref{eq:Jf-diff}. Then $\widetilde N_f+\ell\in\widetilde I_{M,d,\widetilde J_f,\ell}$ and $\widetilde b_k(\widetilde J_f)=\Delta_k+2$, so by~\eqref{eq:Jf-diff}--\eqref{eq:dexa-diff} the layers of $\Phi^{\widetilde\dexa}_{f,\varepsilon}$ (Figure~\ref{fig:nn2}) are
\begin{align}
u_j(\bx)
=&\widetilde\dexa\bigl(Mx_j-(M+1)\bigr)
=\lfloor Mx_j\rfloor-(M+1),
\quad j=1,\dots,d,
&&\text{(1st hidden layer)}
\label{eq:tree-layer1}\\
\widetilde z(\bx)
:=&\widetilde\dexa\Bigl(
\widetilde N_f+\sum_{j=1}^{d}(M+1)^{j-1}
\bigl(u_j(\bx)+M+1\bigr)
\Bigr)
&&\smash{\raisebox{-0.5\baselineskip}{%
    \text{(2nd hidden layer)}}}
\notag\\
=&\widetilde\dexa\bigl(\widetilde N_f+r(\bx)\bigr)
=q_{r(\bx)}-q_0,
&&
\label{eq:tree-layer2}\\
\Phi^{\widetilde\dexa}_{f,\varepsilon}(\bx)
:=&-R+\delta q_0+\delta\widetilde z(\bx)
=\widehat f_{r(\bx)}.
&&\text{(output layer)}
\label{eq:tree-output}
\end{align}
\noeqref{eq:tree-layer2}By~\eqref{eq:exact-output} and~\eqref{eq:tree-output}, $\Phi^{\dexa}_{f,\varepsilon}$ and $\Phi^{\widetilde\dexa}_{f,\varepsilon}$ both output $\widehat f_{r(\bx)}$. Since the pointwise estimate~\eqref{eq:basic-error} holds on every $Q_\ell$ and $\bigcup_{\ell=0}^{K-1}Q_\ell=[0,1]^d$, taking the supremum over $\bx\in[0,1]^d$ yields:}

\begin{theorem}[Two hidden layers: $d+1$ neurons]
\label{thm:exact}
Let $d\in\N_+$. For every $f\in\rev{C([0,1]^d)}$ and every $\varepsilon>0$, the networks $\Phi^{\dexa}_{f,\varepsilon}\in\Net\bigl((d,1);\dexa\bigr)$ defined by~\eqref{eq:exact-layer1}--\eqref{eq:exact-output} \rev{and $\Phi^{\widetilde\dexa}_{f,\varepsilon}\in\Net\bigl((d,1);\widetilde\dexa\bigr)$ defined by~\eqref{eq:tree-layer1}--\eqref{eq:tree-output}} satisfy
\begin{equation*}
\|f-\Phi^{\dexa}_{f,\varepsilon}\|_{\Linf}\le\varepsilon\rev{,\quad\|f-\Phi^{\widetilde\dexa}_{f,\varepsilon}\|_{\Linf}\le\varepsilon}.
\end{equation*}
\rev{In particular, $\Net((d,1);\dexa)$ and $\Net((d,1);\widetilde\dexa)$ are universal, and $\nu_d\le d+1$.}
\end{theorem}
All parameters (see details in Appendix~\ref{app:exact}) are integers except the dyadic output weight $\delta$ and the output bias\rev{, which is} $-R$ \rev{for $\Phi^{\dexa}_{f,\varepsilon}$ and $-R+\delta q_0$ for $\Phi^{\widetilde\dexa}_{f,\varepsilon}$}; if, without loss of generality, we take $R\in\N_+$\rev{, the former is an integer and the latter is dyadic}. \rev{For fixed $M$ and $R$ (e.g., $M$ as in \eqref{eq:parameter-rates} for $f\in\Hol$),} the \rev{function $f$ enters $\Phi^{\dexa}_{f,\varepsilon}$} only \rev{through} the second-layer bias, \rev{via} $N_f$, which encodes the quantized values of $f$ required for accuracy $\varepsilon$\rev{, and enters $\Phi^{\widetilde\dexa}_{f,\varepsilon}$ through both the second-layer bias, via $\widetilde N_f$, and the output bias, via $q_0$}. For fixed $d$, the architecture is independent of $f$ and $\varepsilon$, while the \rev{activations} $\dexa$ \rev{and $\widetilde\dexa$ are} independent of all problem parameters.

\subsection{Necessity: no network with $d$ hidden neurons is universal}
\label{subsec:lower}
In this section, we prove that at least $d+1$ hidden neurons are necessary \rev{for universality} when $d\ge2$. 
The proof relies on two activation-independent results: Proposition~\ref{prop:first-layer} shows that the first hidden layer must contain at least $d$ neurons, and Proposition~\ref{prop:shallow} shows that a single-hidden-layer network of fixed width is insufficient. 

\begin{proposition}[First-layer width obstruction]
\label{prop:first-layer}
Let $d\ge2$ and let
$\gact:\R\to\R$ be arbitrary. For every architecture
$\bm n=(n_1,\ldots,n_L)$ with $n_1<d$, there exist a function
$f_0\in\rev{C([0,1]^d)}$ and a constant $\varepsilon_0>0$ such that
\begin{equation}
\inf_{\Phi\in\Net^{\mathrm{skip}}(\bm n;\gact)}
\|f_0-\Phi\|_{\Linf}
\ge \varepsilon_0.
\label{eq:first-layer-bound}
\end{equation}
\end{proposition}
See the proof in Appendix~\ref{app:proof_first-layer}.
\begin{proposition}[Fixed-width single-hidden-layer obstruction]
\label{prop:shallow}
Let $d\ge2$ and $n_1\in\N_+$,
and let $\gact\in L^1_{\mathrm{loc}}(\R)$. There exists a function
$f_h\in\rev{C([0,1]^d)}$ and a constant $\varepsilon_1>0$ such that
\begin{equation}
\inf_{\Phi\in\Net((n_1);\gact)}
\|f_h-\Phi\|_{\Linf}
\ge \varepsilon_1.
\label{eq:shallow-bound}
\end{equation}
In particular, \rev{$\Net((n_1);\gact)$ is not universal}.
\end{proposition}
See the proof in Appendix~\ref{app:proof_shallow}. 
In fact, that fixed-width single-hidden-layer networks are never dense in $C([0,1]^d)$, $d\ge2$, is classical for continuous activations: it follows from the non-density of sums of $m$ ridge functions \citep{lin1993fundamentality,maiorov1999ridge,pinkus1999approximation}. We still include a direct proof because it is valid for arbitrary locally integrable $\gact$. 

Combining the results of the above two propositions, we obtain the following theorem:
\begin{theorem}[Lower bound of neurons]
\label{thm:lower}
Let $d\ge2$. Let
$L\in\N_+$, $\bm n\in\N_+^L$, and
$\gact\in L^1_{\mathrm{loc}}(\R)$. If
$\Net^{\mathrm{skip}}(\bm n;\gact)$
\rev{is universal} in the
sense of Definition~\ref{def:fixed-size}, then
$|\bm n|\ge d+1$.
Consequently, together with Theorem~\ref{thm:exact},
\begin{equation}
\rev{\nu_d=d+1\quad\text{for every }d\ge2.}
\end{equation}
\end{theorem}
See the proof in Appendix ~\ref{app:proof_lower}. 
The hypotheses $\gact\in L^1_{\mathrm{loc}}(\R)$ and $d\ge2$ are essential; see the two remarks below.
\begin{remark}[The case $d=1$]
The above results (Propositions~\ref{prop:first-layer}--\ref{prop:shallow} and Theorem~\ref{thm:lower}) need $d\ge2$. For $d=1$, \citet{guliyev2016single} construct an activation for which a single hidden neuron approximates every $f\in C([0,1])$ to arbitrary accuracy, so $\rev{\nu_1}=1$.
\end{remark}
\begin{remark}[Continuous activations and $L^p$ approximation]
The discontinuity of the activation cannot be dispensed with for uniform approximation
with the optimal neuron count $d+1$ established in
Theorems~\ref{thm:exact} and~\ref{thm:lower}, at least when $d=2$:
Proposition~\ref{prop:cont} shows that no network with three hidden
neurons, a continuous activation, and no skip connections can achieve
arbitrary accuracy in $L^\infty([0,1]^2)$.
If the error criterion is
weakened to $L^p([0,1]^d)$, $\rev{0<p}<\infty$, however, a $(d,1)$
architecture with a continuous activation can achieve arbitrary accuracy
(see Theorem~\ref{thm:contLp}).
\end{remark}

\section{Elementary activation: $d+3$ neurons (or $d+2$ with a skip)}
\label{sec:elementary}

The activation $\dexa$ in Theorem~\ref{thm:exact} performs digit
extraction in a single step by encoding every possible digit table in
its graph. 
We now replace this lookup mechanism with elementary radix
arithmetic. Store the quantized values as the base-$B$ integer $A_f=\sum_\ell q_\ell B^\ell$. Its $\ell$-th digit can be recovered by $q_\ell=\lfloor A_f/B^\ell\rfloor-B\lfloor A_f/B^{\ell+1}\rfloor$. 
Consequently, retrieving the digit indexed by $r(\bx)$ requires computing
$B^{-r(\bx)}$ and applying the floor function twice. This motivates an
elementary activation that combines exponential and floor operations as
in \eqref{eq:eact-intro}.

Let $f\in\rev{C([0,1]^d)}$, $\varepsilon>0$, $M,K,\delta,B,q_\ell$ as in~\eqref{eq:parameters}--\eqref{eq:quantized-values}, and
\begin{equation}
A_f:=\sum_{\ell=0}^{K-1}q_\ell B^{\ell}\in\{0,\dots,B^K-1\},\quad
T_0(\bx):=\Bigl\lfloor\frac{A_f}{B^{r(\bx)}}\Bigr\rfloor,\quad
T_1(\bx):=\Bigl\lfloor\frac{A_f}{B^{r(\bx)+1}}\Bigr\rfloor,
\label{eq:Af}
\end{equation}
so that $T_0(\bx)-BT_1(\bx)=q_{r(\bx)}$ (all base-$B$ digits at positions greater than $r(\bx)$ cancel). 
Therefore, we define the outputs of the hidden layers and the affine output layer of the elementary activation neural network with $\bc:=(1,M+1,\dots,(M+1)^{d-1})^{\top}$,
\begin{align}
\quad\quad \bu(\bx)&:=\eact(M\bx)=\bigl(\lfloor Mx_1\rfloor,\dots,\lfloor Mx_d\rfloor\bigr)^{\top}, &&\text{(1st hidden layer)}\label{eq:el-layer1}\\
\quad\quad v(\bx)&:=\eact\bigl(-1-\log_2(B)\,\bc^{\top}\bu(\bx)\bigr)
=\tfrac12B^{-r(\bx)}-1, &&\text{(2nd hidden layer)}\label{eq:el-layer2}\\
\quad\quad \bw(\bx)&:=\eact\Bigl(\begin{bmatrix}2A_f\\ 2A_f/B\end{bmatrix}\bigl(v(\bx)+1\bigr)\Bigr)
=\begin{bmatrix}T_0(\bx)\\ T_1(\bx)\end{bmatrix}, &&\text{(3rd hidden layer)}\label{eq:el-layer3}\\
\quad\quad \Phi^{\eact}_{f,\varepsilon}(\bx)&:=-R+\delta\,[1,\,-B]\,\bw(\bx)=-R+\delta\,q_{r(\bx)}=\widehat f_{r(\bx)}. &&\text{(output layer)}\label{eq:el-output}
\end{align}
Here $Mx_j\ge0$ in~\eqref{eq:el-layer1} selects the floor branch of $\sigma(t)$; the pre-activation part in~\eqref{eq:el-layer2} is $\le-1$ and selects the exponential branch, and $\log_2B\in\N$ because $B$ is a power of two; the pre-activations in~\eqref{eq:el-layer3} are nonnegative and select the floor branch. The network has widths $(d,1,2)$ (see Figure~\ref{fig:nn_basic}).

\paragraph{Saving one neuron with a skip connection.}
The first neuron at the third hidden layer~\eqref{eq:el-layer3} can be dispensed with if the output layer may read the second hidden layer directly, because $A_f/B^{r(\bx)}$ itself differs from $\lfloor A_f/B^{r(\bx)}\rfloor$ by the fractional part $\eta_{r(\bx)}:=A_fB^{-r(\bx)}-\lfloor A_fB^{-r(\bx)}\rfloor\in[0,1)$, whose contribution $\delta\eta_{r(\bx)}$ to the output is below the quantization step. Thus, keeping~\eqref{eq:el-layer1}--\eqref{eq:el-layer2} and setting
\begin{equation}
w^{\mathrm{skip}}(\bx):=\eact\bigl(\tfrac{2A_f}{B}(v(\bx)+1)\bigr)=T_1(\bx),
\quad
\Phi^{\eact,\mathrm{skip}}_{f,\varepsilon}(\bx):=-R+2\delta A_f\bigl(v(\bx)+1\bigr)-\delta B\,w^{\mathrm{skip}}(\bx),
\label{eq:skip-network}
\end{equation}
one obtains $\Phi^{\eact,\mathrm{skip}}_{f,\varepsilon}(\bx)=-R+\delta\bigl(A_fB^{-r(\bx)}-BT_1(\bx)\bigr)=\widehat f_{r(\bx)}+\delta\eta_{r(\bx)}$, a network in $\Net^{\mathrm{skip}}((d,1,1);\eact)$ with $d+2$ hidden neurons (see Figure~\ref{fig:nn_skip}).

These two networks also provide uniform approximation of \rev{every continuous function}:
\begin{theorem}[Elementary activation: $d+3$ neurons, or $d+2$ with a skip connection]
\label{thm:elementary}
Let $d\in\N_+$.
For every $f\in\rev{C([0,1]^d)}$ and every $\varepsilon>0$, the networks
$\Phi^{\eact}_{f,\varepsilon}\in\Net((d,1,2);\eact)$
and
$\Phi^{\eact,\mathrm{skip}}_{f,\varepsilon}
\in\Net^{\mathrm{skip}}((d,1,1);\eact)$,
defined respectively by~\eqref{eq:el-layer1}--\eqref{eq:el-output}
and~\eqref{eq:el-layer1}, \eqref{eq:el-layer2}, \eqref{eq:skip-network}
satisfy
\begin{equation}
\bigl\|f-\Phi^{\eact}_{f,\varepsilon}\bigr\|_{\Linf}\le\varepsilon
\quad\text{and}\quad
\bigl\|f-\Phi^{\eact,\mathrm{skip}}_{f,\varepsilon}\bigr\|_{\Linf}
\le\varepsilon.
\end{equation}
\end{theorem}

See the proof in Appendix~\ref{app:elementary}. Compared with the first construction of \citet{yarotsky2021elementary}, which also uses $d+2$ neurons and a skip connection, $\Phi^{\eact,\mathrm{skip}}_{f,\varepsilon}$ requires only a single shared activation, and all of its parameters are rational and given in closed form (see Table~\ref{tab:numerics-comparison}).

\rev{\begin{remark}[Vector-valued outputs]\label{rem:vector}
The universality in Theorems~\ref{thm:exact} and~\ref{thm:elementary} also extends to continuous vector-valued targets $\bm f=(f^{(1)},\ldots,f^{(d_{\mathrm{out}})})$, with the same activations and with errors measured in the Euclidean norm on $\R^{d_{\mathrm{out}}}$: with $\dexa$ or $\widetilde\dexa$, $d+d_{\mathrm{out}}$ hidden neurons suffice, which is the minimum among networks without skip connections for input dimension $d\ge2$, and with $\eact$, $d+d_{\mathrm{out}}+2$ hidden neurons suffice, or $d+d_{\mathrm{out}}+1$ with skip connections. Appendix~\ref{app:vector} gives these networks, the matching lower bound, the bit complexity, and an exact-arithmetic verification.
\end{remark}}

\section{Bit Complexity of Fixed-Size Approximators}
\label{sec:bits}
In each of our constructions, the target function is encoded by
a single integer, namely, $N_f$ in~\eqref{eq:Jf-Nf}, \rev{$\widetilde N_f$ in~\eqref{eq:tree-layer2},} or $A_f$
in~\eqref{eq:Af}\rev{; $\Phi^{\widetilde\dexa}_{f,\varepsilon}$ additionally stores $q_0$ in its output bias~\eqref{eq:tree-output}}. Although the number of neurons is fixed, the size of
this integer increases as the prescribed accuracy improves. A natural
measure of storage is therefore the bit complexity defined
in~\eqref{eq:bit-complexity}. This complexity is readily computed because all our network parameters are rational: $\delta$ \rev{and the output bias $-R+\delta q_0$ of $\Phi^{\widetilde\dexa}_{f,\varepsilon}$ are} dyadic, $B$ is a power of two, and the remaining quantities are integers for $R\in\N_+$. The preceding discussion leads to the following theorem:

\begin{theorem}[Bit complexity of the constructions]
\label{thm:bits}
Let $d\ge2$, $\alpha\in(0,1]$, $\lambda>0$, and $R\in\N_+$.
For the networks $\Phi^{\dexa}_{f,\varepsilon}$, \rev{$\Phi^{\widetilde\dexa}_{f,\varepsilon}$,} $\Phi^{\eact}_{f,\varepsilon}$, and $\Phi^{\eact,\mathrm{skip}}_{f,\varepsilon}$ in Theorems~\ref{thm:exact} and~\ref{thm:elementary}\rev{, with $M$ as in~\eqref{eq:parameter-rates}}, as $\varepsilon\to0$,
$\sup_{f\in\Hol}\bit\bigl(\Phi^{\dexa}_{f,\varepsilon}\bigr)
=\Theta\bigl(\varepsilon^{-d/\alpha}\log\frac1\varepsilon\bigr)$,
\rev{$\sup_{f\in\Hol}\bit\bigl(\Phi^{\widetilde\dexa}_{f,\varepsilon}\bigr)
=\Theta\bigl(\varepsilon^{-d/\alpha}\bigr)$,}
$\sup_{f\in\Hol}\bit\bigl(\Phi^{\eact}_{f,\varepsilon}\bigr)
=\Theta\bigl(\varepsilon^{-d/\alpha}\log\frac1\varepsilon\bigr)$,
$\sup_{f\in\Hol}\bit\bigl(\Phi^{\eact,\mathrm{skip}}_{f,\varepsilon}\bigr)
=\Theta\bigl(\varepsilon^{-d/\alpha}\log\frac1\varepsilon\bigr)$.
\end{theorem}
The proof is given in Appendix~\ref{app:bits}.
We next compare this complexity with the information-theoretic lower
bound. The classical Kolmogorov--Tikhomirov entropy estimate
\citep[Theorem~XIV, equation~(68)]{kolmogorov1959varepsilon} states
that the \emph{$\varepsilon$-entropy} of $\Hol$ in the uniform norm,
defined as the base-$2$ logarithm of the minimum cardinality of an
$\varepsilon$-cover of $\Hol$, is
$\Theta(\varepsilon^{-d/\alpha})$. Equivalently, the minimum
cardinality of such a cover is
$2^{\Theta(\varepsilon^{-d/\alpha})}$.
For a fixed architecture, any
collection of $N$ distinct rational-parameter networks has worst-case
bit complexity $\Omega(\log_2 N)$. Consequently, any family of
networks approximating the entire class to accuracy $\varepsilon$ must
have worst-case bit complexity
$\Omega(\varepsilon^{-d/\alpha})$. This yields the following result:
\begin{proposition}[Bit complexity lower bound]
\label{prop:entropy}
Let $d\in\N_+$, $\alpha\in(0,1]$, $\lambda,R>0$. Fix any $L\in\N_+$, $\bm n\in\N_+^L$, and activation $\gact:\R\to\R$.
For each $\varepsilon>0$ and $f\in\Hol$, let
$\Phi_{f,\varepsilon}^\phi\in\Net^{\mathrm{skip}}(\bm n;\gact)$ be any network with rational parameters satisfying $\|f-\Phi_{f,\varepsilon}^\phi\|_{\Linf}\le\varepsilon$.
Then
$\sup_{f\in\Hol}\bit(\Phi^\phi_{f,\varepsilon})=
\Omega(\varepsilon^{-d/\alpha})$ as 
$\varepsilon\to0$.
\end{proposition}

Combining Proposition~\ref{prop:entropy} with
Theorem~\ref{thm:bits}, \rev{$\Phi^{\dexa}_{f,\varepsilon}$, $\Phi^{\eact}_{f,\varepsilon}$ and $\Phi^{\eact,\mathrm{skip}}_{f,\varepsilon}$} are optimal \rev{in bit complexity} up to a
logarithmic factor\rev{, and $\Phi^{\widetilde\dexa}_{f,\varepsilon}$ is optimal}. In particular, the polynomial storage requirement
$\varepsilon^{-d/\alpha}$ is unavoidable \rev{even} for a fixed-size universal approximator. The ``curse of
memory'' identified by \citet{zhang2022deep} is therefore not a defect
of a particular construction but an information-theoretic necessity.

\rev{The bit complexity can come from either the magnitude or the precision of the parameters. A lower bound on the magnitude was recently obtained by \citet{li2026optimal}: for a fixed-size network with a globally H\"older continuous activation to reach $L^p$ accuracy $\varepsilon$ on $\Hol$, its largest weight or bias in absolute value, $T$, must satisfy $\log T\ge c\,\varepsilon^{-d/\alpha}$ for a constant $c>0$. Their proof needs the continuity of the activation: a small change of the parameters causes only a small change of the output, so the digits far after the binary point can be dropped, and parameters with few digits cannot distinguish all functions in $\Hol$; hence large parameters are needed. For a discontinuous activation, this argument does not apply, and a network could also use parameters with many digits after the binary point. Proposition~\ref{prop:entropy}, which counts these digits as well, holds for every activation. In our constructions, however, almost all the information is still stored in a single large integer.}

\rev{\begin{remark}[Scope of the optimality]\label{rem:modulus}
The optimality in bit complexity is established only over the H\"older class~\eqref{eq:holder-class}. For $f\in C([0,1]^d)$ with $\omega_f\le\omega$ for a fixed concave $\omega$, and $M$ in~\eqref{eq:parameters} the smallest integer with $\omega(1/M)\le\varepsilon/2$, the proof of Theorem~\ref{thm:bits} gives $\Theta(\omega^{-1}(\varepsilon/2)^{-d})$ bits for $\Phi^{\widetilde\dexa}_{f,\varepsilon}$, whereas the class has $\varepsilon$-entropy of order $\omega^{-1}(2\varepsilon)^{-d}$ \citep{lorentz1966metric}, so the bit complexity of $\Phi^{\widetilde\dexa}_{f,\varepsilon}$ matches the $\varepsilon$-entropy when $\omega^{-1}(2\varepsilon)/\omega^{-1}(\varepsilon/2)$ stays bounded, as for $\omega(t)=\lambda t^\alpha$, and $\Phi^{\widetilde\dexa}_{f,\varepsilon}$ is not optimal when this ratio is unbounded, e.g., for $\omega(t)=1/\log(1/t)$. For smoothness $s>1$, our networks, being constant on each cell $Q_\ell$ rather than using a local Taylor expansion, need at least of order $\varepsilon^{-d}$ bits and are not optimal, as the $\varepsilon$-entropy is only of order $\varepsilon^{-d/s}$ \citep{kolmogorov1959varepsilon}.
\end{remark}}

\section{Numerical verification in exact arithmetic}
\label{sec:numerics}
The explicit \rev{constructions} of $\Phi_{f,\varepsilon}^\varrho$ \rev{and $\Phi^{\widetilde\dexa}_{f,\varepsilon}$} in Theorem~\ref{thm:exact} \rev{permit}
exact-arithmetic evaluation. We carry this out for $d=2,3$ to provide
a computational check of the \rev{constructions}.

We consider four functions $f_1,\ldots,f_4$ on $[0,1]^2$ and two
functions $f_5,f_6$ on $[0,1]^3$, with certified
H\"older parameters $(\alpha,\lambda,R)$ in the
$\ell^\infty$ norm:
\begin{align*}
f_1(\bx)&=\tfrac12\sin(\pi x_1)\cos(\pi x_2), & (\alpha,\lambda,R)&=(1,\pi,1);\\
f_2(\bx)&=|x_1-\tfrac12|^{1/2}+\tfrac12|x_2-x_1|^{1/2}, & (\alpha,\lambda,R)&=(\tfrac12,\tfrac74,2);\\
f_3(\bx)&=\exp\bigl[\sin(\pi p)\cos(\pi p)\bigr]\log\bigl[1+\tfrac{p^2}{2+p}\bigr]+\tfrac{\sin[3\pi(x_1+x_2)]}{1+p^2}, & (\alpha,\lambda,R)&=(1,19,1);\\
f_4(\bx)&=e^{-3\,\mathrm{dist}(\bx,\Gamma)}\cos\bigl(4\pi\,\mathrm{dist}(\bx,\Gamma)\bigr), & (\alpha,\lambda,R)&=(1,18.3,1);\\
f_5(\bx)&=\tfrac12\sin(\pi x_1x_2)\cos(\pi x_3)+\tfrac14\bigl|x_2-x_1x_3\bigr|, & (\alpha,\lambda,R)&=(1,\tfrac{39}{10},1);\\
f_6(\bx)&=\tfrac12\cos(\pi\,\|\bx-\bc\|_2),\quad \bc=(0.3,0.6,0.2),& (\alpha,\lambda,R)&=(1,\tfrac{11}{4},1).
\end{align*}
For $f_3$, $p:=x_1x_2$. For $f_4$, $\Gamma$ denotes the boundary of the level-$3$ Koch snowflake obtained by applying three iterations of the classical outward Koch construction~\citep{koch1904courbe} to an
upward-pointing equilateral triangle centered at
$(\frac12,\frac12)$ with circumradius $0.36$, and define
$\mathrm{dist}(\bx,\Gamma):=\min_{\by\in\Gamma}\|\bx-\by\|_2$.

We perform the verification over a sequence of dyadic accuracies starting from $\varepsilon=2^{-1}$, with the finest tested accuracy chosen so that $\bit(N_f)$ is on the order of $10^7$ bits. Table~\ref{tab:numerics-comparison} in Appendix \ref{app:fig_tab} reports the parameters $\delta,M,K=(M+1)^d,B$ from~\eqref{eq:parameters}, the bit length $\bit(N_f)$ from~\eqref{eq:Jf-Nf}, \rev{the bit length $\bit(\widetilde N_f)$ from~\eqref{eq:tree-layer2},} and the maximal error $\max|f_i-\Phi^{\dexa}_{f_i,\varepsilon}|$ over $3{,}004$ exactly represented test points for $d=2$ and $3{,}008$ for $d=3$, comprising $2{,}000$ random dyadic points, $1{,}000$ boundary points, and the $2^d$ corners. The visualizations for \rev{$f_1,\ldots,f_6$, namely the targets}, the network outputs \rev{$\Phi^{\dexa}_{f_i,\varepsilon}$}, and the pointwise errors on a $161^2$ display grid at the three common accuracy levels $\varepsilon=2^{-1},2^{-2},2^{-3}$, together with the maximal errors over the full sequence of tested accuracies, are provided in Appendix~\ref{app:fig_tab} (\rev{Figures~\ref{fig:numerics-f4}--\ref{fig:numerics-errors}}). All observed errors are below $\varepsilon$ as guaranteed, and \rev{the ratios} $\bit(N_f)/(n\log_2(n+1))$ (for $\Phi^\dexa_{f,\varepsilon}$)\rev{, $\bit(\widetilde N_f)/(K\log_25)$ (for $\Phi^{\widetilde\dexa}_{f,\varepsilon}$),} and $\bit(A_f)/(K\log_2B)$ (for $\Phi^\sigma_{f,\varepsilon}$ and $\Phi^{\sigma,\mathrm{skip}}_{f,\varepsilon}$) approach $1$ as $\varepsilon$ decreases in every case, consistently with Theorem~\ref{thm:bits}.

\paragraph{Implementation.}
The dyadic quantities and nonintegral rational parameters are stored exactly as \texttt{Fraction} objects; $R,M,B,K,n$ and the integer parameters as Python integers; $(q_\ell)_{\ell=0}^{K-1}$ as a temporary \texttt{int64} array; and the large integers $S_n,J_f,N_f$ \rev{and $\widetilde S_{M,d},\widetilde J_f,\widetilde N_f$} as arbitrary-precision \texttt{mpz} objects \rev{of \texttt{gmpy2}}. Arbitrary-precision arithmetic is necessary because $\bit(N_f)$ reaches
the order of $10^7$, far beyond the precision of double precision. \rev{Appendix~\ref{app:cost} gives the operation counts and measured times for building and evaluating the networks.} 

\section{Discussion}
\label{sec:discussion}

\paragraph{Limitations.}
Our results concern representational capacity rather than practical
training. \rev{A minimal neuron count does not mean a small bit complexity or a cheap evaluation.} The encoding integers $N_f$ and $A_f$ may require
$\Theta(\varepsilon^{-d/\alpha}\log(1/\varepsilon))$ bits\rev{, and $\widetilde N_f$ still $\Theta(\varepsilon^{-d/\alpha})$ bits}, far exceeding
standard fixed-precision floating-point arithmetic in the gradient descent-based training in actual function fitting. Moreover, in the $(d,1)$ constructions, $\dexa$ \rev{and $\widetilde\dexa$ are lookup functions} rather than practical nonlinearities, and their evaluation
can be computationally expensive.

\paragraph{Open problems.}
(i) Can the optimal count $d+1$ be attained with an elementary
discontinuous activation, or can $d+2$ neurons suffice without a
skip connection?
(ii) What is the minimal neuron count for uniform approximation with a continuous activation since $d+1$ neurons do not suffice when $d=2$?
\rev{(iii) For smoothness $s>1$, can a fixed-size network attain the optimal bit complexity of order $\varepsilon^{-d/s}$?}

\bibliography{refs}
\bibliographystyle{plainnat}

\appendix
\section*{Appendix}
\section{Details of Section~\ref{subsec:upper}}
\label{app:exact}

\paragraph{The activation $\varrho$ is well defined.}
Fix $n\in\N_+$. The unit intervals $I_{n,j,r}$ of~\eqref{eq:unit-intervals}, indexed by integers $j$ and $r$ such that $0\le j<(n+1)^n$ and $0\le r<n$, form a partition of
$[S_n,S_{n+1})$.
Indeed, let $k=nj+r$ ranging over $\{0,\dots,n(n+1)^n-1\}$,
\begin{equation}
    \bigcup_{j=0}^{(n+1)^n-1}
    \bigcup_{r=0}^{n-1}
    I_{n,j,r}
    =
    \bigcup_{k=0}^{n(n+1)^n-1}
    [S_n+k,S_n+k+1) =
    [S_n,S_n+n(n+1)^n) =
    [S_n,S_{n+1}).
\label{eq:activation-block-partition}
\end{equation}
Since $S_1=0$ and $S_n\to\infty$, it follows that
$    \bigcup_{n=1}^{\infty}
    \bigcup_{j=0}^{(n+1)^n-1}
    \bigcup_{r=0}^{n-1}
    I_{n,j,r}=[0,\infty).$
Therefore, $\dexa$ in~\eqref{eq:dexa} is well defined on $\R$. It is constant on each unit interval, takes values in $\{0,\dots,n\}$ on $[S_n,S_{n+1})$, and every compact set meets finitely many blocks; hence $\dexa$ is locally bounded and Borel measurable, so $\dexa\in L^1_{\mathrm{loc}}(\R)$.

\paragraph{Digit extraction identity of the activation $\varrho$.}
Write $j\in\{0,\dots,(n+1)^n-1\}$ with base-$(n+1)$ expansion $j=\sum_{s=0}^{n-1}b_s(n+1)^{n-1-s}$, $b_s\in\{0,\dots,n\}$. For $r\in\{0,\dots,n-1\}$,
\begin{equation}
\Bigl\lfloor\frac{j}{(n+1)^{n-1-r}}\Bigr\rfloor=\sum_{s=0}^{r}b_s(n+1)^{r-s},\quad
\Bigl\lfloor\frac{j}{(n+1)^{n-r}}\Bigr\rfloor=\sum_{s=0}^{r-1}b_s(n+1)^{r-1-s},
\end{equation}
Subtracting $n+1$ times the second from the first leaves $b_r$, i.e.
$\lfloor j/(n+1)^{n-1-r}\rfloor-(n+1)\lfloor j/(n+1)^{n-r}\rfloor=b_r$.
Since $S_n+nj+r\in I_{n,j,r}$ as defined in \eqref{eq:unit-intervals}, this proves~\eqref{eq:digit-extraction}.

\paragraph{Parameters of the network $\Phi_{f,\varepsilon}^\varrho$.}
For the record, the parameters of $\Phi^{\dexa}_{f,\varepsilon}$ are: $W_1=M\,\mathrm{Id}_d$, $\bb_1=-(M+1)\bm 1$; $W_2=\bigl(1,M+1,\dots,(M+1)^{d-1}\bigr)$, $b_2=N_f+(M+1)\sum_{j=1}^d(M+1)^{j-1}=N_f+(M+1)\frac{(M+1)^d-1}{M}$; $a=\delta$, $\rev{b_3}=-R$. All are integers except $\delta$ (dyadic) and $-R$ (an integer once $R\in\N_+$).

\rev{%
\paragraph{The activation $\widetilde\dexa$ is well defined.}
Fix a grid $(M,d)\in\N_+^2$ and recall that $K=(M+1)^d$. The unit intervals $\widetilde I_{M,d,j,\ell}$ in~\eqref{eq:dexa-diff} are indexed by $j\in\{0,\dots,5^{K-1}-1\}$, the $5^{K-1}$ possible values of $\widetilde J_f$ in~\eqref{eq:Jf-diff}, and by the $K$ addresses $\ell\in\{0,\dots,K-1\}$. Hence $k=Kj+\ell$ ranges over $\{0,\dots,K5^{K-1}-1\}$, each value once, and
\[
\bigcup_{j=0}^{5^{K-1}-1}\bigcup_{\ell=0}^{K-1}\widetilde I_{M,d,j,\ell}=\bigcup_{k}[\widetilde S_{M,d}+k,\widetilde S_{M,d}+k+1)=[\widetilde S_{M,d},\widetilde S_{M,d}+K5^{K-1}).
\]
Since $\widetilde S_{1,1}=0$, each block starts where the preceding one ends, and each grid has finitely many preceding grids ($(M'+1)^{d'}\le K$ forces $M'\le K-1$ and $d'\le\log_2K$), it follows that $\bigcup_{(M,d)\in\N_+^2}[\widetilde S_{M,d},\widetilde S_{M,d}+K5^{K-1})=[0,\infty)$ with $K=(M+1)^d$. Therefore, $\widetilde\dexa$ in~\eqref{eq:dexa-diff} is well defined on $\R$. It is constant on each unit interval, satisfies $|\widetilde\dexa|\le2dM$ on the block of $(M,d)$ because $|\widetilde b_k(j)-2|\le2$ and $\mathcal A(\ell)$ has at most $dM$ elements, and every compact set meets finitely many blocks, so $\widetilde\dexa$ is locally bounded and Borel measurable and therefore lies in $L^1_{\mathrm{loc}}(\R)$.

\paragraph{Parameters of the network $\Phi^{\widetilde\dexa}_{f,\varepsilon}$.}
They are those of $\Phi^{\dexa}_{f,\varepsilon}$ above, except $b_2=\widetilde N_f+(M+1)\frac{(M+1)^d-1}{M}$ and $\rev{b_3}=-R+\delta q_0$; all are integers except the dyadic numbers $\delta$ and $\rev{b_3}$.
}

\section{Proofs of Section~\ref{subsec:lower}}
\label{app:lower}

\subsection{Proof of Proposition~\ref{prop:first-layer}}\label{app:proof_first-layer}
\rev{Fix $\alpha\in(0,1]$ and $\lambda,R>0$.} Define $f_0:[0,1]^d\to\R$ by $f_0(\bx):=\zeta\|\bx-\bx_0\|_\infty^\alpha$, $\bx_0:=\left(\tfrac12,\ldots,\tfrac12\right)$, $\zeta:=\min\{\lambda,2^\alpha R\}$. 

\paragraph{Step 1: Verification that $f_0\in\Hol$.} Since $\max_{\bx\in[0,1]^d}\|\bx-\bx_0\|_\infty=\frac12$, we have $\|f_0\|_{\Linf}=\zeta2^{-\alpha}\le R$. For $a,b\ge0$ and $\alpha\in(0,1]$, one has
$|a^\alpha-b^\alpha|\le|a-b|^\alpha$ by the subadditivity of $t\mapsto t^\alpha$. Hence, by the reverse triangle inequality,
\begin{align*}
|f_0(\bx)-f_0(\by)|&=\zeta\bigl|\|\bx-\bx_0\|_\infty^\alpha-\|\by-\bx_0\|_\infty^\alpha\bigr|
\le\zeta\bigl|\|\bx-\bx_0\|_\infty-\|\by-\bx_0\|_\infty\bigr|^\alpha\\
&\le\zeta\|\bx-\by\|_\infty^\alpha\le\lambda\|\bx-\by\|_\infty^\alpha .
\end{align*}

\paragraph{Step 2: Error lower bound from the kernel of the first layer.} The input enters $\Phi$ only through $W_1\bx+\bb_1$ with $W_1\in\R^{n_1\times d}$. Since $n_1<d$, there exists $\bz\in\ker W_1$ with $\|\bz\|_\infty=1$; put $\bx_1:=\bx_0+\frac12\bz\in[0,1]^d$. Then $W_1\bx_1+\bb_1=W_1\bx_0+\bb_1$, hence $\Phi(\bx_1)=\Phi(\bx_0)$, whereas $f_0(\bx_0)=0$ and $f_0(\bx_1)=\zeta2^{-\alpha}$. Therefore
\begin{equation}
    \zeta2^{-\alpha}
    =
    |f_0(\boldsymbol{x}_1)-f_0(\boldsymbol{x}_0)| \le
    |f_0(\boldsymbol{x}_1)-\Phi(\boldsymbol{x}_1)|
    +
    |\Phi(\boldsymbol{x}_0)-f_0(\boldsymbol{x}_0)| \le
    2\|f_0-\Phi\|_{L^\infty([0,1]^d)}.
\end{equation}
Therefore, by definition of $\zeta$, $\|f_0-\Phi\|_{L^\infty([0,1]^d)}
\ge
\frac{1}{2}\zeta2^{-\alpha} =
\frac{1}{2}
\min\bigl\{\lambda2^{-\alpha},R\bigr\}
=:\varepsilon_0>0$.\qed

\subsection{Proof of Proposition~\ref{prop:shallow}}\label{app:proof_shallow}
\rev{Fix $\alpha\in(0,1]$ and $\lambda,R>0$.} Let $\Omega:=(0,1)^d$ and define
$h:\R^d\to\R$, $h(\bx):=\exp(\|\bx\|_2^2)$.
We write $(C_c^\infty(\Omega))'$ for the continuous dual of $C_c^\infty(\Omega)$.

\paragraph{Step 1: a suitable directional derivative of $\Phi$ vanishes.}
Let $\Phi(\bx)=\rev{b_2+{}}\sum_{i=1}^{n_1}
a_i\gact(\bw_i\cdot\bx\rev{{}+b_{1,i}})
\in\Net((n_1);\gact)$.
Since $\gact\in L^1_{\mathrm{loc}}(\R)$, each function
$a_i\gact(\bw_i\cdot\bx\rev{{}+b_{1,i}})$ is locally integrable on $\R^d$ and
can therefore be viewed as a regular distribution on $\Omega$.

Denote $\mathbb S^{d-1}:=\{\bx\in\R^d:\|\bx\|_2=1\}$. For each $i=1,\ldots,n_1$, choose a unit vector
$\bv_i\in\mathbb S^{d-1}$ satisfying $\bv_i\cdot\bw_i=0$.
Such a vector exists because $d\ge2$; if $\bw_i=0$, we may choose any unit vector. Write
$D_{\bv}:=\bv\cdot\nabla$
for the distributional directional derivative.

For every test function $\psi\in C_c^\infty(\Omega)$, extended by
zero to $\R^d$, we have
$\left\langle
D_{\bv_i}
\bigl[a_i\gact(\bw_i\cdot\bx\rev{{}+b_{1,i}})\bigr],
\psi\right\rangle
=-\int_{\R^d}
a_i\gact(\bw_i\cdot\bx\rev{{}+b_{1,i}})
D_{\bv_i}\psi(\bx)\,d\bx$.    
Using the orthogonal decomposition
$\bx=\by+s\bv_i$,
$\by\in\bv_i^\perp$,
whose Jacobian is one, and the identity
$\bw_i\cdot(\by+s\bv_i)=\bw_i\cdot\by$
gives
\begin{equation}
\big\langle
D_{\bv_i}
\bigl[a_i\gact(\bw_i\cdot\bx\rev{{}+b_{1,i}})\bigr],
\psi
\big\rangle
=
-\int_{\bv_i^\perp}
a_i\gact(\bw_i\cdot\by\rev{{}+b_{1,i}})
\Big[
\int_{\R}
\frac{\partial}{\partial s}
\psi(\by+s\bv_i)\,ds
\Big]d\by
=0.    
\end{equation}
The inner integral vanishes because $\psi$ has compact support.
Therefore,
\begin{equation}
D_{\bv_i}
\bigl[a_i\gact(\bw_i\cdot\bx\rev{{}+b_{1,i}})\bigr]
=0
\quad\text{in }(C_c^\infty(\Omega))'.
\label{eq:ridge-directional-zero}
\end{equation}

Now define
$\mathcal L_{\bv}:=
D_{\bv_1}D_{\bv_2}\cdots D_{\bv_{n_1}}$,
$\bv:=(\bv_1,\ldots,\bv_{n_1})
\in(\mathbb S^{d-1})^{n_1}$.
Constant-coefficient directional derivatives commute. For each
$i$, the operator $\mathcal L_{\bv}$ contains the factor
$D_{\bv_i}$, which makes the $i$th ridge term vanish by
\eqref{eq:ridge-directional-zero}. It also makes the constant term
vanish. Hence
$\mathcal L_{\bv}\Phi=0$
in $(C_c^\infty(\Omega))'$.

\paragraph{Step 2: the same directional derivative of $h$ does not vanish.}
For every $\bu\in\R^d$,
$D_{\bu}h(\bx)=2(\bu\cdot\bx)h(\bx)$.
Repeated application of the product rule gives
$\mathcal L_{\bv}h=P_{\bv}h$,
where $P_{\bv}$ is a polynomial of degree $n_1$ whose homogeneous
part of degree $n_1$ is
$2^{n_1}\prod_{i=1}^{n_1}(\bv_i\cdot\bx)$.
Since each $\bv_i$ is a unit vector,
$\bv_i\cdot\bx$ is a nonzero linear polynomial. Therefore,
$\prod_{i=1}^{n_1}(\bv_i\cdot\bx)$
is a nonzero polynomial. It follows that
$P_{\bv}\not\equiv0$, and hence
\begin{equation}
\mathcal L_{\bv}h=P_{\bv}h\not\equiv0
\label{eq:h-derivative-nonzero}
\end{equation}
for every
$\bv\in(\mathbb S^{d-1})^{n_1}$.

\paragraph{Step 3: derive a uniform lower bound for the approximation error.} Choose $\chi\in C_c^\infty(\Omega)$ such that $\chi\ge0$ and
$\chi>0$ on a nonempty open subset of $\Omega$, and define
$A_{n_1}(\bv):=\int_\Omega\chi(\bx)
\big|\mathcal L_{\bv}h(\bx)\big|^2\,d\bx$.
The function $\mathcal L_{\bv}h$ is real analytic and, by
\eqref{eq:h-derivative-nonzero}, is not identically zero. It
therefore cannot vanish on the open set where $\chi>0$, and hence $A_{n_1}(\bv)>0$.
Moreover, $\bv\mapsto A_{n_1}(\bv)$ is continuous. Since
$(\mathbb S^{d-1})^{n_1}$ is compact,
\begin{equation}
a_{n_1}
:=
\min_{\bv\in(\mathbb S^{d-1})^{n_1}}
A_{n_1}(\bv)
>0.
\label{eq:a-n1-definition}
\end{equation}

For the directions $\bv_1,\ldots,\bv_{n_1}$ selected in Step~1,
define
$\psi_{\bv}:=\chi\,\mathcal L_{\bv}h
\in C_c^\infty(\Omega)$.
Since $\mathcal L_{\bv}\Phi=0$
in $(C_c^\infty(\Omega))'$, we obtain
$A_{n_1}(\bv)
=
\left\langle
\mathcal L_{\bv}h,\psi_{\bv}
\right\rangle
=
\left\langle
\mathcal L_{\bv}(h-\Phi),\psi_{\bv}
\right\rangle
=
(-1)^{n_1}
\int_\Omega
(h-\Phi)\mathcal L_{\bv}\psi_{\bv}\,d\bx$.    
Therefore,
\begin{equation}
A_{n_1}(\bv)
\le
\|h-\Phi\|_{L^\infty(\Omega)}
\|\mathcal L_{\bv}\psi_{\bv}\|_{L^1(\Omega)}.
\label{eq:testing-bound}
\end{equation}

By the Leibniz rule, $\mathcal L_{\bv}\psi_{\bv}=\sum_{S\subseteq\{1,\dots,n_1\}}\bigl(\prod_{i\in S}D_{\bv_i}\bigr)\chi\cdot\bigl(\prod_{i\notin S}D_{\bv_i}\bigr)\mathcal L_{\bv}h$ involves derivatives of the fixed smooth functions $\chi$ and $h$ of order at most $2n_1$ with coefficients depending polynomially in $\bv$. Hence $(\bx,\bv)\mapsto\mathcal L_{\bv}\psi_{\bv}(\bx)$ is continuous on the compact set $\supp\chi\times(\mathbb S^{d-1})^{n_1}$. It follows that
$0<b_{n_1}
:=
\sup_{\bv\in(\mathbb S^{d-1})^{n_1}}
\|\mathcal L_{\bv}\psi_{\bv}\|_{L^1(\Omega)}
<\infty$.
Combining \eqref{eq:testing-bound} with
\eqref{eq:a-n1-definition} gives
\begin{equation}
\|h-\Phi\|_{L^\infty(\Omega)}
\ge
\frac{A_{n_1}(\bv)}{b_{n_1}}
\ge
\frac{a_{n_1}}{b_{n_1}}
>0,\quad \text{for every}~\Phi\in\Net((n_1);\gact).
\label{eq:h-shallow-lower-bound}
\end{equation}

\paragraph{Step 4: rescale $h$ into the H\"older class.}
On $[0,1]^d$,
$\|h\|_\infty=e^d$
and
$\|\nabla h(\bx)\|_1=
2e^{\|\bx\|_2^2}\sum_{j=1}^d x_j
\le 2de^d$.
Therefore, by the mean value theorem,
$|h(\bx)-h(\by)|
\le
2de^d\|\bx-\by\|_\infty
\le
2de^d\|\bx-\by\|_\infty^\alpha$,
where the last inequality follows from
$\|\bx-\by\|_\infty\le1$ and $\alpha\le1$.

Set $\gamma:=
\min\left\{Re^{-d},
\frac{\lambda}{2de^d}
\right\}$,
$f_h:=\gamma h$.
Then $f_h\in\Hol$. Since $\Net((n_1);\gact)$ is invariant under
multiplication by the positive constant $\gamma$, it follows from
\eqref{eq:h-shallow-lower-bound} that
\begin{equation}
\inf_{\Phi\in\Net((n_1);\gact)}
\|f_h-\Phi\|_{\Linf}=\gamma
\inf_{\Phi\in\Net((n_1);\gact)}
\|h-\Phi\|_{\Linf}\ge
\gamma\frac{a_{n_1}}{b_{n_1}}
=:\varepsilon_1>0.    
\end{equation}\qed
\subsection{Proof of Theorem~\ref{thm:lower}}\label{app:proof_lower}
Suppose, to the contrary, that $|\bm n|\le d$. If $L\ge2$, then
$n_1\le |\bm n|-n_2\le d-1<d$,
so Proposition~\ref{prop:first-layer} exhibits an $f_0\in\rev{C([0,1]^d)}$ that
cannot be approximated by any
$\Phi\in\Net^{\mathrm{skip}}(\bm n;\gact)$ with error smaller than
$\varepsilon_0$. If $L=1$, then
$\Net^{\mathrm{skip}}(\bm n;\gact)=\Net((n_1);\gact)$, and
Proposition~\ref{prop:shallow} exhibits an $f_h\in\rev{C([0,1]^d)}$ that cannot
be approximated by any such network with error smaller than
$\varepsilon_1$. Both cases contradict \rev{universality}.
Therefore, $|\bm n|\ge d+1$.

Together with the upper bound provided by
Theorem~\ref{thm:exact}, this yields $\rev{\nu_d}=d+1$.\qed
\section{Details and Proofs of Section~\ref{sec:elementary}}
\label{app:elementary}

\paragraph{Parameters of the networks $\Phi^{\eact}_{f,\varepsilon}$ and $\Phi^{\eact,\mathrm{skip}}_{f,\varepsilon}$.}
$\Phi^{\eact}_{f,\varepsilon}$: $W_1=M\,\mathrm{Id}_d$, $\bb_1=0$; $W_2=-\log_2(B)\,\bc^{\top}$, $b_2=-1$; $W_3=(2A_f,\,2A_f/B)^{\top}$, $\bb_3=(2A_f,\,2A_f/B)^{\top}$; $\bm a=(\delta,-\delta B)$, $\rev{b_4}=-R$.

$\Phi^{\eact,\mathrm{skip}}_{f,\varepsilon}$: the same first two layers; $W_3=b_3=2A_f/B$; output weights $2\delta A_f$ on $v$ and $-\delta B$ on $w^{\mathrm{skip}}$, bias $-R+2\delta A_f$. All parameters are integers or dyadic rationals.

\subsection{Proof of Theorem~\ref{thm:elementary}}
\begin{proof}
For $\Phi^{\eact}_{f,\varepsilon}$, \eqref{eq:el-output} and~\eqref{eq:basic-error} give $|f(\bx)-\Phi^{\eact}_{f,\varepsilon}(\bx)|<\varepsilon$ on every cell. 

For $\Phi^{\eact,\mathrm{skip}}_{f,\varepsilon}$, on $Q_\ell$ we have $\Phi^{\eact,\mathrm{skip}}_{f,\varepsilon}=\widehat f_\ell+\delta\eta_\ell$ with $0\le\delta\eta_\ell<\delta$ and $0\le f(\bx^{(\ell)})-\widehat f_\ell<\delta$ by~\eqref{eq:quantized-values}, so $|f(\bx^{(\ell)})-\Phi^{\eact,\mathrm{skip}}_{f,\varepsilon}(\bx)|=|(f(\bx^{(\ell)})-\widehat f_\ell)-\delta\eta_\ell|<\delta$ and the triangle inequality with $|f(\bx)-f(\bx^{(\ell)})|\le\rev{\omega_f(1/M)}\le\varepsilon/2$ gives the claim. 
\end{proof}
\section{Proof of Theorem \ref{thm:bits}}
\label{app:bits}

Throughout the proof, we consider $d\ge2$, $\alpha\in(0,1]$, $\lambda>0$, $R\in\N_+$ are fixed and $\varepsilon\to0$; by~\eqref{eq:parameter-rates}, $M=\Theta(\varepsilon^{-1/\alpha})$, $K=\Theta(\varepsilon^{-d/\alpha})$, $B=\Theta(\varepsilon^{-1})$, and $\log_2B=\Theta(\log(1/\varepsilon))$.

\subsection{Bit complexity of $A_f$}
\paragraph{Upper bound.} Since $0\le q_\ell\le B-1$, $0\le A_f=\sum_{\ell=0}^{K-1}q_\ell B^{\ell}\le(B-1)\sum_{\ell=0}^{K-1} B^{\ell}=B^K-1$, hence $\bit(A_f)\le\lceil\log_2B^K\rceil=K\log_2B$ (an integer, as $B$ is a power of two). Thus $\sup_{f\in\Hol}\bit(A_f)=O(\varepsilon^{-d/\alpha}\log(1/\varepsilon))$.

\paragraph{Lower bound.} Let $f_R\equiv R$. By~\eqref{eq:parameters}, $B/2<\lfloor2R/\delta\rfloor+1\le B$ as soon as $\lfloor2R/\delta\rfloor+1\ge2$, i.e.\ for small $\varepsilon$; since $B/2$ and $\lfloor2R/\delta\rfloor$ are integers, $B/2\le\lfloor2R/\delta\rfloor\le B-1$. For $f_R$ every digit equals $q_\ell=\lfloor2R/\delta\rfloor\ge B/2$, so 
\begin{equation}\label{eq:lb_A_fR}
A_{f_R}\ge\frac B2\sum_{\ell<K}B^\ell\ge\frac12B^K    
\end{equation}
and $\bit(A_{f_R})\ge K\log_2B-1$. Hence
\begin{equation}
\sup_{f\in\Hol}\bit(A_f)=\Theta\bigl(K\log_2B\bigr)=\Theta\bigl(\varepsilon^{-d/\alpha}\log\frac1\varepsilon\bigr).
\label{eq:app-Af-bits}
\end{equation}

\subsection{Bit complexity of $\Phi^{\eact}_{f,\varepsilon}$ and $\Phi^{\eact,\mathrm{skip}}_{f,\varepsilon}$}
Using the parameter lists in Appendix~\ref{app:elementary}: the first hidden layer has $d$ nonzero weights equal to $M$, contributing $d\,\bit(M)=O(d\log(1/\varepsilon))$. The second hidden layer has weights $\log_2(B)(M+1)^{j-1}$, $j=1,\dots,d$, and bias $-1$, with $\bit(\log_2(B)(M+1)^{j-1})=O(j\log(1/\varepsilon))$, contributing $O(d^2\log(1/\varepsilon))$ in total. The third hidden layer contains $2A_f$ and $2A_f/B$ (each at most twice); since $2A_f/B$ in lowest terms has numerator at most $2A_f$ and denominator at most $B$, $\bit(2A_f/B)\le\bit(A_f)+\bit(B)+O(1)$, so this layer contributes $O(\bit(A_f)+\log(1/\varepsilon))$. The output layer has parameters $\delta$, $\delta B$, $R$ (and $2\delta A_f$, $-R+2\delta A_f$ in the skip variant), contributing $O(\bit(A_f)+\log(1/\varepsilon))$. 
Altogether, since  
$d^2\log(1/\varepsilon)
=
o\big(
\varepsilon^{-d/\alpha}
\log(1/\varepsilon)
\big)$ as $\varepsilon\to0$,
$\sup_{f\in\Hol}\bit\bigl(\Phi^{\eact}_{f,\varepsilon}\bigr)=\Theta(\sup_{f\in\Hol}\bit\bigl(A_f\bigr))$ and $\sup_{f\in\Hol}\bit\bigl(\Phi^{\eact,\mathrm{skip}}_{f,\varepsilon}\bigr)=\Theta(\sup_{f\in\Hol}\bit\bigl(A_f\bigr))$. Thus,
\begin{equation}
\sup_{f\in\Hol}\bit\bigl(\Phi^{\eact}_{f,\varepsilon}\bigr)
=\Theta\bigl(\varepsilon^{-d/\alpha}\log\frac1\varepsilon\bigr),\
\sup_{f\in\Hol}\bit\bigl(\Phi^{\eact,\mathrm{skip}}_{f,\varepsilon}\bigr)
=\Theta\bigl(\varepsilon^{-d/\alpha}\log\frac1\varepsilon\bigr).
\end{equation}

\subsection{Bit complexity of $\Phi^{\dexa}_{f,\varepsilon}$}
By the parameter list at Appendix~\ref{app:exact}, all parameters other than the bias $b_2=N_f+(M+1)((M+1)^d-1)/M$ contribute $O(d^2\log(1/\varepsilon))$ bits, and $\bit(b_2)=\bit(N_f)+O(d\log(1/\varepsilon))$. Since $n=\max\{K,B-1\}$ and $K=\Theta(\varepsilon^{-d/\alpha})$ dominates $B=\Theta(\varepsilon^{-1})$ (as $d/\alpha\ge2$), we have $n=K=\Theta(\varepsilon^{-d/\alpha})$ for small $\varepsilon$. From $S_n\le N_f<S_{n+1}$ and
\[
(n-1)n^{\,n-1}\le S_n=\sum_{k=1}^{n-1}k(k+1)^k\le(n-1)\cdot(n-1)n^{\,n-1}\le n^{\,n+1},\quad S_{n+1}\le(n+1)^{\,n+2},
\]
we get $(n-1)\log_2n\le\log_2N_f\le(n+2)\log_2(n+1)$, i.e.\ $\bit(N_f)=\Theta(n\log n)=\Theta(\varepsilon^{-d/\alpha}\log(1/\varepsilon))$ for every $f\in\Hol$. Hence $\sup_f\bit(\Phi^{\dexa}_{f,\varepsilon})=\Theta(\varepsilon^{-d/\alpha}\log(1/\varepsilon))$.

\rev{%
\subsection{Bit complexity of $\Phi^{\widetilde\dexa}_{f,\varepsilon}$}
\label{app:bits-opt}
Using the parameter lists in Appendix~\ref{app:exact}, the parameters other than $b_2$ and $\rev{b_3}$ coincide with those of $\Phi^{\dexa}_{f,\varepsilon}$ and contribute $O(d^2\log(1/\varepsilon))$ bits, as in the previous subsection. Since $\rev{b_3}=(q_0-R/\delta)\,\delta$ with $0\le q_0\le2R/\delta$, $\bit(\rev{b_3})=O(\log(1/\varepsilon))$; moreover $\bit(b_2)=\bit(\widetilde N_f)+O(d\log(1/\varepsilon))$. For the upper bound, every grid $(M',d')$ whose block in~\eqref{eq:dexa-diff} precedes or equals that of $(M,d)$ has $K'=(M'+1)^{d'}\le K$. Since $M'+1\le K'$ and $2^{d'}\le K'$, this forces $M'\in\{1,\dots,K-1\}$ and $d'\in\{1,\dots,\lfloor\log_2K\rfloor\}$, so there are at most $(K-1)\lfloor\log_2K\rfloor\le K\log_2K$ such grids. Their blocks have total length $\widetilde S_{M,d}+K5^{K-1}$, and since $\widetilde J_f<5^{K-1}$ by~\eqref{eq:Jf-diff},
\[
\widetilde N_f=\widetilde S_{M,d}+K\widetilde J_f<\widetilde S_{M,d}+K5^{K-1}=\sum_{(M',d')}K'5^{K'-1}\le K\log_2K\cdot K5^{K-1},
\]
Hence $\bit(\widetilde N_f)\le\log_2\widetilde N_f+1<(K-1)\log_25+2\log_2K+\log_2\log_2K+1=K\log_25+O(\log K)$. For the lower bound, let $K\ge3$. Then $(K-2,1)\in\N_+^2$ is a grid with $K'=(K-2+1)^1=K-1<K$, so by the order in~\eqref{eq:dexa-diff} its block, of length $(K-1)5^{K-2}$, precedes that of $(M,d)$ and is counted in $\widetilde S_{M,d}$. Hence $\widetilde N_f\ge\widetilde S_{M,d}\ge(K-1)5^{K-2}$ and $\bit(\widetilde N_f)>(K-2)\log_25$. Therefore, for every $f\in\Hol$, as $\varepsilon\to0$,
\[
\bit\bigl(\Phi^{\widetilde\dexa}_{f,\varepsilon}\bigr)=K\log_25+O\bigl(\log(1/\varepsilon)\bigr)=\Theta(\varepsilon^{-d/\alpha})
\]
by~\eqref{eq:parameter-rates}. \qed
}

\rev{\subsection{VC dimension}\label{app:vc}}
The inadequacy of neuron count alone is also reflected in VC
dimension. 
Indeed, consider $\Net((d,1);\dexa)$ and any finite set of distinct points $\bx_1,\ldots,\bx_N\in[0,1]^d$ with arbitrary labels $y_1,\ldots,y_N\in\{-1,1\}$. There exists a bounded Lipschitz function $f$ such that $f(\bx_i)=y_i$ for every $i$. By Theorem~\ref{thm:exact}, there is a network $\Phi\in\Net((d,1);\dexa)$ satisfying $\|f-\Phi\|_{\Linf}\le1/2$, and hence $y_i\Phi(\bx_i)>0$, $i=1,\ldots,N$. Thus, the classifier $\operatorname{sign}\Phi$ realizes the prescribed labels. Since the points and labels were arbitrary, $\Net((d,1);\dexa)$ has infinite VC dimension. The same argument applies
to the network classes in Theorem~\ref{thm:elementary}. Hence, for
these superexpressive activations, the number of neurons or parameters
alone does not control statistical complexity.

\section{Continuous activations: uniform obstruction and $L^p$ construction}\label{app:continuous_Linf}
The activation $\dexa$ of Theorem~\ref{thm:exact} is discontinuous, while Theorem~\ref{thm:lower} allows arbitrary locally integrable activations. Appendix~\ref{app:continuous_obstruction} shows that the discontinuity is not incidental: with a continuous activation, $d+1$ hidden neurons do not suffice for uniform approximation, already for $d=2$. Appendix~\ref{app:continuous_Lp} shows that the obstruction is specific to the uniform norm: a single continuous activation $\dexac$, independent of $f$, $\varepsilon$, $p$ and $d$, gives the $(d,1)$ architecture an $L^p$ guarantee for every $p\in\rev{(0,}\ \infty)$.
\subsection{A uniform-approximation obstruction for $(2,1)$ networks}\label{app:continuous_obstruction}
\begin{proposition}[Three-neuron obstruction for continuous activations]
\label{prop:cont}
Let $d=2$ and let
$\gact:\R\to\R$ be continuous. For every $L\in\N_+$ and every
architecture $\bm n=(n_1,\ldots,n_L)\in\N_+^L$ with
$|\bm n|=3$, there exist a function $f\in\rev{C([0,1]^2)}$ and a constant
$\varepsilon_2>0$
such that
\begin{equation}
\inf_{\Phi\in\Net(\bm n;\gact)}
\|f-\Phi\|_{\Linf}
\ge \varepsilon_2.
\label{eq:cont-bound}
\end{equation}
Consequently, no fixed architecture with three hidden neurons and a
continuous activation \rev{is universal}.
\end{proposition}

\begin{proof}
For $|\bm n|=3$, the only possible architectures are $(3)$, $(1,2)$, $(2,1)$, $(1,1,1)$.
The architecture $(3)$ is excluded by Proposition~\ref{prop:shallow}, while $(1,2)$ and $(1,1,1)$ are excluded by Proposition~\ref{prop:first-layer}. It therefore remains to only consider $\bm n=(2,1)$\rev{, for which we construct $f\in\Hol$ for fixed $\alpha\in(0,1]$ and $\lambda,R>0$}.

\paragraph{Step 0: growth of the function and the collision condition.}
Let $\bx_0:=(\frac12,\frac12)$. For
$\bx\in[0,1]^2$, write
$\by:=\bx-\bx_0=(y_1,y_2)$, and define
\begin{equation}\label{eq:fc}
    f_{\mathrm c}(\bx):=\|\by\|_2+\frac1{10}\kappa(\by),\quad\kappa(\by):=y_1^3+2y_2^3.
\end{equation}
Since $\by\in[-\frac12,\frac12]^2$, we have
$|\kappa(\by)|
\le |y_1|^3+2|y_2|^3
\le \frac34\|\by\|_2$.
Thus, with $\theta:=3/40$,
\begin{equation}
(1-\theta)\|\bx-\bx_0\|_2
\le f_{\mathrm c}(\bx)
\le (1+\theta)\|\bx-\bx_0\|_2.
\label{eq:fc-growth}
\end{equation}
In particular, $f_{\mathrm c}(\bx_0)=0$.

Let $\Phi\in\Net((2,1);\gact)$. 
In the notation of~\eqref{eq:network},
write $\bm\eta^\top,\bm\xi^\top\in\R^2$ for the rows of $W_1$,
$\bb_1=(b_{1,1},b_{1,2})$, $W_2=(w_1,w_2)$, $\bb_2=b_2$, and
$\bm a=a$. Then $\Phi=g\circ h$ with
$h(\bx)=\psi_1(\bm\eta^\top\bx)+\psi_2(\bm\xi^\top\bx)$,
where $\psi_i(t):=w_i\,\gact(t+b_{1,i})$ and $g(t):=a\,\gact(t+b_2)\rev{{}+b_3}$;
thus $g,\psi_1,\psi_2:\R\to\R$ are continuous.
Suppose that
$\|f_{\mathrm c}-\Phi\|_{L^\infty([0,1]^2)}
\le\varepsilon$.
Whenever $h(\bx)=h(\bx')$, we have
$\Phi(\bx)=g(h(\bx))=g(h(\bx'))=\Phi(\bx')$, and hence, for $\bx,\bx'\in[0,1]^2$, the \emph{collision condition},
\begin{equation}
h(\bx)=h(\bx')\ \Longrightarrow\ 
|f_{\mathrm c}(\bx)-f_{\mathrm c}(\bx')|
\le|f_{\mathrm c}(\bx)-\Phi(\bx)|+|\Phi(\bx')-f_{\mathrm c}(\bx')|
\le2\varepsilon.
\label{eq:collision}
\end{equation}
Steps~1--4 then show that~\eqref{eq:collision} forces $\varepsilon\ge\varepsilon_*$
for an absolute constant $\varepsilon_*>0$ defined in Step~4. 
\paragraph{Step 1: a direct error bound for the degenerate case $\bm\eta\parallel\bm\xi$.}
Suppose that $\bm\eta$ and $\bm\xi$ are linearly dependent. There is
then a unit vector $\bm v$ orthogonal to both of them, so that
$h(\bx_0+t\bm v)=h(\bx_0)$
whenever $\bx_0+t\bm v\in[0,1]^2$. This line segment reaches the
boundary of the square and therefore contains a point
$\bm p\in\partial[0,1]^2$ satisfying
$\|\bm p-\bx_0\|_2\ge\frac12$.
By \eqref{eq:collision} and \eqref{eq:fc-growth},
$2\varepsilon\ge|f_{\mathrm c}(\bm p)-f_{\mathrm c}(\bx_0)|
= f_{\mathrm c}(\bm p)
\ge(1-\theta)\|\bm p-\bx_0\|_2\ge\frac{1-\theta}{2}$.
Thus
$\varepsilon\ge\frac{1-\theta}{4}$.

We may henceforth assume that $\bm\eta$ and $\bm\xi$ are linearly
independent.

\paragraph{Step 2: coordinates adapted to the two directions.}
Absorbing $\|\bm\eta\|_2$ and $\|\bm\xi\|_2$ into $\psi_1$ and $\psi_2$, we
may assume $\|\bm\eta\|_2=\|\bm\xi\|_2=1$. Put
$\Delta:=|\det(\bm\eta,\bm\xi)|\in(0,1]$ and let $\bm e$ be the unit vector
orthogonal to $\bm\xi$ with $\bm\eta^\top\bm e=\Delta$, and $\bm e'$ the unit
vector orthogonal to $\bm\eta$ with $\bm\xi^\top\bm e'=\Delta$. For
$\bx\in\R^2$ define
$u(\bx):=\frac{\bm\eta^\top(\bx-\bx_0)}{\Delta}$,
$v(\bx):=\frac{\bm\xi^\top(\bx-\bx_0)}{\Delta}$.
Then $\bx=\bx_0+u(\bx)\bm e+v(\bx)\bm e'$: applying $\bm\eta^\top$ and
$\bm\xi^\top$ to both sides gives the same values, and $\bm\eta,\bm\xi$ span
$\R^2$. 
In these coordinates
$h(\bx)=\psi_1(\bm\eta^\top\bx_0+\Delta u(\bx))+\psi_2(\bm\xi^\top\bx_0+\Delta v(\bx))$;
absorbing the affine changes of variables into $\psi_1,\psi_2$ we may write
\begin{equation}
h(\bx)=\psi_1\bigl(u(\bx)\bigr)+\psi_2\bigl(v(\bx)\bigr).
\label{eq:h-adapted}
\end{equation}

\paragraph{Step 3: constructing pairs with the same $h$-value.}
Let $L_0:=\{\bx\in[0,1]^2:h(\bx)=h(\bx_0)\}$. If $L_0$ contains a point
$\bm p$ with $\|\bm p-\bx_0\|_2\ge\frac14$, then~\eqref{eq:collision}
and~\eqref{eq:fc-growth} give $2\varepsilon\ge f_{\mathrm c}(\bm p)\ge\frac{1-\theta}{4}$,
i.e.\ $\varepsilon\ge\frac{1-\theta}{8}$.

It remains to consider the case $L_0\subset B(\bx_0,\frac14)$,
where $B(\bx_0,\rho):=
\{\bx\in\R^2:\|\bx-\bx_0\|_2<\rho\}$ denotes the open Euclidean ball. Then the
continuous function $h(\bx)-h(\bx_0)$ does not vanish on
$[0,1]^2\setminus B(\bx_0,\frac14)$, which is connected, and hence has constant sign there.
Replacing $(\psi_1,\psi_2,g)$ by $(-\psi_1,-\psi_2,g(-\,\cdot))$ if necessary,
which leaves $\Phi$ unchanged, we may assume
\begin{equation}
h(\bx)>h(\bx_0)\quad\text{for all }\bx\in[0,1]^2\setminus B\bigl(\bx_0,\tfrac14\bigr).
\label{eq:hx_hx0}
\end{equation}
The points $\bx_0\pm\frac14\bm e$ have adapted coordinates $(\pm\frac14,0)$
and lie on $\partial B(\bx_0,\frac14)$, hence in
$[0,1]^2\setminus B(\bx_0,\frac14)$. By~\eqref{eq:hx_hx0}
and~\eqref{eq:h-adapted},
$\psi_1(\pm\tfrac14)+\psi_2(0)>\psi_1(0)+\psi_2(0)$, i.e.
$\psi_1\bigl(-\tfrac14\bigr)>\psi_1(0)$,
$\psi_1\bigl(\tfrac14\bigr)>\psi_1(0)$.
Replacing $(\bm\eta,\bm e,\psi_1)$ by $(-\bm\eta,-\bm e,\psi_1(-\,\cdot))$ if
necessary, which leaves $h$, $\Delta$ and~\eqref{eq:h-adapted} unchanged, we
may assume $\psi_1(-\frac14)\le\psi_1(\frac14)$. Since
$\psi_1(0)<\psi_1(-\frac14)\le\psi_1(\frac14)$, the intermediate value theorem
on $[0,\frac14]$ yields $r\in(0,\frac14]$ with
\begin{equation}
\psi_1(r)=\psi_1\bigl(-\tfrac14\bigr).
\label{eq:psi-collision}
\end{equation}
For $|t|\le\frac14$ define
$\bx_1(t):=\bx_0-\frac14\bm e+t\bm e'$ and $\bx_2(t):=\bx_0+r\bm e+t\bm e'$,
with adapted coordinates $(-\frac14,t)$ and $(r,t)$. Both lie in the closed
ball of radius $\frac14+|t|\le\frac12$ about $\bx_0$, hence in $[0,1]^2$, and
by~\eqref{eq:h-adapted} and~\eqref{eq:psi-collision},
$h(\bx_1(t))=\psi_1(-\frac14)+\psi_2(t)=\psi_1(r)+\psi_2(t)=h(\bx_2(t))$.
Consequently~\eqref{eq:collision} gives
\begin{equation}
\bigl|f_{\mathrm c}(\bx_1(t))-f_{\mathrm c}(\bx_2(t))\bigr|\le2\varepsilon,
\quad|t|\le\tfrac14.
\label{eq:profiles}
\end{equation}
At $t=0$, \eqref{eq:fc-growth} gives $f_{\mathrm c}(\bx_1(0))\ge\frac{1-\theta}{4}$
and $f_{\mathrm c}(\bx_2(0))\le(1+\theta)r$, so
$\frac{1-\theta}{4}-2\varepsilon\le(1+\theta)r$. With $\theta=\frac3{40}$ this
yields $r\ge\frac18$ whenever $\varepsilon\le\frac1{50}$; in that case
$r\in[\frac18,\frac14]$.

\paragraph{Step 4: a uniform gap between their $f_{\mathrm c}$-values.}
Let $\mathbb S^1:=\{\bx\in\R^2:\|\bx\|_2=1\}$ and, on the compact set
$\mathcal K:=\mathbb S^1\times\mathbb S^1\times[\frac18,\frac14]$, define
\begin{equation*}
\mathcal M(\bm e,\bm e',r)
:=\max_{|t|\le1/4}
\bigl|f_{\mathrm c}\bigl(\bx_0-\tfrac14\bm e+t\bm e'\bigr)
-f_{\mathrm c}\bigl(\bx_0+r\bm e+t\bm e'\bigr)\bigr| .
\end{equation*}
$\mathcal M$ is continuous
on $\mathcal K$. We claim that
\begin{equation}
\mathcal M(\bm e,\bm e',r)>0\quad\text{for every }(\bm e,\bm e',r)\in\mathcal K .
\label{eq:M-positive}
\end{equation}
Note here we include parallel pairs of unit vectors in $\mathcal K$ (not only those
produced by Step~2) so that
the parameter set is compact and yields a uniform lower bound.

\emph{Case $\bm e'=\pm\bm e$.} Choose $t=\pm\frac14$ accordingly, so that
$\bx_0-\frac14\bm e+t\bm e'=\bx_0$, while
$\bx_0+r\bm e+t\bm e'=\bx_0+(r+\frac14)\bm e$. The first value of
$f_{\mathrm c}$ is $0$, the second is at least $(1-\theta)(r+\frac14)>0$
by~\eqref{eq:fc-growth}.

\emph{Case $\bm e'\not\parallel\bm e$.} Put $\mu:=\bm e^\top\bm e'\in(-1,1)$
and assume, for contradiction, that
\begin{equation}
f_{\mathrm c}\bigl(\bx_0-\tfrac14\bm e+t\bm e'\bigr)
=f_{\mathrm c}\bigl(\bx_0+r\bm e+t\bm e'\bigr)
\quad\text{for all }t\in[-\tfrac14,\tfrac14].
\label{eq:restrictions-equal}
\end{equation}
Both
$\|\bx_1(t)-\bx_0\|_2^2=\frac1{16}-\frac{t}{2}\mu+t^2\ge\frac1{16}(1-\mu^2)$ and
$\|\bx_2(t)-\bx_0\|_2^2=r^2+2rt\mu+t^2\ge r^2(1-\mu^2)$, are bounded below by positive
constants on $\R$, so both norms are real-analytic functions of $t$ on
$\R$; the $\kappa$-terms are polynomials in $t$. By the identity theorem
for real-analytic functions, \eqref{eq:restrictions-equal} holds for every
$t\in\R$.

Write $\bm e=(e_1,e_2)$ and $\bm e'=(e_1',e_2')$. Expanding the cubes (the
$t^3$-terms cancel),
\begin{equation}
\begin{aligned}
\kappa\bigl(-\tfrac14\bm e+t\bm e'\bigr)
 -\kappa\bigl(r\bm e+t\bm e'\bigr)
=&-3\bigl(\tfrac14+r\bigr)
 \bigl((e_1')^2e_1+2(e_2')^2e_2\bigr)t^2\\
&+3\bigl(\tfrac1{16}-r^2\bigr)
 \bigl(e_1'e_1^2+2e_2'e_2^2\bigr)t
 -\bigl(\tfrac1{64}+r^3\bigr)(e_1^3+2e_2^3).
\end{aligned}
\label{eq:kappa-expansion}
\end{equation}
a polynomial in $t$ of degree at most two. By~\eqref{eq:restrictions-equal}
and the definition~\eqref{eq:fc} of $f_{\mathrm c}$,
\begin{equation}
\kappa\bigl(-\tfrac14\bm e+t\bm e'\bigr)-\kappa\bigl(r\bm e+t\bm e'\bigr)
=-10\Bigl[\bigl\|-\tfrac14\bm e+t\bm e'\bigr\|_2-\bigl\|r\bm e+t\bm e'\bigr\|_2\Bigr]
\quad\text{for all }t\in\R .
\label{eq:kappa-norm-relation}
\end{equation}
By the reverse triangle inequality, $D(t):=\|-\frac14\bm e+t\bm e'\|_2-\|r\bm e+t\bm e'\|_2\le
\big\|-\frac14\bm e-r\bm e\big\|_2=
\frac14+r$ is bounded on $\R$. A polynomial that is
bounded on $\R$ is constant, so the quadratic coefficient
in~\eqref{eq:kappa-expansion} vanishes; since $\frac14+r>0$,
\begin{equation}
(e_1')^2e_1+2(e_2')^2e_2=0 .
\label{eq:quadratic-coefficient}
\end{equation}
Moreover, \eqref{eq:kappa-norm-relation} now shows that  $D(t)$ is constant in $t$.
Rationalizing,
\begin{equation}
D(t)=\frac{\frac1{16}-r^2-2t\bigl(\frac14+r\bigr)\mu}
{\|-\frac14\bm e+t\bm e'\|_2+\|r\bm e+t\bm e'\|_2},
\label{eq:norm-difference}
\end{equation}
whose limits as $t\to+\infty$ and $t\to-\infty$ are $-(\frac14+r)\mu$ and
$(\frac14+r)\mu$, respectively. Since $D$ is constant, these limits agree,
so $\mu=0$: $\bm e\perp\bm e'$. Then
$D(t)=\sqrt{\frac1{16}+t^2}-\sqrt{r^2+t^2}$ tends to $0$ as
$|t|\to\infty$, hence $D\equiv0$, and $t=0$ gives $r=\frac14$. The
right-hand side of~\eqref{eq:kappa-norm-relation} is now identically zero,
so the polynomial~\eqref{eq:kappa-expansion} vanishes identically; its
constant term with $r=\frac14$ gives $-\frac1{32}(e_1^3+2e_2^3)=0$, i.e.
\begin{equation}
e_1^3+2e_2^3=0 .
\label{eq:constant-coefficient}
\end{equation}
Since $\bm e\perp\bm e'$ are unit vectors, $\bm e'=\pm(-e_2,e_1)$, and
\eqref{eq:quadratic-coefficient} becomes $e_2^2e_1+2e_1^2e_2=e_1e_2(e_2+2e_1)=0$.
If $e_1=0$ then $e_2=\pm1$ and $e_1^3+2e_2^3=\pm2$; if $e_2=0$ then
$e_1=\pm1$ and $e_1^3+2e_2^3=\pm1$; if $e_2=-2e_1$ then $e_1\ne0$ and
$e_1^3+2e_2^3=-15e_1^3\ne0$. All three cases
contradict~\eqref{eq:constant-coefficient}. Hence
\eqref{eq:restrictions-equal} is impossible, which proves~\eqref{eq:M-positive}.

\emph{Conclusion.} Since $\mathcal K$ is compact and $\mathcal M$ is continuous, define
$s:=\min_{\mathcal K}\mathcal M>0$. Hence every
$\Phi\in\Net((2,1);\gact)$ satisfies
$\|f_{\mathrm c}-\Phi\|_{\Linf}
\ge\varepsilon_*$,
where
$\varepsilon_*
:=
\min\big\{
\frac{1-\theta}{8},
\frac1{50},
\frac s2
\big\}>0$.

\paragraph{Step 5: rescale $f_{\mathrm c}$ into the H\"older class.}

The map $\by\mapsto\|\by\|_2$ is $\sqrt2$-Lipschitz with respect to
$\|\cdot\|_\infty$. Moreover,
$\|\nabla\kappa(\by)\|_1
=
3y_1^2+6y_2^2
\le\frac94$
for
$\by\in\big[-\frac12,\frac12\big]^2$.
Thus $f_{\mathrm c}$ is Lipschitz with respect to
$\|\cdot\|_\infty$, with constant
$\sqrt2+\frac9{40}<2$.
Also, \eqref{eq:fc-growth} gives
$\|f_{\mathrm c}\|_{L^\infty([0,1]^2)}
\le\frac{1+\theta}{\sqrt2}<1$.
Set
$\gamma':=\min\big\{\frac{\lambda}{2},R\big\}$,
$\widetilde f_{\mathrm c}:=\gamma'f_{\mathrm c}$.
Then
$\|\widetilde f_{\mathrm c}\|_{L^\infty([0,1]^2)}\le R$.
Furthermore, since $\|\bx-\bx'\|_\infty\le1$ on $[0,1]^2$,
$|\widetilde f_{\mathrm c}(\bx)-\widetilde f_{\mathrm c}(\bx')|
\le
2\gamma'\|\bx-\bx'\|_\infty\le
\lambda\|\bx-\bx'\|_\infty\le
\lambda\|\bx-\bx'\|_\infty^\alpha$.
Therefore
$\widetilde f_{\mathrm c}
\in\mathcal H^\alpha_{\lambda,R}([0,1]^2)$.
Finally, $\Net((2,1);\gact)$ is invariant under multiplication of the
output layer by the positive constant $\gamma'$. Hence
\[
\inf_{\Phi\in\Net((2,1);\gact)}
\|\widetilde f_{\mathrm c}-\Phi\|_{L^\infty([0,1]^2)}
=
\gamma'
\inf_{\Phi\in\Net((2,1);\gact)}
\|f_{\mathrm c}-\Phi\|_{L^\infty([0,1]^2)}
\ge \gamma'\varepsilon_*:=\varepsilon_2>0.
\]
\end{proof}

\subsection{$L^p$ approximation with $d+1$ hidden neurons}
\label{app:continuous_Lp}
In this appendix, we construct a continuous activation with which a network of hidden-layer widths $(d,1)$ approximates every $f\in\rev{C([0,1]^d)}$ to accuracy $\varepsilon$ in $L^p([0,1]^d)$.  
The idea of the continuous activation is to smooth the jumps of $\dexa$ over transition intervals whose lengths decay geometrically along the negative axis, and to shift the arguments of the first hidden layer so far into the negative axis that the transition intervals there have negligible total length.

\paragraph{Continuous activation.}
The activation $\dexa$ in~\eqref{eq:dexa} is constant on each interval $[k,k+1)$, $k\in\Z$, and hence can only jump at integers. For $k\in\Z$, let $w_k:=2^{\min\{k,0\}-1}\in(0,\tfrac12]$, i.e.\ $w_k=\tfrac12$ for $k\ge0$ and $w_k=2^{k-1}$ for $k<0$, and define
\begin{equation}
\dexac(t):=\begin{cases}
\dexa(t),& t\notin\bigcup_{k\in\Z}(k-w_k,\,k),\\[2pt]
\dexa(k-1)+\bigl(\dexa(k)-\dexa(k-1)\bigr)\dfrac{t-(k-w_k)}{w_k},& t\in(k-w_k,\,k),\ k\in\Z;
\end{cases}
\label{eq:dexac}
\end{equation}
see Figure~\ref{fig:varrho_c}. Here $\dexac\in C(\R)$, $\dexac(k)=\dexa(k)$ for every $k\in\Z$, and on each $[k-1,k]$ the values of $\dexac$ lie between $\dexa(k-1)$ and $\dexa(k)$. In particular, $\dexac$ is locally bounded, $\lfloor t\rfloor\le\dexac(t)\le\lfloor t\rfloor+1$ for $t<0$, and $0\le\dexac(t)\le n$ on $[S_n,S_{n+1}]$. Only the negative axis is reached by the first hidden layer below, and there the total length of the transition intervals $(k-w_k,k)$ with $k\le-N$ is $\sum_{k\ge N}2^{-k-1}=2^{-N}$ for every $N\in\N_0$. The ramp lengths on the positive axis are immaterial, because the second hidden layer will be evaluated at integers outside a small exceptional set. 
\begin{figure}[h]
\centering
\includegraphics[width=0.7\linewidth]{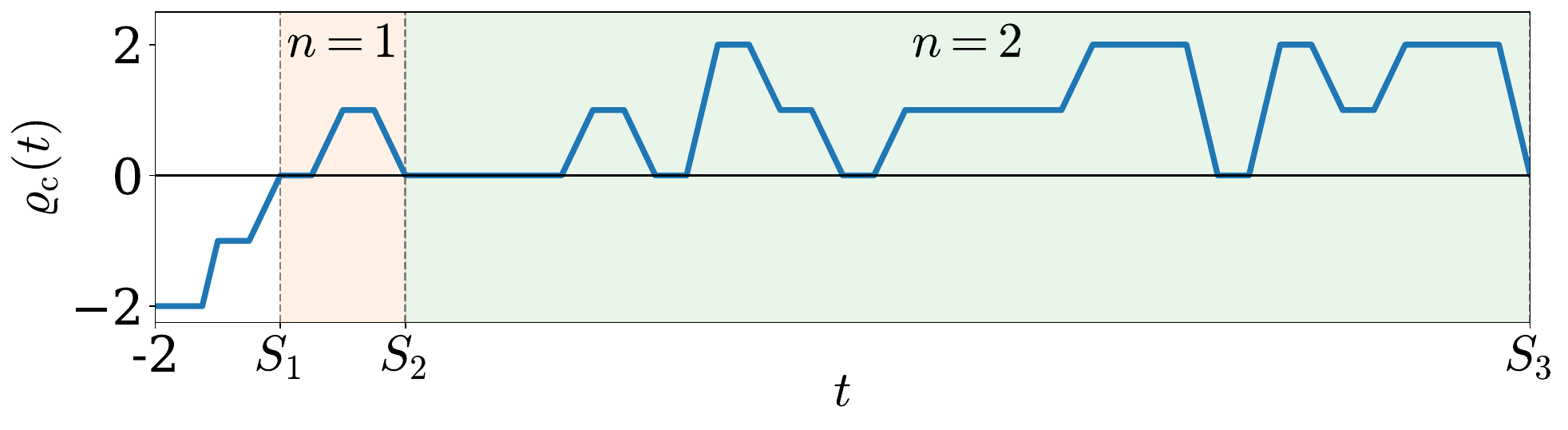}
\vspace{-4mm}
\caption{$\dexac$ in $[-2,S_3)$, to be compared with $\dexa$ in Figure~\ref{fig:varrho}. Each jump of $\dexa$ at an integer $k$ is replaced by a linear ramp on $(k-w_k,k)$, of length $\tfrac12$ for $k\ge0$ and of geometrically shrinking length $2^{k-1}$ for $k<0$.}
\label{fig:varrho_c}
\end{figure}

\paragraph{Network.}
Let $p\in\rev{(0,\infty)}$, $f\in\rev{C([0,1]^d)}$ and $\varepsilon>0$\rev{, and set $\varepsilon_p:=2^{-1/p}\varepsilon$}. Let $\delta,M,B,K,n,q_\ell,J_f,N_f$ be defined by~\eqref{eq:parameters}, \eqref{eq:quantized-values} and~\eqref{eq:Jf-Nf} with $\varepsilon$ replaced by $\rev{\varepsilon_p}$, so that $\rev{\omega_f(1/M)}+\delta\le\rev{\varepsilon_p}$, and put
\begin{equation}
N_0:=\max\Bigl\{0,\ \Bigl\lceil\rev{1+{}}\log_2 d+p\log_2\frac{2R+\delta n}{\varepsilon}\Bigr\rceil\Bigr\}\in\N_0 ,
\label{eq:N0}
\end{equation}
so that $d\,2^{-N_0}(2R+\delta n)^p\le\rev{\varepsilon^p/2}$. The network $\Phi^{\dexac}_{f,\varepsilon,p}\in\Net((d,1);\dexac)$ is defined by
\begin{align}
u_j(\bx)&:=\dexac\bigl(Mx_j-(M+1)-N_0\bigr),\quad j=1,\dots,d,&&\text{(1st hidden layer)}\label{eq:cont-layer1}\\
z(\bx)&:=\dexac\Bigl(N_f+\sum_{j=1}^{d}(M+1)^{j-1}\bigl(u_j(\bx)+M+1+N_0\bigr)\Bigr),&&\text{(2nd hidden layer)}\label{eq:cont-layer2}\\
\Phi^{\dexac}_{f,\varepsilon,p}(\bx)&:=-R+\delta\,z(\bx).&&\text{(output layer)}\label{eq:cont-output}
\end{align}
Compared with~\eqref{eq:exact-layer1}--\eqref{eq:exact-output}, the only changes are the activation and the shift $N_0$, which is compensated in the second-layer bias; now the second-layer bias is the integer $N_f+(M+1+N_0)\sum_{j=1}^d(M+1)^{j-1}=N_f+(M+1+N_0)(K-1)/M$.

\paragraph{Transition region.}
For $\bx\in[0,1]^d$, the arguments $t_j(\bx):=Mx_j-(M+1)-N_0$ in~\eqref{eq:cont-layer1} lie in $[-(M+1)-N_0,\,-1-N_0]$, so the transition intervals that can be met are $(k-w_k,k)$ with $k\in\{-M-N_0,\dots,-1-N_0\}$, each of length $w_k\le2^{-N_0-2}$ in $t_j$, i.e.\ $w_k/M$ in $x_j$. Define
$\Lambda:=\bigl\{\bx\in[0,1]^d:\ t_j(\bx)\in\bigcup_{k\in\Z}(k-w_k,k)\text{ for some }j\bigr\}$,
whose Lebesgue measure satisfies
\begin{equation}
|\Lambda|\le d\sum_{k\le-1-N_0}\frac{w_k}{M}=\frac{d\,2^{-N_0-1}}{M}\le d\,2^{-N_0}.
\label{eq:transition-measure}
\end{equation}

\begin{lemma}\label{lem:cont-exact}
For $\bx\in[0,1]^d\setminus\Lambda$, $u_j(\bx)=\lfloor Mx_j\rfloor-(M+1)-N_0$ for all $j$ and $z(\bx)=q_{r(\bx)}$; hence $\Phi^{\dexac}_{f,\varepsilon,p}(\bx)=\widehat f_{r(\bx)}$ and $|f(\bx)-\Phi^{\dexac}_{f,\varepsilon,p}(\bx)|<\rev{\varepsilon_p}$. For all $\bx\in[0,1]^d$, $0\le z(\bx)\le n$ and $|f(\bx)-\Phi^{\dexac}_{f,\varepsilon,p}(\bx)|\le2R+\delta n$.
\end{lemma}
\begin{proof}
If $\bx\notin\Lambda$, each $t_j(\bx)$ is negative and outside every transition interval, so $\dexac(t_j(\bx))=\dexa(t_j(\bx))=\lfloor t_j(\bx)\rfloor=\lfloor Mx_j\rfloor-(M+1)-N_0$. The argument of $\dexac$ in~\eqref{eq:cont-layer2} is then the integer $N_f+r(\bx)$, at which $\dexac$ agrees with $\dexa$, and $\dexa(N_f+r(\bx))=q_{r(\bx)}$ by~\eqref{eq:digit-extraction}; the error bound is~\eqref{eq:basic-error} with $\rev{\varepsilon_p}$.

For arbitrary $\bx$, since $t_j(\bx)\in[-(M+1)-N_0,-1-N_0]$ and $\dexac$ lies between the neighboring values of $\lfloor\cdot\rfloor$ on this range, $u_j(\bx)+M+1+N_0\in[0,M]$. Hence the preactivation part in~\eqref{eq:cont-layer2} lies in $[N_f,\,N_f+M\sum_{j=1}^d(M+1)^{j-1}]=[N_f,\,N_f+K-1]\subset[S_n,S_{n+1}-1]$, because $K\le n$ and $N_f+n-1\le S_n+n(n+1)^n-1=S_{n+1}-1$. On $[S_n,S_{n+1}]$ we have $0\le\dexac\le n$, so $0\le z(\bx)\le n$, $\Phi^{\dexac}_{f,\varepsilon,p}(\bx)\in[-R,-R+\delta n]$, and since $f(\bx)\in[-R,R]$ the last claim follows.
\end{proof}

\begin{theorem}[Fixed continuous activation, $(d,1)$ architecture, $L^p$ guarantee]
\label{thm:contLp}
Let $d\in\N_+$ and $p\in\rev{(0,\infty)}$. For every $f\in\rev{C([0,1]^d)}$ and every $\varepsilon>0$, the network $\Phi^{\dexac}_{f,\varepsilon,p}\in\Net((d,1);\dexac)$ defined by~\eqref{eq:cont-layer1}--\eqref{eq:cont-output} satisfies
\begin{equation}
\bigl\|f-\Phi^{\dexac}_{f,\varepsilon,p}\bigr\|_{L^p([0,1]^d)}\le\varepsilon .
\end{equation}
The activation $\dexac$ is continuous and does not depend on $f$, $\varepsilon$, $p$ or $d$. \rev{Since $C([0,1]^d)$ is dense in $L^p([0,1]^d)$, $\Net((d,1);\dexac)$ is dense in $L^p([0,1]^d)$.}
\end{theorem}
\begin{proof}
Splitting $[0,1]^d$ into $\Lambda$ and its complement and using~\eqref{eq:transition-measure} and Lemma~\ref{lem:cont-exact},
\begin{align*}
\bigl\|f-\Phi^{\dexac}_{f,\varepsilon,p}\bigr\|_{L^p([0,1]^d)}^p
&=\int_{[0,1]^d\setminus\Lambda}\bigl|f-\Phi^{\dexac}_{f,\varepsilon,p}\bigr|^p+\int_{\Lambda}\bigl|f-\Phi^{\dexac}_{f,\varepsilon,p}\bigr|^p\\
&\le\rev{\varepsilon_p^p}+d\,2^{-N_0}(2R+\delta n)^p\le\rev{\frac{\varepsilon^p}{2}+\frac{\varepsilon^p}{2}=}\varepsilon^p .
\end{align*}
\end{proof}
\begin{remark}[Parameters and bit complexity of $\Phi^{\dexac}_{f,\varepsilon,p}$]\rev{\label{rem:cont-bits}}
For $R\in\N_+$, all parameters of $\Phi^{\dexac}_{f,\varepsilon,p}$ are
integers except the dyadic output weight $\delta$, and $\dexac$ takes
rational values at rational arguments, so the network can be evaluated in
exact arithmetic as in Section~\ref{sec:numerics}. Compared with
$\Phi^{\dexa}_{f,\rev{\varepsilon_p}}$, only two parameters change: the
first-layer biases decrease by $N_0$ and the second-layer bias increases by
the integer $N_0(K-1)/M$. Since $\delta n\le\delta K+\delta(B-1)$ with
$\delta K=\Theta(\varepsilon^{1-d/\alpha})$ and $\delta(B-1)\le4R+2\delta$
by~\eqref{eq:parameters}, \eqref{eq:N0} gives
$N_0=O\bigl(p\,\tfrac d\alpha\log\tfrac1\varepsilon\bigr)$, so these changes
cost $O(\log\frac1\varepsilon)$ bits for fixed $d,\alpha,\lambda,R,p$, while
the dominant term $\bit(N_f)=\Theta(\varepsilon^{-d/\alpha}\log\frac1\varepsilon)$
from the proof of Theorem~\ref{thm:bits} is unchanged. Hence, for $d\ge2$,
$\sup_{f\in\Hol}\bit\bigl(\Phi^{\dexac}_{f,\varepsilon,p}\bigr)
=\Theta\bigl(\varepsilon^{-d/\alpha}\log\tfrac1\varepsilon\bigr)$:
the continuous $L^p$ construction has the same asymptotic bit complexity as
the discontinuous $L^\infty$ construction.
\end{remark}

\rev{\begin{remark}[Tree difference encoding with a continuous activation]\label{rem:cont-tree}
The construction above can also be applied to $\widetilde\dexa$ in~\eqref{eq:dexa-diff}, which is constant on each $[k,k+1)$, $k\in\Z$, and equals $\dexa$ on $(-\infty,0]$. Let $\widetilde\dexa_c$ be defined as $\dexac$ with $\dexa$ replaced by $\widetilde\dexa$, so that $\widetilde\dexa_c\in C(\R)$, and let $\Phi^{\widetilde\dexa_c}_{f,\varepsilon,p}$ be the network~\eqref{eq:cont-layer1}--\eqref{eq:cont-output} with $\dexac$, $N_f$ and the output bias $-R$ replaced by $\widetilde\dexa_c$, $\widetilde N_f$ and $-R+\delta q_0$. Analogously to Theorem~\ref{thm:contLp}, $\|f-\Phi^{\widetilde\dexa_c}_{f,\varepsilon,p}\|_{L^p([0,1]^d)}\le\varepsilon$, and, as in Remark~\ref{rem:cont-bits}, this construction preserves the bit complexity $\Theta(\varepsilon^{-d/\alpha})$ of $\Phi^{\widetilde\dexa}_{f,\varepsilon}$ in Theorem~\ref{thm:bits} for $d\ge2$.
\end{remark}}

\rev{\section{Vector-valued outputs}\label{app:vector}}

\rev{The setting of Section~\ref{sec:setting} also applies to continuous vector-valued targets $\bm f=(f^{(1)},\ldots,f^{(d_{\mathrm{out}})})\in C([0,1]^d;\R^{d_{\mathrm{out}}})$. We measure the error of a network $\Phi:[0,1]^d\to\R^{d_{\mathrm{out}}}$ by $\|\bm f-\Phi\|_{L^\infty([0,1]^d;\R^{d_{\mathrm{out}}})}:=\operatorname*{ess\,sup}_{\bx\in[0,1]^d}|\bm f(\bx)-\Phi(\bx)|$, where $|\cdot|$ is the Euclidean norm.}

\rev{Let $\Net_{d_{\mathrm{out}}}(\bm n;\gact)$ and $\Net^{\mathrm{skip}}_{d_{\mathrm{out}}}(\bm n;\gact)$ denote the classes $\Net(\bm n;\gact)$ and $\Net^{\mathrm{skip}}(\bm n;\gact)$ of Section~\ref{sec:setting} with values in $\R^{d_{\mathrm{out}}}$. Their hidden layers are those of~\eqref{eq:network}, with $W_k\in\R^{n_k\times n_{k-1}}$ ($n_0=d$) and $\bb_k\in\R^{n_k}$ for $k=1,\ldots,L$, and their output weights and bias are $\bm a\in\R^{n_L\times d_{\mathrm{out}}}$, $\bm a_k\in\R^{n_k\times d_{\mathrm{out}}}$ and $\bb_{L+1}\in\R^{d_{\mathrm{out}}}$ in place of $\bm a\in\R^{n_L}$, $\bm a_k\in\R^{n_k}$ and $b_{L+1}\in\R$. Thus the $d_{\mathrm{out}}$ outputs share the hidden layers, and the sizes of $W_k$ and $\bb_k$ are fixed by the widths. For example, in $\Net_{d_{\mathrm{out}}}((d,d_{\mathrm{out}});\gact)$ the second hidden layer has $d_{\mathrm{out}}$ neurons, so $W_2\in\R^{d_{\mathrm{out}}\times d}$ and $\bb_2\in\R^{d_{\mathrm{out}}}$. The class $\Net_{d_{\mathrm{out}}}(\bm n;\gact)$ is \emph{universal} if it is dense in $C([0,1]^d;\R^{d_{\mathrm{out}}})$ in the norm $\|\cdot\|_{L^\infty([0,1]^d;\R^{d_{\mathrm{out}}})}$, and the same definition applies to $\Net^{\mathrm{skip}}_{d_{\mathrm{out}}}(\bm n;\gact)$.}

\rev{We define the \emph{minimal neuron count} for $d_{\mathrm{out}}$ outputs as $\nu_{d,d_{\mathrm{out}}}:=\min\{|\bm n|:\exists L\in\N_+,\ \bm n\in\N_+^L,\ \gact\in L^1_{\mathrm{loc}}(\R)\ \text{such that}\ \Net_{d_{\mathrm{out}}}(\bm n;\gact)\ \text{is universal}\}$, so that $\nu_{d,1}=\nu_d$ in Definition~\ref{def:fixed-size}.}

\rev{Throughout, $R$, $\delta$, $M$ and $B$ are chosen as in Section~\ref{sec:setting} for all components simultaneously, with $\varepsilon$ replaced by $\varepsilon/\sqrt{d_{\mathrm{out}}}$, and $q^{(i)}_\ell$, $\widehat f^{(i)}_\ell$, $N_{f^{(i)}}$, $\widetilde N_{f^{(i)}}$ and $A_{f^{(i)}}$ denote the quantities of Sections~\ref{sec:setting}, \ref{subsec:upper} and~\ref{sec:elementary} for $f^{(i)}$. By~\eqref{eq:basic-error}, a network whose $i$th output equals $\widehat f^{(i)}_{r(\bx)}$ for every $i$ has component errors below $\varepsilon/\sqrt{d_{\mathrm{out}}}$ and hence Euclidean error below $\varepsilon$. Hence the constructions with the lookup activations $\dexa$ and $\widetilde\dexa$ in Theorem~\ref{thm:exact} and with the elementary activation $\eact$ in Theorem~\ref{thm:elementary} extend easily to vector-valued targets, as stated in Propositions~\ref{prop:vector-exact} and~\ref{prop:vector-elementary}.}

\rev{\begin{proposition}[Lookup activations: $d+d_{\mathrm{out}}$ neurons]\label{prop:vector-exact}
Let $\bm f\in C([0,1]^d;\R^{d_{\mathrm{out}}})$ and $\varepsilon>0$. Let $\Phi^{\dexa}_{\bm f,\varepsilon}\in\Net_{d_{\mathrm{out}}}((d,d_{\mathrm{out}});\dexa)$ be the network whose $i$th component is the network $\Phi^{\dexa}_{f^{(i)},\varepsilon/\sqrt{d_{\mathrm{out}}}}$ of Theorem~\ref{thm:exact}, with the first hidden layer~\eqref{eq:exact-layer1} shared and, as in~\eqref{eq:exact-layer2} and~\eqref{eq:exact-output},
\begin{align*}
z_i(\bx)&:=\dexa\bigl(N_{f^{(i)}}+r(\bx)\bigr)=q^{(i)}_{r(\bx)},\quad i=1,\ldots,d_{\mathrm{out}}, &&\textup{(2nd hidden layer)}\\
\Phi^{\dexa}_{\bm f,\varepsilon}(\bx)&:=\bigl(-R+\delta z_i(\bx)\bigr)_{i=1}^{d_{\mathrm{out}}}=\bigl(\widehat f^{(i)}_{r(\bx)}\bigr)_{i=1}^{d_{\mathrm{out}}}. &&\textup{(output layer)}
\end{align*}
Let $\Phi^{\widetilde\dexa}_{\bm f,\varepsilon}\in\Net_{d_{\mathrm{out}}}((d,d_{\mathrm{out}});\widetilde\dexa)$ be the network whose $i$th component is the network $\Phi^{\widetilde\dexa}_{f^{(i)},\varepsilon/\sqrt{d_{\mathrm{out}}}}$ of Theorem~\ref{thm:exact}, with the first hidden layer shared and, as in~\eqref{eq:tree-layer2} and~\eqref{eq:tree-output},
\begin{align*}
\widetilde z_i(\bx)&:=\widetilde\dexa\bigl(\widetilde N_{f^{(i)}}+r(\bx)\bigr)=q^{(i)}_{r(\bx)}-q^{(i)}_0,\quad i=1,\ldots,d_{\mathrm{out}}, &&\textup{(2nd hidden layer)}\\
\Phi^{\widetilde\dexa}_{\bm f,\varepsilon}(\bx)&:=\bigl(-R+\delta q^{(i)}_0+\delta\widetilde z_i(\bx)\bigr)_{i=1}^{d_{\mathrm{out}}}=\bigl(\widehat f^{(i)}_{r(\bx)}\bigr)_{i=1}^{d_{\mathrm{out}}}. &&\textup{(output layer)}
\end{align*}
Both networks satisfy $\|\bm f-\Phi\|_{L^\infty([0,1]^d;\R^{d_{\mathrm{out}}})}\le\varepsilon$.
\end{proposition}}

\rev{\begin{proposition}[Elementary activation: $d+d_{\mathrm{out}}+2$ neurons, or $d+d_{\mathrm{out}}+1$ with skip connections]\label{prop:vector-elementary}
Let $\bm f\in C([0,1]^d;\R^{d_{\mathrm{out}}})$ and $\varepsilon>0$, and let $A_{\bm f}:=\sum_{\ell=0}^{K-1}\sum_{i=1}^{d_{\mathrm{out}}}q^{(i)}_\ell B^{d_{\mathrm{out}}\ell+i-1}$, which stores $q^{(1)}_\ell,\ldots,q^{(d_{\mathrm{out}})}_\ell$ in consecutive base-$B$ digits. Let $\Phi^{\eact}_{\bm f,\varepsilon}\in\Net_{d_{\mathrm{out}}}((d,1,d_{\mathrm{out}}+1);\eact)$ be the network with the first hidden layer~\eqref{eq:el-layer1} and
\begin{align*}
v(\bx)&:=\eact\bigl(-1-d_{\mathrm{out}}\log_2(B)\,\bc^{\top}\bu(\bx)\bigr)=\tfrac12B^{-d_{\mathrm{out}}r(\bx)}-1, &&\textup{(2nd hidden layer)}\\
w_i(\bx)&:=\eact\bigl(2A_{\bm f}B^{-i}(v(\bx)+1)\bigr)=\bigl\lfloor A_{\bm f}B^{-d_{\mathrm{out}}r(\bx)-i}\bigr\rfloor,\quad i=0,\ldots,d_{\mathrm{out}}, &&\textup{(3rd hidden layer)}\\
\Phi^{\eact}_{\bm f,\varepsilon}(\bx)&:=\bigl(-R+\delta\bigl(w_{i-1}(\bx)-Bw_i(\bx)\bigr)\bigr)_{i=1}^{d_{\mathrm{out}}}=\bigl(\widehat f^{(i)}_{r(\bx)}\bigr)_{i=1}^{d_{\mathrm{out}}}. &&\textup{(output layer)}
\end{align*}
Let $\Phi^{\eact,\mathrm{skip}}_{\bm f,\varepsilon}\in\Net^{\mathrm{skip}}_{d_{\mathrm{out}}}((d,1,d_{\mathrm{out}});\eact)$ be the network whose $i$th component is the network $\Phi^{\eact,\mathrm{skip}}_{f^{(i)},\varepsilon/\sqrt{d_{\mathrm{out}}}}$ of Theorem~\ref{thm:elementary}, with the first two hidden layers~\eqref{eq:el-layer1} and~\eqref{eq:el-layer2} shared. That is, its $i$th third-layer neuron and its $i$th output are $w^{\mathrm{skip}}$ and $\Phi^{\eact,\mathrm{skip}}_{f,\varepsilon}$ in~\eqref{eq:skip-network} with $A_f$ replaced by $A_{f^{(i)}}$.
Then both networks satisfy $\|\bm f-\Phi\|_{L^\infty([0,1]^d;\R^{d_{\mathrm{out}}})}\le\varepsilon$.
\end{proposition}}

\rev{\begin{proposition}[Last-hidden-layer obstruction]\label{prop:vector-lower}
Let $d,d_{\mathrm{out}}\in\N_+$ and let $\gact:\R\to\R$ be arbitrary. For every architecture $\bm n=(n_1,\ldots,n_L)$ with $n_L<d_{\mathrm{out}}$, there exist a function $\bm f\in C([0,1]^d;\R^{d_{\mathrm{out}}})$ and a constant $\varepsilon_3>0$ such that
\begin{equation*}
\inf_{\Phi\in\Net_{d_{\mathrm{out}}}(\bm n;\gact)}\|\bm f-\Phi\|_{L^\infty([0,1]^d;\R^{d_{\mathrm{out}}})}\ge\varepsilon_3.
\end{equation*}
\end{proposition}}

\rev{\begin{proof}
Fix $\alpha\in(0,1]$ and $\lambda,R>0$. Let $\zeta:=\min\{R,\lambda d_{\mathrm{out}}^{-\alpha}\}$, and take $\bm f:=\bm f_\wedge=(f_\wedge^{(1)},\ldots,f_\wedge^{(d_{\mathrm{out}})})$ with the tent-shaped components $f_\wedge^{(i)}(\bx):=\zeta\max\{0,1-|d_{\mathrm{out}}x_1-i|\}$.

\paragraph{Step 1: Verification that $\bm f_\wedge\in\Hol^{d_{\mathrm{out}}}$.} Since $0\le f_\wedge^{(i)}\le\zeta\le R$, we have $\|f_\wedge^{(i)}\|_{\Linf}\le R$. Since $t\mapsto\max\{0,1-|t|\}$ is $1$-Lipschitz with values in $[0,1]$, and $\min\{1,s\}\le s^\alpha$ for $s\ge0$,
\begin{equation*}
|f_\wedge^{(i)}(\bx)-f_\wedge^{(i)}(\by)|\le\zeta\min\{1,d_{\mathrm{out}}\|\bx-\by\|_\infty\}\le\zeta d_{\mathrm{out}}^{\alpha}\|\bx-\by\|_\infty^\alpha\le\lambda\|\bx-\by\|_\infty^\alpha .
\end{equation*}

\paragraph{Step 2: The values of $\Phi$ lie in a hyperplane.} Let $\Phi\in\Net_{d_{\mathrm{out}}}(\bm n;\gact)$. The output weight matrix $\bm a\in\R^{n_L\times d_{\mathrm{out}}}$ of $\Phi$ in~\eqref{eq:network} has fewer rows than columns because $n_L<d_{\mathrm{out}}$, so its kernel contains some $\bm\vartheta\in\R^{d_{\mathrm{out}}}$ with $|\bm\vartheta|=1$, that is, $\bm a\bm\vartheta=\bm 0$. Hence, by~\eqref{eq:network}, $\bm\vartheta^{\top}\Phi(\bx)=(\bm a\bm\vartheta)^{\top}h_L(\bx)+\bm\vartheta^{\top}\bb_{L+1}=\bm\vartheta^{\top}\bb_{L+1}=:\beta$ for every $\bx\in[0,1]^d$, so all values of $\Phi$ lie in the hyperplane $\{\by\in\R^{d_{\mathrm{out}}}:\bm\vartheta^{\top}\by=\beta\}$.

\paragraph{Step 3: The error along $\bm\vartheta$ is at most the error.} Let $\tau:=\|\bm f_\wedge-\Phi\|_{L^\infty([0,1]^d;\R^{d_{\mathrm{out}}})}$. Since $\bm\vartheta^{\top}\Phi(\bx)=\beta$ by Step~2 and $|\bm\vartheta|=1$, the Cauchy--Schwarz inequality gives $|\bm\vartheta^{\top}\bm f_\wedge(\bx)-\beta|=|\bm\vartheta^{\top}(\bm f_\wedge(\bx)-\Phi(\bx))|\le|\bm f_\wedge(\bx)-\Phi(\bx)|$ for every $\bx\in[0,1]^d$. Since $\tau$ is the essential supremum of the right-hand side, the left-hand side is at most $\tau$ for almost every $\bx$. The left-hand side is continuous in $\bx$, so if it exceeded $\tau$ at some point, it would exceed $\tau$ on a relatively open neighborhood of that point in $[0,1]^d$, which has positive measure. Hence $|\bm\vartheta^{\top}\bm f_\wedge(\bx)-\beta|\le\tau$ for every $\bx\in[0,1]^d$, including the points on the boundary of $[0,1]^d$.

\paragraph{Step 4: Error lower bound.} Let $\bm p_i:=(i/d_{\mathrm{out}},0,\ldots,0)\in[0,1]^d$ for $i=0,\ldots,d_{\mathrm{out}}$. Since $f_\wedge^{(j)}(\bm p_i)=\zeta\max\{0,1-|i-j|\}$, we have $\bm f_\wedge(\bm p_0)=\bm 0$ and $\bm f_\wedge(\bm p_i)=\zeta\bm e_i$ for $i\ge1$, where $\bm e_i$ is the $i$th unit vector of $\R^{d_{\mathrm{out}}}$. Write $\bm\vartheta=(\vartheta_1,\ldots,\vartheta_{d_{\mathrm{out}}})$, so that $\bm\vartheta^{\top}\bm e_i=\vartheta_i$. Taking $\bx=\bm p_0$ in Step~3 gives $|\bm\vartheta^{\top}\bm 0-\beta|=|\beta|\le\tau$. Taking $\bx=\bm p_i$ with $i\ge1$ gives $|\bm\vartheta^{\top}(\zeta\bm e_i)-\beta|=|\zeta\vartheta_i-\beta|\le\tau$. Writing $\zeta\vartheta_i=(\zeta\vartheta_i-\beta)+\beta$, the triangle inequality and the two bounds give $\zeta|\vartheta_i|\le|\zeta\vartheta_i-\beta|+|\beta|\le\tau+\tau=2\tau$, that is, $|\vartheta_i|\le2\tau/\zeta$ for $i=1,\ldots,d_{\mathrm{out}}$, since $\zeta>0$. Since $|\bm\vartheta|=1$, summing the squares of these $d_{\mathrm{out}}$ bounds gives $1=\sum_{i=1}^{d_{\mathrm{out}}}\vartheta_i^2\le d_{\mathrm{out}}(2\tau/\zeta)^2$, that is, $\tau\ge\varepsilon_3:=\zeta/(2\sqrt{d_{\mathrm{out}}})$. Since $\Phi$ was arbitrary, this proves the claim.
\end{proof}}

\rev{\begin{theorem}[Minimal neuron count for vector-valued outputs]\label{thm:vector-nu}
For $d\ge2$ and $d_{\mathrm{out}}\in\N_+$, $\nu_{d,d_{\mathrm{out}}}=d+d_{\mathrm{out}}$.
\end{theorem}}

\rev{\begin{proof}
If $\Net_{d_{\mathrm{out}}}(\bm n;\gact)$ is universal with $\gact\in L^1_{\mathrm{loc}}(\R)$, then Propositions~\ref{prop:first-layer} and~\ref{prop:shallow}, applied to the first component, give $n_1\ge d$ and $L\ge2$, and Proposition~\ref{prop:vector-lower} gives $n_L\ge d_{\mathrm{out}}$. Hence $|\bm n|\ge n_1+n_L\ge d+d_{\mathrm{out}}$, and Proposition~\ref{prop:vector-exact} attains this count.
\end{proof}}

\rev{\begin{remark}[Bit complexity]\label{rem:vector-bits}
For $\bm f\in\Hol^{d_{\mathrm{out}}}$ and fixed $d_{\mathrm{out}}$, replacing $\varepsilon$ by $\varepsilon/\sqrt{d_{\mathrm{out}}}$ in~\eqref{eq:parameters} and~\eqref{eq:parameter-rates} changes $\delta$, $M$, $B$ and $K$ only by constant factors. Hence, as in Section~\ref{sec:bits}, the construction $\Phi^{\widetilde\dexa}_{\bm f,\varepsilon}$ is optimal in bit complexity, and $\Phi^{\dexa}_{\bm f,\varepsilon}$, $\Phi^{\eact}_{\bm f,\varepsilon}$ and $\Phi^{\eact,\mathrm{skip}}_{\bm f,\varepsilon}$ are optimal up to a logarithmic factor.
\end{remark}}

\rev{\paragraph{Numerical verification.}
We verify Propositions~\ref{prop:vector-exact} and~\ref{prop:vector-elementary} for $d=d_{\mathrm{out}}=2$ with the gradient field $\bm f_7$ of $-\cos(\pi x_1)\cos(\pi x_2)/(2\pi)$,
\begin{equation}\label{eq:vector-test}
\bm f_7(\bx)=\bigl(\tfrac12\sin(\pi x_1)\cos(\pi x_2),\ \tfrac12\cos(\pi x_1)\sin(\pi x_2)\bigr),
\end{equation}
whose components are $f_1$ of Section~\ref{sec:numerics} and $f_1$ with $x_1$ and $x_2$ interchanged, so that $(\alpha,\lambda,R)=(1,\pi,1)$. Table~\ref{tab:numerics-vector} reports the parameters $\delta$, $M$, $K$ and $B$, the bit lengths of the stored integers and the maximal Euclidean errors over the $3{,}004$ test points constructed as in Section~\ref{sec:numerics} for $\varepsilon=2^{-1},\ldots,2^{-6}$. In addition, with $\bm f=\bm f_7$, every output of $\Phi^{\widetilde\dexa}_{\bm f,\varepsilon}$ equals $-R+\delta q^{(i)}_{r(\bx)}$, $\Phi^{\eact}_{\bm f,\varepsilon}=\Phi^{\dexa}_{\bm f,\varepsilon}$, and $0\le\Phi^{\eact,\mathrm{skip}}_{\bm f,\varepsilon,i}-\Phi^{\eact}_{\bm f,\varepsilon,i}<\delta$ $(i=1,2)$ hold exactly at every test point.}

\rev{\begin{table}[t]
\centering
\caption{\rev{Parameters, encoding bit lengths, and empirical errors for the networks of Propositions~\ref{prop:vector-exact} and~\ref{prop:vector-elementary} applied to $\bm f=\bm f_7$ in~\eqref{eq:vector-test}, with $\delta,M,B$ given by~\eqref{eq:parameters} and~\eqref{eq:parameter-rates} at accuracy $\varepsilon/\sqrt2$, $\bit(N_{\bm f}):=\bit(N_{f^{(1)}})+\bit(N_{f^{(2)}})$ and $\bit(\widetilde N_{\bm f}):=\bit(\widetilde N_{f^{(1)}})+\bit(\widetilde N_{f^{(2)}})$. The empirical errors are $E_{\varrho}:=\max_{\bx\in\mathcal X_{\mathrm{test}}}|\bm f_7(\bx)-\Phi^{\dexa}_{\bm f_7,\varepsilon}(\bx)|$ and $E_{\widetilde\varrho}$, $E_{\sigma}$ and $E_{\sigma,\mathrm{skip}}$, defined in the same way with $\Phi^{\dexa}_{\bm f_7,\varepsilon}$ replaced by $\Phi^{\widetilde\dexa}_{\bm f_7,\varepsilon}$, $\Phi^{\eact}_{\bm f_7,\varepsilon}$ and $\Phi^{\eact,\mathrm{skip}}_{\bm f_7,\varepsilon}$, where $|\cdot|$ is the Euclidean norm on $\R^2$ and $\mathcal X_{\mathrm{test}}$ is the set of $3{,}004$ test points constructed as in Section~\ref{sec:numerics}.}}
\vspace{-2mm}
\label{tab:numerics-vector}
\setlength{\tabcolsep}{2.5pt}
\resizebox{\linewidth}{!}{%
\begin{tabular}{@{}c*{12}{r}@{}}
\toprule
$\varepsilon$ & $\delta$ & $M$ & $K$ & $B$
& $\bit(A_{\bm f})$ & $\bit(N_{\bm f})$ & $\bit(\widetilde N_{\bm f})$
& $\displaystyle\tfrac{\bit(A_{\bm f})}{d_{\mathrm{out}}K\log_2B}$
& $\displaystyle\tfrac{\bit(N_{\bm f})}{d_{\mathrm{out}}n\log_2(n+1)}$
& $\displaystyle\tfrac{\bit(\widetilde N_{\bm f})}{d_{\mathrm{out}}K\log_25}$
& $E_{\varrho}/E_{\widetilde\varrho}/E_{\sigma}$ & $E_{\sigma,\mathrm{skip}}$ \\
\midrule
$2^{-1}$ & $2^{-3}$ & 18 & 361 & 32
& 3\,608 & 6\,144 & \rev{1\,692} & 0.9994 & 1.0012 & \rev{1.0093}
& 0.2326 & 0.1985 \\
$2^{-2}$ & $2^{-4}$ & 36 & 1\,369 & 64
& 16\,426 & 28\,538 & \rev{6\,376} & 0.9999 & 1.0003 & \rev{1.0029}
& 0.1244 & 0.1062 \\
$2^{-3}$ & $2^{-5}$ & 72 & 5\,329 & 128
& 74\,604 & 131\,956 & \rev{24\,770} & 1.0000 & 1.0001 & \rev{1.0009}
& 0.0640 & 0.0545 \\
$2^{-4}$ & $2^{-6}$ & 143 & 20\,736 & 256
& 331\,774 & 594\,718 & \rev{96\,322} & 1.0000 & 1.0000 & \rev{1.0003}
& 0.0305 & 0.0261 \\
$2^{-5}$ & $2^{-7}$ & 285 & 81\,796 & 512
& 1\,472\,326 & 2\,669\,798 & \rev{379\,880} & 1.0000 & 1.0000 & \rev{1.0001}
& 0.0156 & 0.0133 \\
$2^{-6}$ & $2^{-8}$ & 569 & 324\,900 & 1\,024
& 6\,497\,998 & 11\,897\,622 & \rev{1\,508\,824} & 1.0000 & 1.0000 & \rev{1.0000}
& 0.0081 & 0.0066 \\
\bottomrule
\end{tabular}}
\end{table}}

\rev{\section{Computational cost}\label{app:cost}}

\rev{We count the bit operations needed to build each network from $q_0,\ldots,q_{K-1}$ and to evaluate it at one input, and we report running times measured with one CPU thread for each test function (Appendix~\ref{app:cost-measured}).}

\rev{\subsection{Cost of arithmetic operations}\label{app:cost-ops}
Adding, subtracting or comparing integers costs a number of bit operations proportional to their length, and shifting an integer, that is, multiplying it by a power of two or taking the integer part of its quotient by a power of two, also costs a number of bit operations proportional to the length of the result. Multiplying two integers costs, in theory, their length times its logarithm \citep{harvey2021integer}, and so do a division with remainder, a power computed by repeated squaring, and a product or division with a short integer. Converting an integer between binary and base $5$ or base $n+1$ by divide and conquer costs one further logarithmic factor \citep{brent2010modern}. These costs are in terms of the lengths of the integers.

The discussed networks also contain dyadic numbers, that is, integers divided by powers of two, for example $\delta$, the output bias $-R+\delta q_0$ of $\Phi^{\widetilde\dexa}_{f,\varepsilon}$, and $2A_f/B$ and $v(\bx)=\frac12B^{-r(\bx)}-1$ in $\Phi^{\eact}_{f,\varepsilon}$. Adding two dyadic numbers amounts to one shift, which brings them to a common denominator, and one addition of the numerators, and multiplying them amounts to multiplying the numerators and the denominators. Hence the costs above also apply to dyadic numbers, where the relevant lengths are those of the numerators and the denominators.

As detailed in Appendix~\ref{app:bits}, the longest integers that occur in building or evaluating $\Phi^{\dexa}_{f,\varepsilon}$ and $\Phi^{\eact}_{f,\varepsilon}$, such as $N_f$, $S_n$, $J_f$, $(n+1)^{n-r(\bx)}$, $A_f$ and the denominator $2B^{r(\bx)}$ of $v(\bx)$, have $O(\varepsilon^{-d/\alpha}\log(1/\varepsilon))$ bits, and the longest ones for $\Phi^{\widetilde\dexa}_{f,\varepsilon}$, such as $\widetilde N_f$, $\widetilde S_{M,d}$ and $\widetilde J_f$, have $O(\varepsilon^{-d/\alpha})$ bits. Operations on shorter integers are cheaper, so the cost of building and evaluating the networks is determined by the operations on these longest integers.}

\rev{\subsection{Operations of each network}\label{app:cost-bounds}
Table~\ref{tab:cost-bounds} gives the bit operations of the construction and of one evaluation of each network, obtained from the costs and lengths in Appendix~\ref{app:cost-ops} with $K,n=\Theta(\varepsilon^{-d/\alpha})$, together with the steps that determine them, which we explain below.

\begin{table}[ht]
\centering
\caption{\rev{Bit operations of the construction and of one evaluation, from the cell address $r(\bx)$ on, and the steps that determine them. Constants depend on $d,\alpha,\lambda,R$.}}
\label{tab:cost-bounds}
\small
\setlength{\tabcolsep}{4pt}
\rev{%
\begin{tabular}{@{}lllll@{}}
\toprule
network & construction & determined by & evaluation & determined by\\
\midrule
$\Phi^{\dexa}_{f,\varepsilon}$ & $O(\varepsilon^{-2d/\alpha}\log^2(1/\varepsilon))$ & $S_n$ & $O(\varepsilon^{-d/\alpha}\log^2(1/\varepsilon))$ & division by $n$, digit $b_r(j)$\\
\addlinespace
$\Phi^{\widetilde\dexa}_{f,\varepsilon}$ & $O(\varepsilon^{-d/\alpha}\log^2(1/\varepsilon))$ & $\Delta_\ell$, $\widetilde J_f$, $\widetilde S_{M,d}$ & $O(\varepsilon^{-d/\alpha}\log^2(1/\varepsilon))$ & digits $\widetilde b_k(j)$\\
\addlinespace
$\Phi^{\eact}_{f,\varepsilon}$, $\Phi^{\eact,\mathrm{skip}}_{f,\varepsilon}$ & $O(\varepsilon^{-d/\alpha}\log(1/\varepsilon))$ & $A_f$ & $O(\varepsilon^{-d/\alpha}\log(1/\varepsilon))$ & shifts of $A_f$\\
\bottomrule
\end{tabular}}
\end{table}

\paragraph{The offset $S_n$.} We know no closed form of $S_n$ in~\eqref{eq:unit-intervals}, because its terms have the different bases $k+1$, unlike a geometric sum. Hence we add its $n-1$ terms. Together they have $\Theta(n^2\log n)$ bits, so this summation takes at least $\Theta(n^2\log n)$ bit operations. Since $S_n$ depends only on $n$ and not on $f$, it can be computed once and stored for all targets, and then the construction for each $f$ costs only $O(\varepsilon^{-d/\alpha}\log^3(1/\varepsilon))$ bit operations, for converting $J_f$ to base $n+1$.

\paragraph{The offset $\widetilde S_{M,d}$.} The offset $\widetilde S_{M,d}$ in~\eqref{eq:dexa-diff} is the sum of the block lengths $K'5^{K'-1}$, $K'=(M'+1)^{d'}$, over all grids $(M',d')$ whose blocks precede that of $(M,d)$, and each such grid has $K'\le K$. For the grids with $d'=1$, these lengths are $(M'+1)5^{M'}$, a linear function of $M'$ times a power of $5$, and their sum has a closed form, as a geometric sum does, which needs one power of $5$ and hence $O(\varepsilon^{-d/\alpha}\log(1/\varepsilon))$ bit operations by Appendix~\ref{app:cost-ops}. For $d'=2$, $K'\le K$ gives $M'+1\le\sqrt K$, so there are at most $\sqrt K$ such grids. For each $d'\ge3$, it gives $M'+1\le K^{1/3}$, and $2^{d'}\le K'\le K$ leaves fewer than $\log_2K$ values of $d'$, so there are at most $K^{1/3}\log_2K$ such grids. Summing the block lengths of these at most $\sqrt K+K^{1/3}\log_2K$ grids takes $O(\varepsilon^{-d/\alpha}\log^2(1/\varepsilon))$ bit operations.}

\paragraph{\rev{Evaluation of $\Phi^{\dexa}_{f,\varepsilon}$ and $\Phi^{\widetilde\dexa}_{f,\varepsilon}$.}} \rev{For $\Phi^{\dexa}_{f,\varepsilon}$, as} shown in \eqref{eq:exact-layer1}, the first hidden layer uses only the
negative branch, which is evaluated directly using the floor function.
In the second hidden layer \eqref{eq:exact-layer2}, since
$t=N_f+\ell=S_n+nJ_f+\ell$\rev{, which two comparisons place in the block $[S_n,S_{n+1})$},
dividing $t-S_n$ by $n$ yields quotient $J_f$ and remainder $\ell$.
The activation $\varrho$ then extracts the corresponding
base-$(n+1)$ digit of $J_f$:
$\varrho(t)
=
\big\lfloor
\frac{J_f}{(n+1)^{n-1-\ell}}
\big\rfloor-
(n+1)
\big\lfloor
\frac{J_f}{(n+1)^{n-\ell}}
\big\rfloor
=q_\ell$\rev{, which needs the power $(n+1)^{n-1-\ell}$, one division of $J_f$ by it, and one division of the quotient by $n+1$}.
Although constructing $N_f$ requires the values
$(q_\ell)_{\ell=0}^{K-1}$, once $N_f$ has been formed, pointwise network evaluation recovers each required $q_\ell$ directly from $N_f$,
without repeatedly accessing a lookup table. \rev{Similarly, in the second hidden layer of $\Phi^{\widetilde\dexa}_{f,\varepsilon}$~\eqref{eq:tree-layer2}, $t=\widetilde N_f+\ell=\widetilde S_{M,d}+K\widetilde J_f+\ell$, which two comparisons place in the block $[\widetilde S_{M,d},\widetilde S_{M,d}+K5^{K-1})$, so dividing $t-\widetilde S_{M,d}$ by $K$ yields $\widetilde J_f$ and $\ell$, and $\widetilde\dexa(t)$ in~\eqref{eq:dexa-diff} sums $\widetilde b_k(\widetilde J_f)-2$ over the at most $dM$ addresses $k\in\mathcal A(\ell)$.}

\rev{\paragraph{Evaluation of $\Phi^{\eact}_{f,\varepsilon}$ by shifts.} Shifts are available only for $\Phi^{\eact}_{f,\varepsilon}$ and $\Phi^{\eact,\mathrm{skip}}_{f,\varepsilon}$, because only there the stored digits $q_\ell$ are in the dyadic base $B$. Since $v(\bx)+1=2^{-1-r(\bx)\log_2B}$ by~\eqref{eq:el-layer2}, the third hidden layer~\eqref{eq:el-layer3} computes $T_0(\bx)=\lfloor A_f2^{-r(\bx)\log_2B}\rfloor$ and $T_1(\bx)=\lfloor A_f2^{-(r(\bx)+1)\log_2B}\rfloor$, that is, $A_f$ shifted right by $r(\bx)\log_2B$ and $(r(\bx)+1)\log_2B$ bits, which deletes the base-$B$ digits $q_0,\dots,q_{r(\bx)-1}$ and $q_0,\dots,q_{r(\bx)}$ of $A_f$ in~\eqref{eq:Af}. The skip network needs only $T_1(\bx)$ and $A_fB^{-r(\bx)}$, which is $A_f$ with its binary point moved. The digits of $J_f$ and $\widetilde J_f$ are in base $n+1$ and $5$, which are in general not powers of two, so $\Phi^{\dexa}_{f,\varepsilon}$ and $\Phi^{\widetilde\dexa}_{f,\varepsilon}$ need divisions and base conversions.}

\rev{\subsection{Measured costs}\label{app:cost-measured}
Table~\ref{tab:cost-measured} reports measured times for all settings of Table~\ref{tab:numerics-comparison}, which lists the bit lengths. All times are measured on one thread of an AMD Ryzen Threadripper PRO 5975WX, using Python with \texttt{gmpy2}, with the six test functions run at the same time, one process each. Construction times start from the quantized values, and evaluation times are medians over $32$ random dyadic points. The time of $S_n$ is the CPU time of its direct summation, in one parallel pass over all terms up to $n=2{,}000{,}376$.

The bounds of Table~\ref{tab:cost-bounds} hold up to constants, and the time of one bit operation is not known, so Table~\ref{tab:cost-measured} can be compared with Table~\ref{tab:cost-bounds} only in the order of the networks and in the growth of the times. In all $28$ settings, the construction of $\Phi^{\dexa}_{f,\varepsilon}$ takes the longest, with more than $88\%$ of its time spent on $S_n$, and the evaluations of $\Phi^{\eact}_{f,\varepsilon}$ and $\Phi^{\eact,\mathrm{skip}}_{f,\varepsilon}$ by shifts are very fast, taking at most $0.36$\,ms. Some of the times, however, do not grow like the bounds, for three reasons.
\begin{itemize}
\item In small settings, the times are dominated by a fixed cost per call of the Python interpreter, e.g., for function calls and the creation of integer and rational objects, so they grow more slowly than the bounds.
\item Table~\ref{tab:cost-bounds} uses the best multiplication cost known in theory, the length times its logarithm \citep{harvey2021integer}. GMP, the package we use, instead multiplies long integers by an FFT method in the style of \citet{schonhage1971schnelle}, whose cost is the length times its logarithm times the logarithm of its logarithm, so the times grow faster than the bounds.
\item The six test functions were run at the same time, and $S_n$ was summed by parallel processes. These processes compete for the cores and the memory bandwidth of the machine, which slows the operations on the longest integers most.
\end{itemize}}

\rev{\begin{table}[t]
\centering
\caption{\rev{Measured construction and evaluation times for the settings of Table~\ref{tab:numerics-comparison}. Both $\Phi^{\eact}_{f,\varepsilon}$ and $\Phi^{\eact,\mathrm{skip}}_{f,\varepsilon}$ depend on $f$ only through $A_f$, so they share its construction, and both are evaluated by shifts (Appendix~\ref{app:cost-bounds}).}}
\label{tab:cost-measured}
\vspace{-2mm}
\setlength{\tabcolsep}{3pt}
\rev{\resizebox{\linewidth}{!}{%
\begin{tabular}{@{}lcr*{11}{r}@{}}
\toprule
 & & & \multicolumn{4}{c}{$\Phi^{\dexa}_{f,\varepsilon}$} & \multicolumn{4}{c}{$\Phi^{\widetilde\dexa}_{f,\varepsilon}$} & \multicolumn{3}{c}{$\Phi^{\eact}_{f,\varepsilon}$, $\Phi^{\eact,\mathrm{skip}}_{f,\varepsilon}$}\\
\cmidrule(lr){4-7}\cmidrule(lr){8-11}\cmidrule(l){12-14}
$f$ & $\varepsilon$ & $K=n$ & $S_n$ & $J_f$ & constr. & eval. & $\widetilde S_{M,d}$ & $\Delta_\ell,\widetilde J_f$ & constr. & eval. & constr. & eval.\ $\Phi^{\eact}_{f,\varepsilon}$ & eval.\ $\Phi^{\eact,\mathrm{skip}}_{f,\varepsilon}$\\
 & & & [\textbf{s}] & [ms] & [\textbf{s}] & [ms] & [ms] & [ms] & [ms] & [ms] & [ms] & [ms] & [ms]\\
\midrule
$f_1$ & $2^{-1}$ & 196 & 0.00020 & 0.026 & 0.00023 & 0.020 & 0.024 & 0.075 & 0.099 & 0.020 & 0.043 & 0.0036 & 0.0037 \\
$f_1$ & $2^{-2}$ & 729 & 0.0017 & 0.075 & 0.0018 & 0.025 & 0.030 & 0.041 & 0.071 & 0.024 & 0.037 & 0.0038 & 0.0038 \\
$f_1$ & $2^{-3}$ & 2\,704 & 0.025 & 0.40 & 0.025 & 0.053 & 0.057 & 0.068 & 0.13 & 0.037 & 0.11 & 0.0044 & 0.0044 \\
$f_1$ & $2^{-4}$ & 10\,404 & 0.74 & 1.5 & 0.74 & 0.32 & 0.14 & 0.17 & 0.31 & 0.097 & 0.42 & 0.0051 & 0.0056 \\
$f_1$ & $2^{-5}$ & 41\,209 & 21 & 7.9 & 21 & 1.9 & 0.59 & 0.69 & 1.3 & 0.53 & 2.0 & 0.0086 & 0.011 \\
$f_1$ & $2^{-6}$ & 163\,216 & 532 & 40 & 532 & 11 & 3.2 & 3.4 & 6.5 & 3.7 & 5.4 & 0.026 & 0.038 \\
$f_1$ & $2^{-7}$ & 649\,636 & 14\,442 & 226 & 14\,442 & 56 & 17 & 17 & 34 & 23 & 19 & 0.10 & 0.15 \\
\addlinespace[2pt]
$f_2$ & $2^{-1}$ & 2\,500 & 0.020 & 0.28 & 0.021 & 0.049 & 0.065 & 0.10 & 0.17 & 0.037 & 0.11 & 0.0044 & 0.0046 \\
$f_2$ & $2^{-2}$ & 38\,809 & 19 & 7.2 & 19 & 1.7 & 0.57 & 1.1 & 1.7 & 0.49 & 1.7 & 0.0078 & 0.0092 \\
$f_2$ & $2^{-3}$ & 616\,225 & 12\,821 & 197 & 12\,821 & 54 & 16 & 21 & 37 & 21 & 19 & 0.068 & 0.10 \\
\addlinespace[2pt]
$f_3$ & $2^{-1}$ & 5\,929 & 0.18 & 0.81 & 0.18 & 0.13 & 0.099 & 0.21 & 0.31 & 0.059 & 0.23 & 0.0046 & 0.0047 \\
$f_3$ & $2^{-2}$ & 23\,409 & 5.2 & 4.1 & 5.2 & 0.92 & 0.32 & 0.37 & 0.69 & 0.25 & 0.56 & 0.0058 & 0.0064 \\
$f_3$ & $2^{-3}$ & 93\,025 & 143 & 23 & 143 & 5.5 & 1.6 & 2.5 & 4.1 & 1.7 & 2.7 & 0.013 & 0.017 \\
$f_3$ & $2^{-4}$ & 370\,881 & 3\,573 & 112 & 3\,573 & 26 & 8.6 & 11 & 20 & 11 & 7.9 & 0.042 & 0.064 \\
$f_3$ & $2^{-5}$ & 1\,481\,089 & 99\,021 & 637 & 99\,021 & 144 & 47 & 46 & 92 & 65 & 33 & 0.17 & 0.27 \\
\addlinespace[2pt]
$f_4$ & $2^{-1}$ & 5\,625 & 0.15 & 0.76 & 0.15 & 0.12 & 0.097 & 0.19 & 0.29 & 0.056 & 0.18 & 0.0046 & 0.0047 \\
$f_4$ & $2^{-2}$ & 21\,904 & 4.4 & 3.5 & 4.4 & 0.88 & 0.29 & 0.50 & 0.79 & 0.23 & 0.87 & 0.0058 & 0.0064 \\
$f_4$ & $2^{-3}$ & 86\,436 & 121 & 18 & 121 & 4.8 & 1.4 & 2.3 & 3.7 & 1.5 & 3.0 & 0.012 & 0.016 \\
$f_4$ & $2^{-4}$ & 344\,569 & 3\,054 & 103 & 3\,055 & 26 & 7.7 & 11 & 19 & 10 & 7.9 & 0.040 & 0.062 \\
$f_4$ & $2^{-5}$ & 1\,375\,929 & 83\,761 & 533 & 83\,761 & 130 & 43 & 49 & 92 & 60 & 34 & 0.16 & 0.25 \\
\addlinespace[2pt]
$f_5$ & $2^{-1}$ & 4\,913 & 0.10 & 0.58 & 0.10 & 0.13 & 0.087 & 0.20 & 0.28 & 0.055 & 0.23 & 0.0055 & 0.0055 \\
$f_5$ & $2^{-2}$ & 35\,937 & 15 & 6.3 & 15 & 1.6 & 0.52 & 1.2 & 1.7 & 0.44 & 0.48 & 0.0080 & 0.0089 \\
$f_5$ & $2^{-3}$ & 262\,144 & 1\,610 & 78 & 1\,610 & 20 & 5.5 & 9.3 & 15 & 6.8 & 5.8 & 0.027 & 0.042 \\
$f_5$ & $2^{-4}$ & 2\,000\,376 & 198\,146 & 878 & 198\,147 & 242 & 70 & 85 & 155 & 98 & 45 & 0.19 & 0.36 \\
\addlinespace[2pt]
$f_6$ & $2^{-1}$ & 1\,728 & 0.0089 & 0.20 & 0.0091 & 0.047 & 0.051 & 0.089 & 0.14 & 0.038 & 0.077 & 0.0054 & 0.0055 \\
$f_6$ & $2^{-2}$ & 12\,167 & 1.1 & 1.7 & 1.1 & 0.37 & 0.16 & 0.24 & 0.41 & 0.11 & 0.28 & 0.0062 & 0.0064 \\
$f_6$ & $2^{-3}$ & 91\,125 & 137 & 20 & 137 & 5.3 & 1.6 & 2.2 & 3.8 & 1.6 & 2.5 & 0.012 & 0.017 \\
$f_6$ & $2^{-4}$ & 704\,969 & 17\,206 & 250 & 17\,206 & 64 & 19 & 25 & 43 & 25 & 16 & 0.073 & 0.12 \\
\bottomrule
\end{tabular}}}
\end{table}}

\section{Additional Figures and Tables}\label{app:fig_tab}
In this appendix, we present additional figures and tables that help to understand our construction or numerical results.

\rev{\begin{figure}[htbp]
\centering
\includegraphics[width=\linewidth]{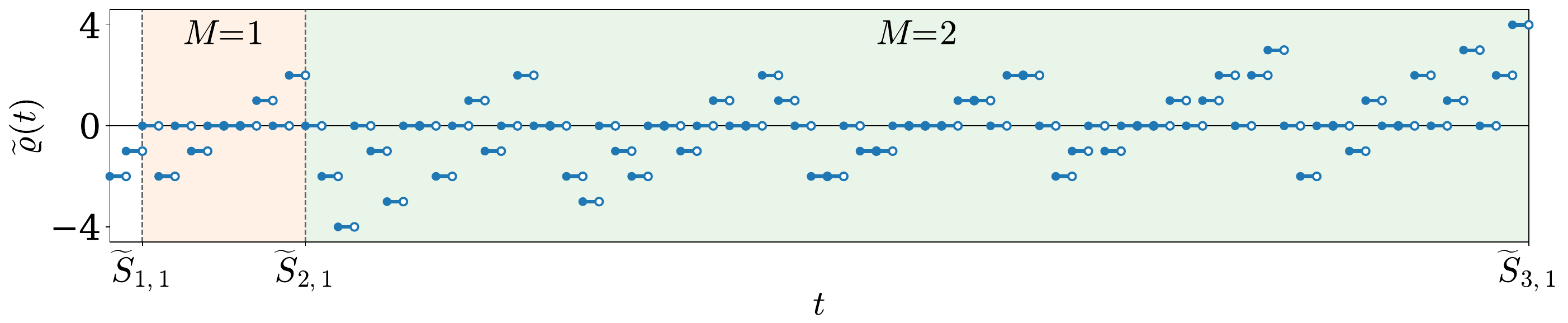}
\vspace{-6mm}
\caption{\rev{$\widetilde\dexa$ in $[-2,\widetilde S_{3,1})$, i.e., on $[-2,0)$ and the blocks $(M,d)=(1,1)$ and $(2,1)$ ($\widetilde S_{1,1}=0$, $\widetilde S_{2,1}=10$, $\widetilde S_{3,1}=85$); compare $\dexa$ in Figure~\ref{fig:varrho}. On the block of $(M,d)$, $|\widetilde\dexa|\le2dM$, since it sums at most $dM$ differences in $\{-2,\dots,2\}$.}}
\label{fig:varrho_tilde}
\vspace{-4mm}
\end{figure}}
\begin{table}[htbp!]
    \centering
    \caption{Position encoding example for $d=2$ and $M=3$, where
    $I_1=[0,\frac13)$, $I_2=[\frac13,\frac23)$,
    $I_3=[\frac23,1)$, and $I_4=\{1\}$.}
    \label{tab:position_encoding}
    \renewcommand{\arraystretch}{1.5}
    \setlength{\tabcolsep}{4pt}

    \begin{tabular}{ccc|ccc|ccc|ccc}
        \hline
        $Q_{\ell}$ & $\boldsymbol{x}^{(\ell)}$ & $\ell$
        & $Q_{\ell}$ & $\boldsymbol{x}^{(\ell)}$ & $\ell$
        & $Q_{\ell}$ & $\boldsymbol{x}^{(\ell)}$ & $\ell$
        & $Q_{\ell}$ & $\boldsymbol{x}^{(\ell)}$ & $\ell$ \\
        \hline

        $I_1\times I_1$ & $(0,0)$ & $(0)_4$
        &
        $I_2\times I_1$ & $(\frac13,0)$ & $(1)_4$
        &
        $I_3\times I_1$ & $(\frac23,0)$ & $(2)_4$
        &
        $I_4\times I_1$ & $(1,0)$ & $(3)_4$
        \\

        $I_1\times I_2$ & $(0,\frac13)$ & $(10)_4$
        &
        $I_2\times I_2$ & $(\frac13,\frac13)$ & $(11)_4$
        &
        $I_3\times I_2$ & $(\frac23,\frac13)$ & $(12)_4$
        &
        $I_4\times I_2$ & $(1,\frac13)$ & $(13)_4$
        \\

        $I_1\times I_3$ & $(0,\frac23)$ & $(20)_4$
        &
        $I_2\times I_3$ & $(\frac13,\frac23)$ & $(21)_4$
        &
        $I_3\times I_3$ & $(\frac23,\frac23)$ & $(22)_4$
        &
        $I_4\times I_3$ & $(1,\frac23)$ & $(23)_4$
        \\

        $I_1\times I_4$ & $(0,1)$ & $(30)_4$
        &
        $I_2\times I_4$ & $(\frac13,1)$ & $(31)_4$
        &
        $I_3\times I_4$ & $(\frac23,1)$ & $(32)_4$
        &
        $I_4\times I_4$ & $(1,1)$ & $(33)_4$
        \\

        \hline
    \end{tabular}
\end{table}
\begin{figure}[htbp]
\vspace{-4mm}
\centering
\includegraphics[width=0.8\linewidth]{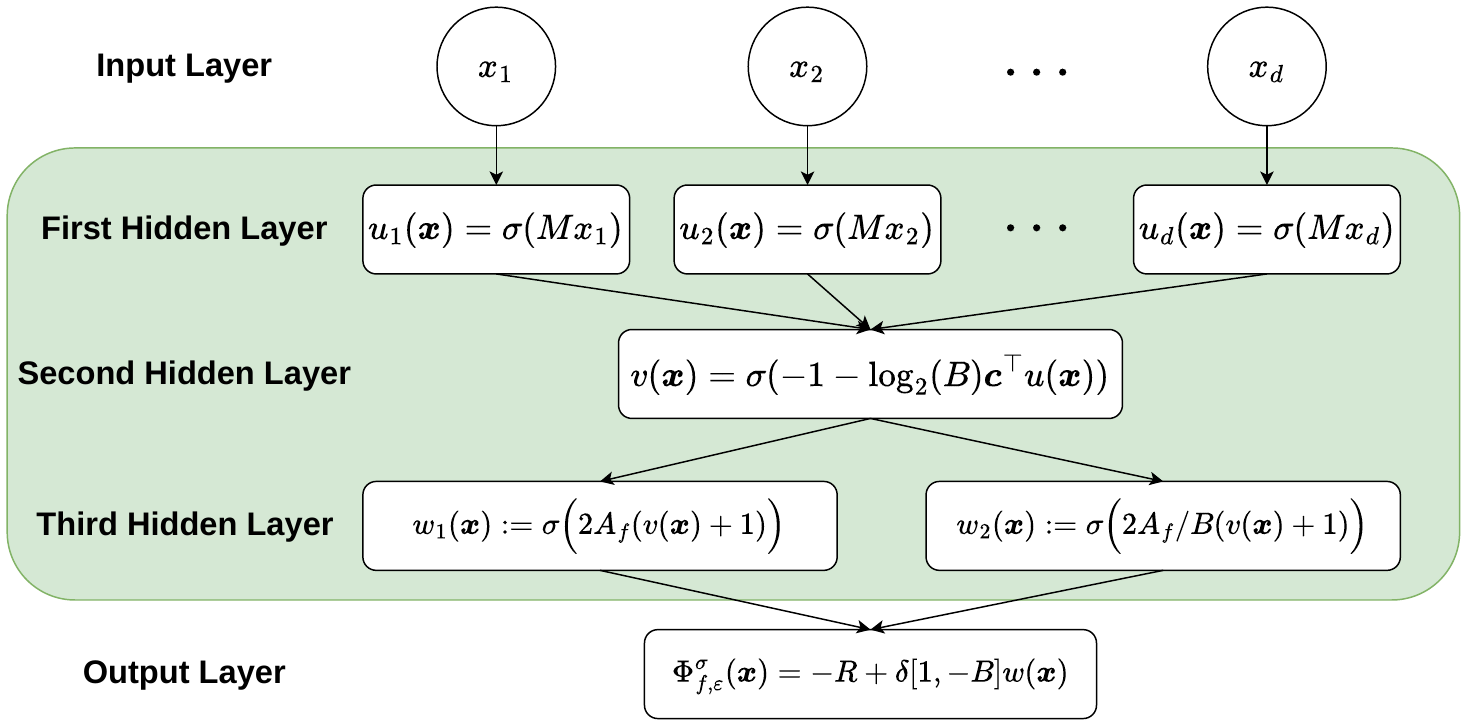}
\vspace{-2mm}
\caption{Architecture of the elementary network
$\Phi^\sigma_{f,\varepsilon}$ in
Theorem~\ref{thm:elementary}. It has hidden-layer widths
$(d,1,2)$ and $d+3$ hidden neurons in total.}
\label{fig:nn_basic}
\vspace{-4mm}
\end{figure}

\begin{figure}[htbp]
\centering
\includegraphics[width=0.8\linewidth]{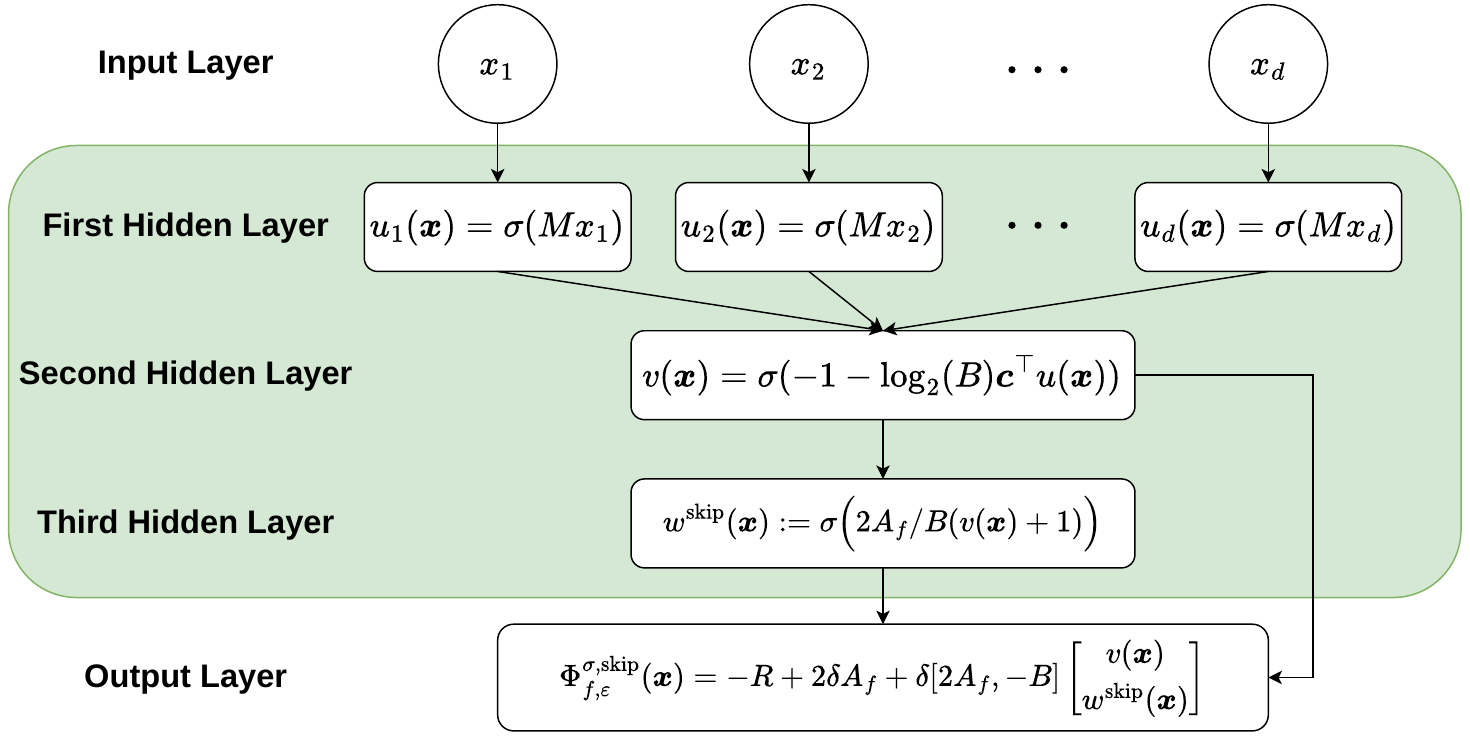}
\vspace{-2mm}
\caption{Architecture of the elementary network with a skip connection,
$\Phi^{\sigma,\mathrm{skip}}_{f,\varepsilon}$, in
Theorem~\ref{thm:elementary}. It has hidden-layer widths
$(d,1,1)$ and $d+2$ hidden neurons in total.}
\label{fig:nn_skip}
\vspace{-8mm}
\end{figure}

\begin{table}[t]
\centering
\caption{Parameters, encoding bit lengths, and empirical errors for the \rev{four} network constructions on $f_1,\ldots,f_6$.
The test set $\mathcal X_{\mathrm{test}}$ contains $3{,}004$ points for
$f_1,\ldots,f_4$ and $3{,}008$ points for $f_5,f_6$.
The empirical errors are
$E_{\varrho}:=\max_{\bx\in\mathcal X_{\mathrm{test}}}|f(\bx)-\Phi^\varrho_{f,\varepsilon}(\bx)|$,
\rev{$E_{\widetilde\varrho}:=\max_{\bx\in\mathcal X_{\mathrm{test}}}|f(\bx)-\Phi^{\widetilde\varrho}_{f,\varepsilon}(\bx)|$,}
$E_{\sigma}:=\max_{\bx\in\mathcal X_{\mathrm{test}}}|f(\bx)-\Phi^\sigma_{f,\varepsilon}(\bx)|$, and
$E_{\sigma,\mathrm{skip}}:=\max_{\bx\in\mathcal X_{\mathrm{test}}}|f(\bx)-\Phi^{\sigma,\mathrm{skip}}_{f,\varepsilon}(\bx)|$.}
\vspace{-2mm}
\label{tab:numerics-comparison}
\setlength{\tabcolsep}{2.5pt}
\resizebox{\linewidth}{!}{%
\begin{tabular}{@{}lc*{12}{r}@{}}
\toprule
$f$ & $\varepsilon$ & $\delta$ & $M$ & $K$ & $B$
& $\bit(A_f)$ & $\bit(N_f)$ & \rev{$\bit(\widetilde N_f)$}
& $\displaystyle\tfrac{\bit(A_f)}{K\log_2B}$
& $\displaystyle\tfrac{\bit(N_f)}{n\log_2(n+1)}$
& \rev{$\displaystyle\tfrac{\bit(\widetilde N_f)}{K\log_25}$}
& $E_{\varrho}$\rev{$/E_{\widetilde\varrho}/E_{\sigma}$} & $E_{\sigma,\mathrm{skip}}$ \\
\midrule
$f_1$ & $2^{-1}$ & $2^{-2}$ & 13 & 196 & 16
& 782 & 1\,497 & \rev{462} & 0.9974 & 1.0021 & \rev{1.0152}
& 0.3275 & 0.2611 \\
$f_1$ & $2^{-2}$ & $2^{-3}$ & 26 & 729 & 32
& 3\,643 & 6\,938 & \rev{1\,701} & 0.9995 & 1.0006 & \rev{1.0049}
& 0.1760 & 0.1438 \\
$f_1$ & $2^{-3}$ & $2^{-4}$ & 51 & 2\,704 & 64
& 16\,222 & 30\,834 & \rev{6\,289} & 0.9999 & 1.0001 & \rev{1.0017}
& 0.0912 & 0.0762 \\
$f_1$ & $2^{-4}$ & $2^{-5}$ & 101 & 10\,404 & 128
& 72\,826 & 138\,847 & \rev{24\,170} & 1.0000 & 1.0000 & \rev{1.0005}
& 0.0416 & 0.0335 \\
$f_1$ & $2^{-5}$ & $2^{-6}$ & 202 & 41\,209 & 256
& 329\,670 & 631\,770 & \rev{95\,699} & 1.0000 & 1.0000 & \rev{1.0002}
& 0.0223 & 0.0186 \\
$f_1$ & $2^{-6}$ & $2^{-7}$ & 403 & 163\,216 & 512
& 1\,468\,942 & 2\,826\,326 & \rev{378\,992} & 1.0000 & 1.0000 & \rev{1.0000}
& 0.0116 & 0.0097 \\
$f_1$ & $2^{-7}$ & $2^{-8}$ & 805 & 649\,636 & 1\,024
& 6\,496\,358 & 12\,544\,008 & \rev{1\,508\,426} & 1.0000 & 1.0000 & \rev{1.0000}
& 0.0056 & 0.0045 \\
\addlinespace
$f_2$ & $2^{-1}$ & $2^{-2}$ & 49 & 2\,500 & 32
& 12\,499 & 28\,225 & \rev{5\,815} & 0.9999 & 1.0002 & \rev{1.0018}
& 0.2978 & 0.2785 \\
$f_2$ & $2^{-2}$ & $2^{-3}$ & 196 & 38\,809 & 64
& 232\,853 & 591\,615 & \rev{90\,126} & 1.0000 & 1.0000 & \rev{1.0002}
& 0.1772 & 0.1435 \\
$f_2$ & $2^{-3}$ & $2^{-4}$ & 784 & 616\,225 & 128
& 4\,313\,574 & 11\,851\,923 & \rev{1\,430\,848} & 1.0000 & 1.0000 & \rev{1.0000}
& 0.0911 & 0.0734 \\
\addlinespace
$f_3$ & $2^{-1}$ & $2^{-2}$ & 76 & 5\,929 & 16
& 23\,715 & 74\,316 & \rev{13\,778} & 1.0000 & 1.0000 & \rev{1.0008}
& 0.4500 & 0.4167 \\
$f_3$ & $2^{-2}$ & $2^{-3}$ & 152 & 23\,409 & 32
& 117\,044 & 339\,781 & \rev{54\,368} & 1.0000 & 1.0000 & \rev{1.0003}
& 0.2370 & 0.2169 \\
$f_3$ & $2^{-3}$ & $2^{-4}$ & 304 & 93\,025 & 64
& 558\,149 & 1\,535\,414 & \rev{216\,013} & 1.0000 & 1.0000 & \rev{1.0001}
& 0.1189 & 0.1110 \\
$f_3$ & $2^{-4}$ & $2^{-5}$ & 608 & 370\,881 & 128
& 2\,596\,166 & 6\,861\,527 & \rev{861\,177} & 1.0000 & 1.0000 & \rev{1.0000}
& 0.0584 & 0.0535 \\
$f_3$ & $2^{-5}$ & $2^{-6}$ & 1\,216 & 1\,481\,089 & 256
& 11\,848\,711 & 30\,359\,706 & \rev{3\,439\,002} & 1.0000 & 1.0000 & \rev{1.0000}
& 0.0301 & 0.0276 \\
\addlinespace
$f_4$ & $2^{-1}$ & $2^{-2}$ & 74 & 5\,625 & 16
& 22\,498 & 70\,078 & \rev{13\,072} & 0.9999 & 1.0000 & \rev{1.0009}
& 0.4219 & 0.3552 \\
$f_4$ & $2^{-2}$ & $2^{-3}$ & 147 & 21\,904 & 32
& 109\,518 & 315\,837 & \rev{50\,873} & 1.0000 & 1.0000 & \rev{1.0003}
& 0.2027 & 0.1654 \\
$f_4$ & $2^{-3}$ & $2^{-4}$ & 293 & 86\,436 & 64
& 518\,614 & 1\,417\,500 & \rev{200\,714} & 1.0000 & 1.0000 & \rev{1.0001}
& 0.0993 & 0.0805 \\
$f_4$ & $2^{-4}$ & $2^{-5}$ & 586 & 344\,569 & 128
& 2\,411\,981 & 6\,338\,158 & \rev{800\,082} & 1.0000 & 1.0000 & \rev{1.0000}
& 0.0465 & 0.0393 \\
$f_4$ & $2^{-5}$ & $2^{-6}$ & 1\,172 & 1\,375\,929 & 256
& 11\,007\,430 & 28\,057\,917 & \rev{3\,194\,828} & 1.0000 & 1.0000 & \rev{1.0000}
& 0.0268 & 0.0209 \\
\addlinespace
$f_5$ & $2^{-1}$ & $2^{-2}$ & 16 & 4\,913 & 16
& 19\,650 & 60\,249 & \rev{11\,419} & 0.9999 & 1.0000 & \rev{1.0010}
& 0.3384 & 0.2839 \\
$f_5$ & $2^{-2}$ & $2^{-3}$ & 32 & 35\,937 & 32
& 179\,683 & 543\,846 & \rev{83\,457} & 1.0000 & 1.0000 & \rev{1.0002}
& 0.1620 & 0.1399 \\
$f_5$ & $2^{-3}$ & $2^{-4}$ & 63 & 262\,144 & 64
& 1\,572\,862 & 4\,718\,598 & \rev{608\,697} & 1.0000 & 1.0000 & \rev{1.0000}
& 0.0864 & 0.0755 \\
$f_5$ & $2^{-4}$ & $2^{-5}$ & 125 & 2\,000\,376 & 128
& 14\,002\,630 & 41\,871\,557 & \rev{4\,644\,749} & 1.0000 & 1.0000 & \rev{1.0000}
& 0.0476 & 0.0412 \\
\addlinespace
$f_6$ & $2^{-1}$ & $2^{-2}$ & 11 & 1\,728 & 16
& 6\,910 & 18\,588 & \rev{4\,022} & 0.9997 & 1.0001 & \rev{1.0024}
& 0.4025 & 0.3612 \\
$f_6$ & $2^{-2}$ & $2^{-3}$ & 22 & 12\,167 & 32
& 60\,833 & 165\,119 & \rev{28\,263} & 1.0000 & 1.0000 & \rev{1.0004}
& 0.1959 & 0.1717 \\
$f_6$ & $2^{-3}$ & $2^{-4}$ & 44 & 91\,125 & 64
& 546\,748 & 1\,501\,341 & \rev{211\,601} & 1.0000 & 1.0000 & \rev{1.0001}
& 0.1092 & 0.0924 \\
$f_6$ & $2^{-4}$ & $2^{-5}$ & 88 & 704\,969 & 128
& 4\,934\,781 & 13\,695\,580 & \rev{1\,636\,906} & 1.0000 & 1.0000 & \rev{1.0000}
& 0.0517 & 0.0453 \\
\bottomrule
\end{tabular}%
}
\end{table}

\begin{figure}[t]
\centering
\includegraphics[width=0.9\linewidth]{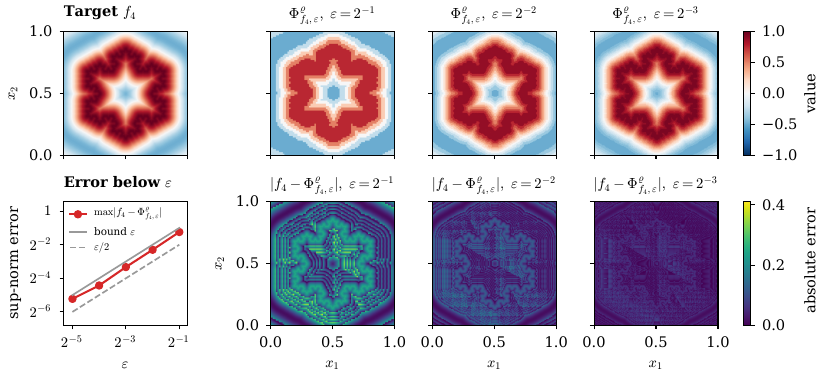}
\vspace{-4mm}
\caption{Numerical verification of Theorem~\ref{thm:exact} for the approximation of $f_4$ by $\Phi^\varrho_{f_4,\varepsilon}$. Top: The target function $f_4$ and the network outputs for $\varepsilon=2^{-1},2^{-2},2^{-3}$ ($M=74,147,293$). Bottom: The sup-norm error over the $3{,}004$ test points versus $\varepsilon$, and the absolute errors $|f_4-\Phi^\varrho_{f_4,\varepsilon}|$ on a $161^2$ grid. The error plot additionally includes $\varepsilon=2^{-4},2^{-5}$. \rev{The networks $\Phi^{\widetilde\dexa}_{f_4,\varepsilon}$ and $\Phi^{\eact}_{f_4,\varepsilon}$ have the same outputs and errors.}}
\label{fig:numerics-f4}
\vspace{-3mm}
\end{figure}

\begin{figure}[t]
\centering
\includegraphics[width=0.9\linewidth]{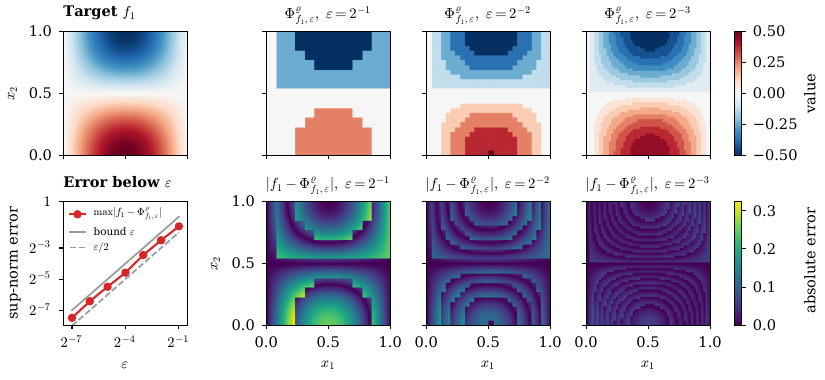}
\vspace{-4mm}
\caption{Numerical verification of Theorem~\ref{thm:exact} for the approximation of the 2D function $f_1$ by the three-neuron network $\Phi^\varrho_{f_1,\varepsilon}$. The panels are arranged as in Figure~\ref{fig:numerics-f4}, with $\varepsilon=2^{-1},2^{-2},2^{-3}$ ($M=13,26,51$); the error plot additionally includes finer networks down to $\varepsilon=2^{-7}$.}
\label{fig:numerics-f1}
\end{figure}
\vspace{-10mm}
\begin{figure}[t]
\centering
\includegraphics[width=0.9\linewidth]{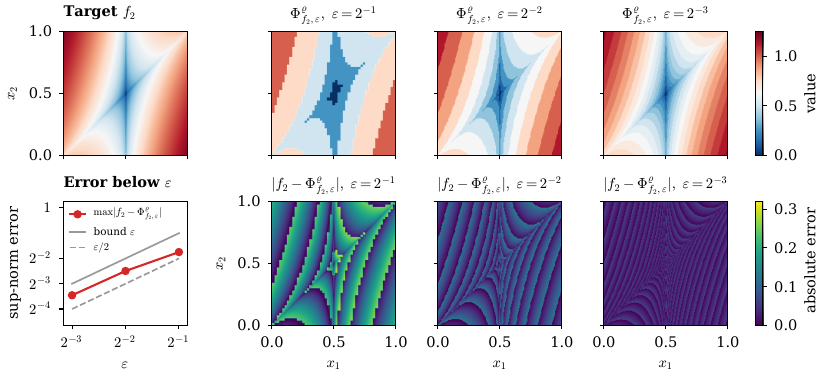}
\vspace{-4mm}
\caption{Numerical verification of Theorem~\ref{thm:exact} for the approximation of the 2D function $f_2$ by the three-neuron network $\Phi^\varrho_{f_2,\varepsilon}$. The panels are arranged as in Figure~\ref{fig:numerics-f4}, with $\varepsilon=2^{-1},2^{-2},2^{-3}$ ($M=49,196,784$).}
\label{fig:numerics-f2}
\end{figure}
\vspace{-10mm}
\begin{figure}[t]
\centering
\includegraphics[width=0.9\linewidth]{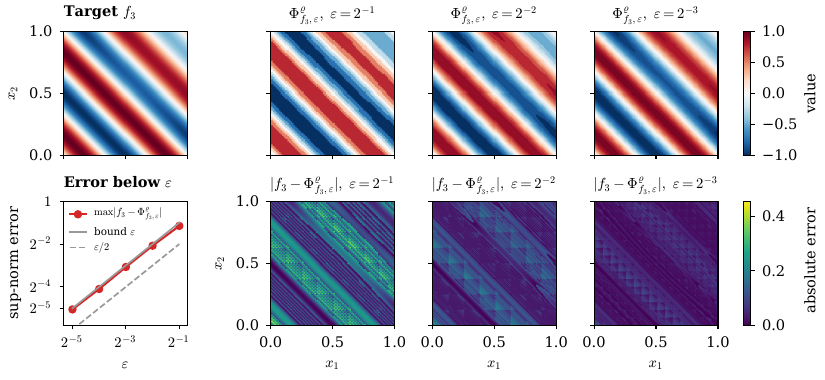}
\vspace{-4mm}
\caption{Numerical verification of Theorem~\ref{thm:exact} for the approximation of the 2D function $f_3$ by the three-neuron network $\Phi^{\dexa}_{f_3,\varepsilon}$. The panels are arranged as in Figure~\ref{fig:numerics-f4}, with $\varepsilon=2^{-1},2^{-2},2^{-3}$ ($M=76,152,304$); the error plot additionally includes finer accuracy down to $\varepsilon=2^{-5}$.}
\label{fig:numerics-f3}
\end{figure}
\begin{figure}[p]
\centering
\includegraphics[width=0.9\linewidth]
{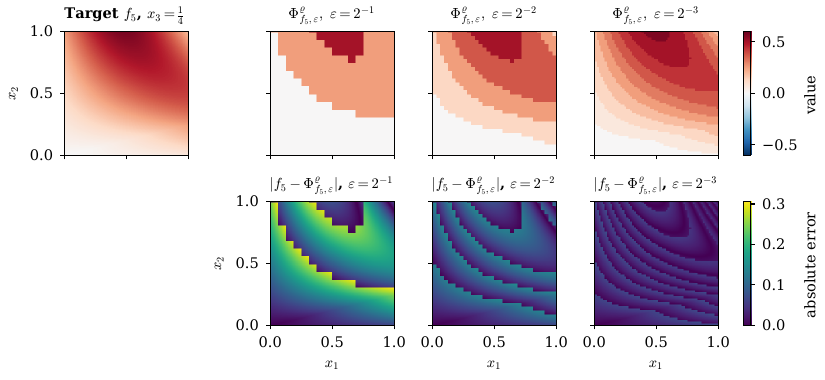}
\vspace{-2mm}

\includegraphics[width=0.9\linewidth]
{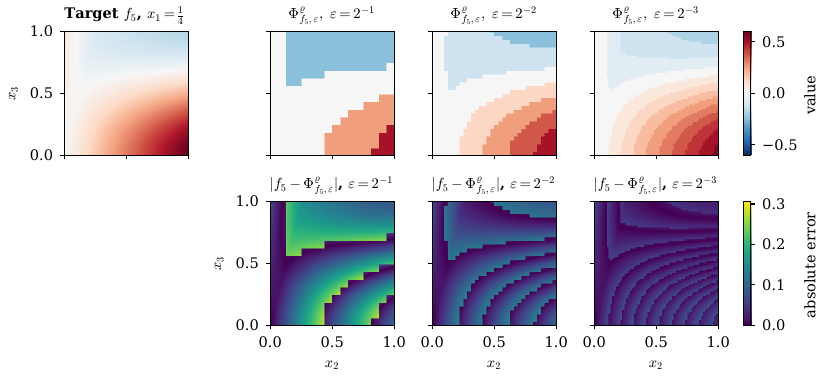}
\vspace{-2mm}

\includegraphics[width=0.9\linewidth]
{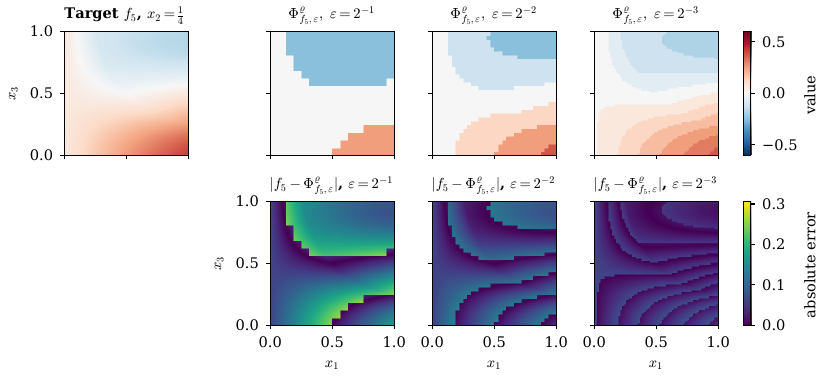}
\vspace{-2mm}

\caption{Numerical verification of Theorem~\ref{thm:exact} for the
approximation of the 3D function $f_5$ by the
four-neuron network $\Phi^\varrho_{f_5,\varepsilon}$.
From top to bottom, the panels show the $x_1x_2$, $x_2x_3$, and
$x_1x_3$ sections, with the remaining coordinate fixed at $1/4$.
Within each panel, the top row contains the target and the network
outputs for $\varepsilon=2^{-1},2^{-2},2^{-3}$
($M=16,32,63$), and the bottom row contains the corresponding
absolute errors.}
\label{fig:numerics-d3-exact-f5}
\end{figure}
\begin{figure}[p]
\centering
\includegraphics[width=0.9\linewidth]
{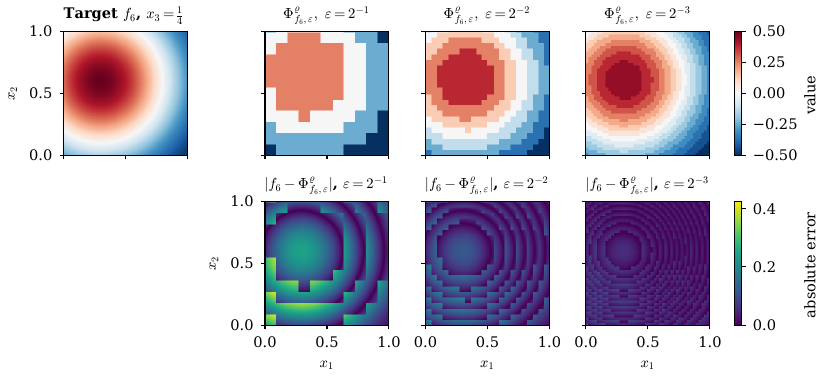}
\vspace{-2mm}

\includegraphics[width=0.9\linewidth]
{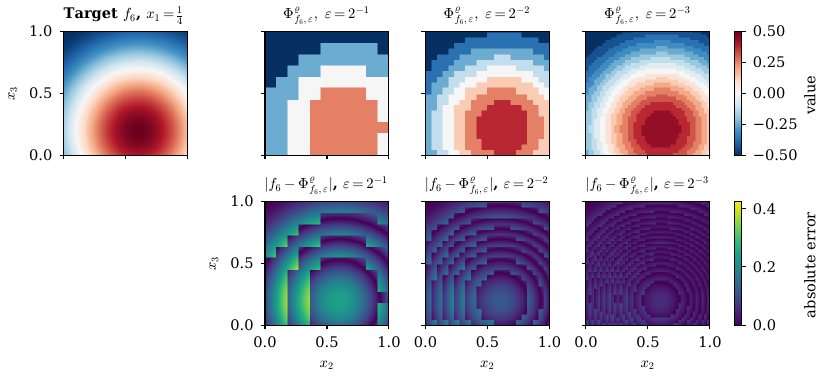}
\vspace{-2mm}

\includegraphics[width=0.9\linewidth]
{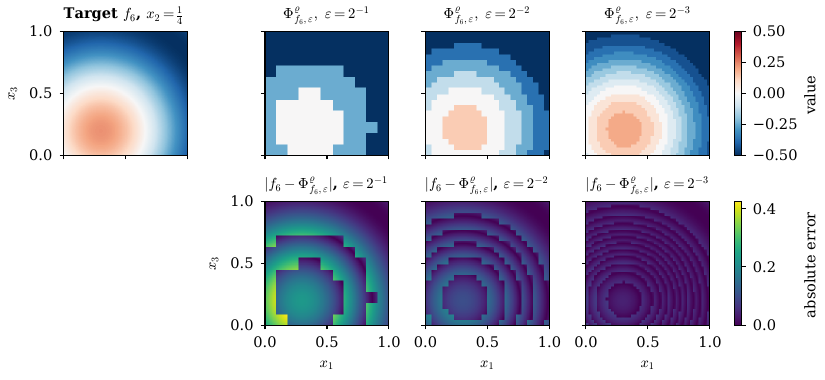}
\vspace{-2mm}

\caption{Numerical verification of Theorem~\ref{thm:exact} for the
approximation of the 3D function $f_6$ by the
four-neuron network $\Phi^\varrho_{f_6,\varepsilon}$.
The sections and panels are arranged as in
Figure~\ref{fig:numerics-d3-exact-f5}, with
$\varepsilon=2^{-1},2^{-2},2^{-3}$ and $M=11,22,44$.}
\label{fig:numerics-d3-exact-f6}
\end{figure}
\clearpage
\begingroup
\makeatletter
\setlength{\@fptop}{0pt}
\makeatother
\begin{figure}[t]
\centering
\begin{subfigure}[t]{0.3\linewidth}
    \centering
    \includegraphics[width=0.9\linewidth]
    {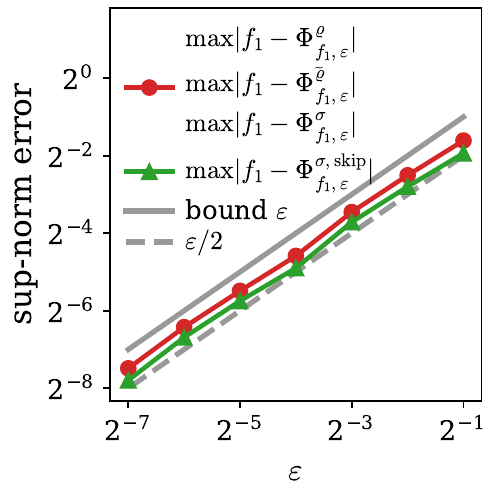}
    \caption{Target function $f_1$.}
    \label{subfig:numerics-d2-error-f1}
\end{subfigure}
\begin{subfigure}[t]{0.3\linewidth}
    \centering
    \includegraphics[width=0.9\linewidth]
    {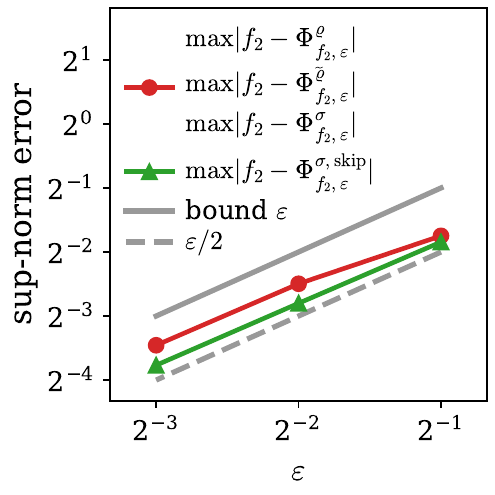}
    \caption{Target function $f_2$.}
    \label{subfig:numerics-d2-error-f2}
\end{subfigure}
\begin{subfigure}[t]{0.3\linewidth}
    \centering
    \includegraphics[width=0.9\linewidth]
    {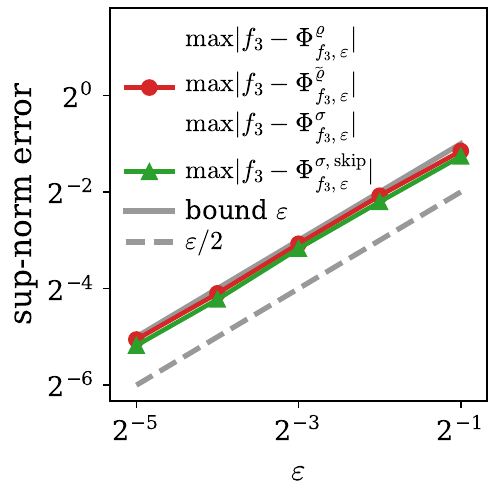}
    \caption{Target function $f_3$.}
    \label{subfig:numerics-d2-error-f3}
\end{subfigure}
\begin{subfigure}[t]{0.3\linewidth}
    \centering
    \includegraphics[width=0.9\linewidth]
    {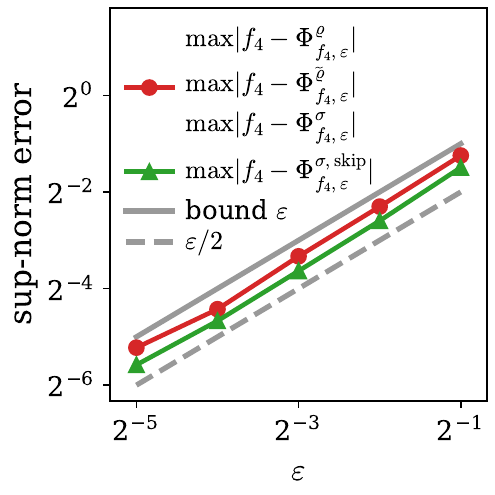}
    \caption{Target function $f_4$.}
    \label{subfig:numerics-d2-error-f4}
\end{subfigure}
\begin{subfigure}[t]{0.3\linewidth}
    \centering
    \includegraphics[width=0.9\linewidth]
    {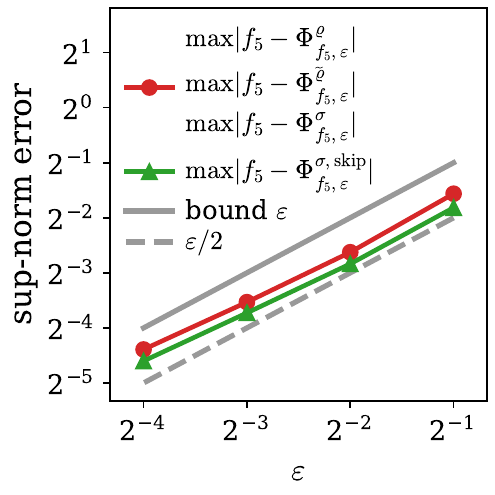}
    \caption{Target function $f_5$.}
    \label{subfig:numerics-d3-error-f5}
\end{subfigure}
\begin{subfigure}[t]{0.3\linewidth}
    \centering
    \includegraphics[width=0.9\linewidth]
    {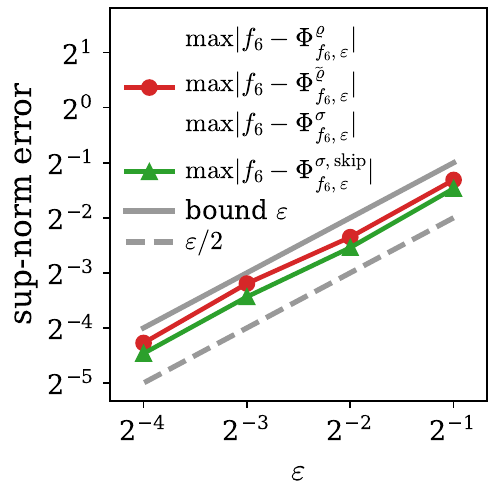}
    \caption{Target function $f_6$.}
    \label{subfig:numerics-d3-error-f6}
\end{subfigure}
\vspace{-2mm}

\caption{Sup-norm errors for the \rev{four} network constructions applied to $f_1,\ldots,f_6$, evaluated in exact arithmetic over $3{,}004$ test points for $f_1,\ldots,f_4$ ($d=2$) and $3{,}008$ test points for
$f_5$ and $f_6$ ($d=3$). Each test set consists of $2{,}000$ random points, $1{,}000$ boundary points, and the $2^d$ corners. \rev{The networks $\Phi^{\dexa}_{f,\varepsilon}$, $\Phi^{\widetilde\dexa}_{f,\varepsilon}$ and $\Phi^{\eact}_{f,\varepsilon}$} attain identical maximal errors and are therefore represented by a single curve, while the skip-connection network is shown separately. All errors remain below the prescribed
accuracy $\varepsilon$.}
\label{fig:numerics-errors}
\end{figure}
\clearpage
\endgroup

\end{document}